\documentclass[12pt]{article}
\usepackage{amsmath}
\usepackage{graphicx}
\usepackage{enumerate}
\usepackage{natbib}
\usepackage{url} 

\usepackage{ifthen}

\usepackage{mathrsfs}
\usepackage{amsmath,amssymb}
\usepackage{bm}
\usepackage[usenames]{color}
\usepackage{amsthm}
\usepackage{multirow} 
\usepackage{enumitem}
\usepackage{float}
\usepackage{caption}
\usepackage{subcaption}
\usepackage{enumitem}
\usepackage{dsfont}
\usepackage{mathtools}
\usepackage{caption}
\usepackage{booktabs}
\usepackage{array}
\usepackage[colorlinks,
linkcolor=red,
anchorcolor=blue,
citecolor=blue
]{hyperref}
\usepackage{pdflscape}
\usepackage{afterpage}

\usepackage{mylatexstyle}
\ifdefined\final
\usepackage[disable]{todonotes}
\else
\usepackage[textsize=tiny]{todonotes}
\fi
\newcommand{\xr}[1]{\textcolor{black}{#1}}

\newcommand{\blind}{0}

\newif\ifincludeappendix
\includeappendixtrue

\usepackage{xr}
\makeatletter

\begin{document}

\def\spacingset#1{\renewcommand{\baselinestretch}%
{#1}\small\normalsize} \spacingset{1}


\if0\blind
{
  \title{\bf Inference for Deep Neural Network Estimators in Generalized Nonparametric Models}
  \author
{
Xuran Meng\thanks{Department of Biostatistics, University of Michigan; e-mail: {\tt xuranm@umich.edu}} 
~~and~~
Yi Li\thanks{Department of Biostatistics, University of Michigan;
  e-mail: {\tt yili@umich.edu}}
}
\date{}
  \maketitle
} \fi


\bigskip
\begin{abstract}
While deep neural networks (DNNs) are  used for prediction, inference on DNN-estimated subject-specific means for categorical or exponential family outcomes remains underexplored. We address this by proposing a DNN estimator under  generalized nonparametric regression models (GNRMs) and developing a rigorous inference framework. Unlike existing approaches that assume independence between estimation errors and inputs to establish the error bound, a condition often violated in GNRMs, we allow for dependence and  our theoretical analysis demonstrates the feasibility of drawing inference under GNRMs.  To implement inference, we consider an Ensemble Subsampling Method (ESM) that leverages U-statistics and the Hoeffding decomposition to construct reliable confidence intervals for DNN estimates. We show that,  under  GNRM settings,  ESM enables model-free variance estimation and accounts for heterogeneity among individuals in the population.
 Through  simulations under nonparametric logistic, Poisson, and binomial regression models, we demonstrate the effectiveness and efficiency of our method. 
  We further apply the method to the electronic  Intensive Care Unit (eICU) dataset, a large scale collection of anonymized health records from ICU patients, to predict ICU readmission risk and offer patient-centric insights for clinical decision making. 
\end{abstract}

\noindent%

\vfill
\newpage
\spacingset{1.9} 

\section{Introduction} 

Given an outcome variable \(y\) and  covariates \(\mathbf{x}\),
estimating the conditional expectation \(\mathbb{E}(y |\xb) \) is central to supervised learning.  Under squared loss, \(\mathbb{E}(y |\xb) \) is the optimal estimator that minimizes the expected estimation error among all measurable functions,  making it a natural and widely adopted target in statistical modeling and machine learning. Even when optimality does not hold under more general losses, 
\(\mathbb{E}(y |\xb) \) 
 remains a valuable target, offering insight into the relationship between covariates and outcomes. It  serves as a foundational component for constructing more complex quantities across domains such as conditional variances, counterfactual means, and risk-adjusted estimates
 \citep{hastie2009elements}.
Methods that estimate \(\mathbb{E}(y|\mathbf{x})\) include linear, parametric and nonparametric regression, kernel regression \citep{nadaraya1964estimating}, support vector regression \citep{drucker1996support}, random forests \citep{breiman2001random} and deep neural networks  (DNNs) \citep{lecun2015deep}.

{With the rapid advancement of artificial intelligence, DNN–based estimators have gained widespread popularity}. However,  the statistical inference surrounding the DNN estimator of $\mathbb{E}(y |\xb)$, especially when  $y$ is categorical or belongs to the exponential family, remains underdeveloped. We  address this gap by establishing a rigorous framework for  inference on $\mathbb{E}(y |\xb)$ estimated by DNNs under a generalized nonparametric  regression model (GNRM) setting [see Eq.   \eqref{eq:model}].

Modern inference  approaches include distribution-free conformal estimation inference \citep{lei2018distribution,angelopoulos2021gentle,huang2024conformal}, extensions of jackknife inference within conformal estimation \citep{alaa2020discriminative,kim2020predictive,barber2021predictive}, debiasing methods \citep{athey2018approximate,guo2021inference}, residual-based bootstrap techniques \citep{zhang2023bootstrap}, and the use of $U$ statistics in regressions, random forests and neural networks \citep{mentch2016quantifying,wager2018estimation,schupbach2020quantifying,wang2022quantifying,fei2024u}.  
Several distinctions and limitations remain. For example, conformal prediction constructs distribution-free prediction sets for 
$y$, targeting outcome coverage rather than inference on 
$\EE(y |\xb)$. Its intervals thus absorb both outcome noise and model uncertainty, making them conservative and inefficient for mean inference. Clinically, inference on risk classifications derived from patients’ risk profiles, which are based on the estimated conditional mean (i.e., risk score), is more actionable, as medical decisions (e.g., triage or treatment intensity) rely on risk stratifications rather than prediction sets \citep{lundberg2018explainable,van2019calibration}.
Jackknife and related leave-one-out methods assume stable, nearly independent residuals, conditions violated by DNNs whose estimates depend heavily on individual samples. Debiasing approaches \citep{athey2018approximate,guo2021inference} and residual bootstraps \citep{zhang2023bootstrap} remain limited to linear models, while resampling-based inference for high-dimensional GLMs \citep{fei2021estimation} is fully parametric. Recent work on uncertainty quantification for random forests \citep{mentch2016quantifying,wager2018estimation} and neural networks \citep{schupbach2020quantifying,fei2024u} offers valuable insights but remains largely confined to continuous outcomes, leaving inference for GNRMs such as logistic and Poisson regressions unresolved.

Theoretically, for valid  inference on $\mathbb{E}(y |\xb)$,  we  need to establish the convergence rate of neural networks in a GNRM framework. Since the universal approximation theorem \citep{hornik1989multilayer}, there has been much progress in understanding the convergence rates of neural networks  \citep{mccaffrey1994convergence,kohler2005adaptive,yarotsky2017error,shen2019nonlinear,kohler2021rate,lu2021deep,shen2021robust,jiao2023deep,fan2024noise,wang2024deep,bhattacharya2024deep,yan2025deep}. For instance, \citet{shen2021robust} established the convergence rate under nonparametric regression assuming  the $p$-th moment of response is bounded for some $p>1$; \citet{fan2024noise}  explored nonasymptotic error bounds of the estimators under several loss functions, such as the least-absolute deviation loss and Huber loss; \citet{wang2024deep}  proposed an efficient and robust nonparametric regression estimator  based on an estimated maximum likelihood estimator; 
\citet{fan2024factor,bhattacharya2024deep} investigated deep neural network estimators for nonparametric interaction models with diverging input dimensions, focusing on theoretical properties such as minimax optimality, approximation error control, and sparsity-driven regularization. 
However,  quantifying uncertainty of  DNN-estimated $\EE (y|\xb)$  within the framework of  GNRM introduces additional challenges. The existing approaches typically assume a nonparametric regression framework, where the distribution of $y - \EE (y|\xb)$ does not depend on $\xb$. This assumption simplifies the theoretical analysis and the derivation of convergence results \citep{takezawa2005introduction,gyorfi2006distribution,fan2018local}. In  generalized nonparametric regression settings, however, the assumption often fails due to  heteroskedasticity, as the distribution of $y - \EE (y|\xb)$ may still depend on $\xb$. Consequently, the standard techniques for bounding error terms and controlling variance no longer apply, and the established results on uniform convergence or concentration inequalities \citep{bartlett2020benign} are  invalid.  As a result, it is unclear whether these methods are applicable for drawing  inference on GNRMs estimated by DNNs.
 

Our methodological and theoretical contributions are as follows. First, we propose a DNN estimator for subject specific means under GNRMs and develop an inference framework that explicitly accounts for covariate-dependent error distributions induced by heteroskedasticity. We provide new proof techniques, including new criterion loss function, covering-number based concentration and a truncation strategy that handles exponential-tailed, heteroskedastic noise \xr{in Section~\ref{sec:additional_proof} in the Supplementary Material},  improving the results of \citet{schmidt2020nonparametric}.  
For example,  with the general loss  induced by GNRMs, the quadratic-based arguments of \citet{schmidt2020nonparametric} no longer apply. This necessitates a complete redevelopment of the error bounds, derived instead from the intrinsic curvature and Lipschitz properties in the new GNRM setting. As a result, all estimation error analyses are novel and  move  beyond  Gaussian approximations. 
Additionally, the presence of exponential tailed, heteroskedastic noise complicates uniform convergence, which we address through a truncation strategy.  
The truncation scheme partitions the analysis into bounded and unbounded noise regimes, where standard convergence arguments apply in the former and concentration inequalities control the latter, yielding sharper error bounds. These theoretical tools enable us to extend  the continuous outcome setting of \citet{schmidt2020nonparametric}, accommodating both continuous and discrete outcomes, and thereby generalizing their framework to the broader GNRM class. 
Second, to quantify the uncertainty of the estimates, we propose an ensemble subsampling method (ESM) under GNRMs, developed from U-statistics \citep{frees1989infinite,hoeffding1992class,lee2019u,boroskikh2020u} and Hoeffding decomposition \citep{hoeffding1992class}. Our model-free variance estimator leverages the intrinsic properties of Hoeffding decomposition and accommodates varying population levels as well as individual variance structures.  {Moreover, we employ incomplete U-statistics to  reduce computational cost, enabling the resample size to grow proportionally with the sample size while limiting the total number of resamples. To mitigate the resulting Monte Carlo bias, we refine the variance estimator of \citet{wager2018estimation} for bias correction.}

In Section~\ref{sec:preamble}, we introduce deep neural networks (DNNs), generalized nonparametric regression models (GNRMs) and the ensemble subsampling method (ESM). Section~\ref{sec:theoretical_DNN} presents theoretical results for the convergence rate of DNNs under the GNRM framework, as well as  results for statistical inference using ESM. In Section~\ref{sec:simulations}, we perform simulations to evaluate the finite sample performance of our method.  Section~\ref{sec:realdata} demonstrates the proposed methodology using the electronic Intensive Care Unit (eICU) dataset to estimate both the probability and expected number of ICU readmissions based on patient characteristics, with accompanying statistical inference. The eICU dataset is a  multi-center resource comprising anonymized clinical data from ICU patients, developed through a collaboration between the Massachusetts Institute of Technology and multiple healthcare partners. The findings may support the identification of high-risk individuals and guide the delivery of personalized, risk-informed care.  Section~\ref{sec:discussion} concludes the paper with discussions.  Additional  experiments and proofs are provided in the  Supplementary Material.

\noindent
{\em Notation.} We employ lower letters (e.g. $a$) for scalars and boldface letters for vectors and matrices (e.g. $\ab$ and $\Ab$). For any matrix $\Ab$, we use $\|\Ab\|_{0}$ and $\|\Ab\|_{\infty}$ to denote its zero norm and infinity norm, respectively. The set of natural and real numbers are denoted by $\NN$ and $\RR$ respectively. For any function $f$ mapping a vector to $\RR $, we use $\|f\|_{\infty}$ to denote its infinity norm. For an integer $n$,  
$[n]= \{1, \ldots, n\}$. We write $X_1(n)=O(X_2(n))$ or  $X_1(n)\precsim X_2(n)$ or $X_2(n) \succsim X_1(n)$ if  there exist $C>0$ and $N_0>0$ such that $|X_1(n)|\leq C|X_2(n)| $ when $n >N_0$.    We denote $X_2(n)  \asymp X_1(n) $ if $X_1(n)=O(X_2(n))$ and $X_2(n)=O(X_1(n))$. We denote $X_1(n)=o(X_2(n))$ if $X_1(n)/X_2(n)\to 0$,  and $X_1(n)\sim X_2(n)$ if $\{X_1(n)/X_2(n)\}$ converges to $1$. When clear from context, we add the index $p$ (for probability) such as $O_p$ for a random variable and its realized value to avoid redundancy.

\section{The Preamble}
\label{sec:preamble}

\subsection{Deep Neural Networks} 
Let $L$ be a positive integer
denoting the \textbf{depth} of a neural network,
and  $\pb=(p_0,...,p_L,p_{L+1})^\top$ be a vector of  positive integers specifying the \textbf{width} of each layer,  with $p_0$ corresponding to the input dimension, $p_1, \dots, p_L$  the dimensions of the $L$ hidden layers, and $p_{L+1}$ the dimension of the output layer. An $(L+1)$-layer DNN with layer-width  $\pb$ is defined as 
\begin{align}
    f(\xb)=\Wb_L f_L(\xb)+\vb_L, \label{eq:def_DNN}
\end{align}
with the recursive expression 
 $   f_{l}(\xb)=\sigma(\Wb_{l-1}f_{l-1}(\xb)+\vb_{l-1}),\quad f_{1}(\xb)=\sigma(\Wb_0\xb+\vb_0).$
Here, $\Wb_l\in\RR^{p_{l+1}\times p_l}$ and vectors $\vb_l\in\RR^{p_{l+1}}$ ($l\in[L]$) are the parameters of this DNN, and $\sigma:\RR\to\RR$ is the activation function, applied element-wise to vectors. A commonly used 
 activation function
 is the rectified linear unit (ReLU) function \citep{nair2010rectified}, or
\(\sigma(x)=\max(x,0). \)  
We focus on ReLU due to its piecewise linearity and projection properties, with potential extensions to other activation functions as discussed in Section~\ref{sec:discussion}.

 {Overparameterized DNNs can memorize noise and overfit \citep{cao2022benign}, so sparsity is used to control complexity.}
 We focus on the following class of  $s$-sparse and $F$-bounded networks for  our  theoretical analysis: 
\begin{align}
\label{eq:class_F2}
\cF(L,\pb,s,F)=\big\{f \text{ of form \eqref{eq:def_DNN}: }\|\Wb_{l}\|_{\infty},\|\vb_l\|_{\infty}\leq 1,\sum_{l=1}^L\|\Wb_l\|_0+\|\vb_l\|_0\leq s, \|f\|_{\infty}\leq F\big\}.
\end{align}
Here, $s\in\NN_+$, and $F>0$ is a constant. $\|\cdot\|_0$ is the number of nonzero entries of matrix or vector, and $\|f\|_{\infty}$ is the sup-norm of function $f$.

\subsection{Generalized Nonparametric Regression}
\label{subsec:GLM}
 Consider outcomes from the exponential family, including commonly used distributions such as normal, Bernoulli, Poisson, and binomial. The density (or mass) function for a random variable \( y \) in 
 the (single-parameter) exponential 
 family is  
\(
p(y|\theta) = h(y) \exp\left\{\theta y - \psi(\theta)\right\},
\)  
where \( \theta \) is the canonical parameter, \( 0< h(y)< \infty \) ensures normalization, and \( \psi(\theta) \) is  convex and smooth,  satisfying $\EE(y|\theta)=\psi'(\theta)$, where  $\psi'(\cdot)$ denotes the derivative of  $\psi(\cdot)$ and 
the inverse function of  $\psi'(\cdot)$ is known as the link function \citep{brown1986fundamentals,fei2021estimation}.

Suppose we observe $n$  independently and identically distributed  (iid) sample $\{(\xb_i,y_i)\}_{i=1}^n$ from $(\xb,y)$, where  $\xb\in\RR^p$ is a covariate vector and $y\in\RR$ is the response variable. Given $\xb$, the distribution of $y$ is modeled as \begin{align}
    p(y | \xb) = h(y) \exp\bigg\{ y f_0(\xb) - \psi(f_0(\xb)) \bigg\}, \label{eq:model}
\end{align}where $f_0:\RR^p\to \RR$  is a nonparametric function  with $f_0(\xb)$ representing the  mean parameter of the response variable $y$ given $\xb$.  Model \eqref{eq:model} is referred to as a Generalized Nonparametric Regression Model (GNRM), where the outcomes follow a distribution from the exponential family.  With an independent test point $(\xb_{\new},y_{\new})$ and that $y_{\new}$, given $\xb_{\new}$, follows  \eqref{eq:model}, we aim to estimate $\EE (y_{\new}|\xb_{\new})$ which is $\psi'(f_0(\xb_{\new}))$,  
Consequently, it is of  interest to estimate $f_0$. By the universal approximation theorem \citep{hornik1989multilayer}, DNNs are one of  suitable candidates for estimating $f_0$.  We consider an optimal  DNN estimator, $\hf_n^{\opt}$, which   minimizes the negative log  likelihood function \eqref{eq:model} over the class of DNNs defined in \eqref{eq:class_F2}, that is, 
\begin{align}
    \hf_n^{\opt}=\argmin_{f\in\cF(L,\pb,s,F)}\frac{1}{n}\sum_{i=1}^n \{-y_i f(\xb_i)+\psi(f(\xb_i))\}.\label{eq:hf}
\end{align}
The estimator \eqref{eq:hf} is applicable to  commonly used models, including these listed below.

\noindent\underline{Nonparametric Gaussian Regression:}   Given $\xb$, we assume that the response variable follows a Gaussian distribution with mean $f_0(\xb)$ and variance $\sigma^2$. In this setting,  $\sigma^2$ does not affect the estimation of the mean function $f_0(\xb)$, which can be seen by examining the likelihood and optimization.
The likelihood function for a nonparametric Gaussian regression model is given by
$    p(y|\xb)=\frac{1}{\sqrt{2\pi}\sigma}e^{-\frac{(y-f_0(\xb))^2}{2\sigma^2}} $. 
An application of \eqref{eq:hf} to this setting yields:
\begin{eqnarray*}
\hf_n^{\opt}& = &  \argmin_{f\in\cF(L,\pb,s,F)}\frac{1}{n}\sum_{i=1}^n \{-y_if(\xb_i) + \frac{1}{2}f^2(\xb_i)\} \\
& = & 
\argmin_{f\in\cF(L,\pb,s,F)}\frac{1}{n}\sum_{i=1}^n \{y_i-f(\xb_i)\}^2,
\end{eqnarray*}
 corresponding to  mean squared error minimization as in \citet{schmidt2020nonparametric}.

\noindent\underline{Nonparametric Logistic Regression:} To model binary outcomes $y \in \{0, 1\}$, we set $h(y) = 1$, $\psi(\theta) = \log(1 + e^{\theta})$ in \eqref{eq:model}. The  estimator  \eqref{eq:hf} in this setting can be written as:
\begin{align*}
\hf_n^{\opt}=\argmin_{f\in\cF(L,\pb,s,F)}\frac{1}{n}\sum_{i=1}^n (-y_i f(\xb_i)+\log[1+\exp\{f(\xb_i)\}]).
\end{align*}

\noindent\underline{Nonparametric Poisson Regression:}
 Considering \( y \in \mathbb{N} \) (non-negative integers), we define \( h(y) = \frac{1}{y!} \) and \( \psi(\theta) = e^{\theta} \) in \eqref{eq:model}.  The  estimator  \eqref{eq:hf}  has the following form: 
\[
\hf_n^{\opt} = \argmin_{f \in \cF(L, \pb, s, F)} \frac{1}{n} \sum_{i=1}^n [ \exp\{f(\xb_i)\} - y_i f(\xb_i)].
\]

\noindent\underline{Nonparametric Binomial Regression:} Given covariates $\xb$, the response variable $y$ follows a Binomial distribution with a fixed number of trials $n_{\text{trial}}$ and the success probability $p(\xb)=\frac{1}{1 + e^{-f_0(\xb)}}$. 
This corresponds to $h(y)={n_{\text{trial}}\choose{y}} $ and $\psi(\theta)=n_{\text{trial}}\log(1+e^{\theta})$.  In this setting,  the  estimator  \eqref{eq:hf} is specified as
\[
\hf_n^{\opt} = \argmin_{f \in \cF(L, \pb, s, F)} \frac{1}{n} \sum_{i=1}^n \Big[ -y_i f(\xb_i) + n_{\text{trial}} \log\{1 + e^{f(\xb_i)}\} \Big].
\]

\subsection{Empirical Risk Minimization and Function Smoothness}
\label{sec:approximate_risk}
 
Recall that the pointwise Kullback--Leibler  divergence between two distributions of the form \eqref{eq:model}, one specified by  \( f \) and another  by the true  \( f_0 \), conditional on \( \xb \), i.e.,
\[
\EE \left[ -y f(\xb) + \psi(f(\xb)) + y f_0(\xb) - \psi(f_0(\xb)) \,\big|\, \xb \right] 
= -\psi'(f_0(\xb))(f(\xb) - f_0(\xb)) + \psi(f(\xb)) - \psi(f_0(\xb)),
\]
where we use \( \EE[y | \xb] = \psi'(f_0(\xb)) \).   This equals the Bregman divergence between \( f(\xb) \) and \( f_0(\xb) \), which is nonnegative and vanishes if and only if \( f(\xb) = f_0(\xb) \).  We then quantify the estimation error of $\hat{f}$, any estimator of $f_0$, by evaluating the pointwise estimation error at an independent test point \xr{\( \bX\)}:
\(
\ell(\bX; \hat{f}, f_0) = -\psi'(f_0(\bX)) (\hat{f}(\bX) -f_0(\bX)) +\psi(\hat{f}(\bX))  - \psi(f_0(\bX)).
\)
We further define the average estimation error, our main theoretical target, as
\begin{align}
R_n(\hat{f}, f_0) := \mathbb{E}[\ell(\bX; \hat{f}, f_0)], \label{eq:def_Rn}
\end{align}
where the expectation is taken with respect to all involved random variables.

 Computing the exact minimizer $\hf_n^{\opt}$ in \eqref{eq:hf} is challenging due to the nonconvex loss \citep{fan2024factor}. Instead, we obtain an approximate minimizer $\hat{f}_n$ 
 using optimization algorithms such as gradient descent. To quantify the optimization error, 
we introduce
\begin{align}
\Delta_{n}(\hf_n,\hf_n^{\opt})= 
\EE\bigg\{\frac{1}{n} \sum_{i=1}^n  (-  y_i \hf_n(\xb_i) +  \psi(\hf_n(\xb_i)) -   \frac{1}{n} \sum_{i=1}^n  (-  y_i \hf_n^{\opt}(\xb_i) +  \psi( \hf_n^{\opt}(\xb_i))\bigg\}\label{eq:obtained_estimator}
\end{align} 
that measures the difference between the expected empirical risk of $\hf_n$ and $\hf_n^{\opt}$. For notational ease, we write $\Delta_n(\hat f_n) := \Delta_n(\hat f_n,\hat f_n^{\opt})$ when the reference to $\hat f_n^{\opt}$ is clear. We will show that a small empirical deviation \( \Delta_{n}(\hf_n) \), combined with controlled model complexity of the neural network class, implies a bound on the average excess risk \( R_n(\hat{f}_n, f_0) \),   forming the foundation for the inference theory developed in subsequent sections.

 For desirable  properties of $\hf_n$, { we assume \( f_0 \) belongs to a H\"older class \citep{schmidt2020nonparametric, fan2024factor}}
  with parameters $\gamma$, $K > 0$, and domain $\DD \subset \RR^r$:
\begin{align*}
    \cG_{r}^{\gamma}(\DD,K)=\bigg\{g:\DD\to\RR:\sum_{\bbeta:\|\bbeta\|_1<\gamma} \|\partial^{\bbeta}g\|_{\infty}+\sum_{\bbeta:\|\bbeta\|_1=\lfloor \gamma\rfloor}\sup_{\xb,\yb\in\DD,\xb\neq \yb}\frac{|\partial^{\bbeta}g(\xb)-\partial^{\bbeta}g(\yb)|}{\|\xb-\yb\|_{\infty}^{\gamma-\lfloor \gamma\rfloor}}\leq K\bigg\},
\end{align*}
where $\lfloor \gamma\rfloor$ is the largest integer strictly smaller than $\gamma$, and $\partial^{\bbeta}=\partial^{\beta_1}\ldots \partial^{\beta_r}$ with $\bbeta=(\beta_1,...,\beta_r)^\top\in\NN^r$. We further assume that $f_0$ is a $(q+1)$-composition ($q\in\NN$) of several H\"older functions. That is, for some vectors $\db=(d_0,...,d_{q+1})\in\NN_+^{q+2}$, $\tb= (t_0,...,t_q)\in\NN_+^{q+1}$ and $\mb=(m_0,...,m_q)\in\RR_+^{q+1}$, $f_0\in  \cG(q,\db,\tb,\mb,K)$, where
\begin{align*}
    \cG(q,\db,\tb,\mb,K):=\bigg\{&f=g_q\circ \cdots \circ g_0: g_i=(g_{ij})_j:[a_i,b_i]^{d_i}\to[a_{i+1},b_{i+1}]^{d_{i+1}},\\
    &g_{ij}\in \cG_{t_i}^{m_i}([a_i,b_i]^{t_i}) \text{ with }|a_i|,|b_i|\leq K  \bigg\}.
\end{align*}
{ The motivation for the assumptions on $f_0$ is that the H\"older class  provides a flexible framework for quantifying smoothness, capturing both integer and fractional orders, and serves as a natural benchmark for assessing neural network approximation \citep{yarotsky2017error}, while
the compositional structure  mirrors the architectural design of deep neural networks with each layer applying a transformation to the output of the previous one \citep{yarotsky2017error}.}
We also note that $\db$ and $\tb$ characterize the dimensions of the function, while  $\mb$ represents the measure of smoothness. For instance, if 
\begin{align*}
f_0(\xb)=g_{21}\Big(&g_{11}\big(g_{01}(x_1,x_2,x_3,x_4),g_{02}((x_5,x_6,x_7,x_8))\big),\\
& g_{12}\big(g_{03}(x_9,x_{10},x_{11},x_{12}),g_{04}((x_{13},x_{14},x_{15},x_{16}))\big),\\
&g_{13}\big(g_{05}(x_{17},x_{18},x_{19},x_{20}),g_{06}((x_{21},x_{22},x_{23},x_{24}))\big) \Big),\quad \xb\in[0,1]^{24},
\end{align*}
and  $g_{ij}$ are twice continuously differentiable, then smoothness $\mb=(2,2,2)$, dimensions $\db=(24,6,3,1)^\top$ and $\tb=(4,2,3)^\top$. 
We further define \xr{$m_j^*=m_j\Pi_{l=i+1}^{q}(m_l\wedge 1 )$ and $\phi_n=\max_{j=0,...,q} n^{-2m_j^*/(2m_j^*+t_j)}$}.

\subsection{Ensemble Subsampling Method (ESM)}
 Consider a scenario where $\xb$ is drawn from an unknown probability measure $\PP_{\xb}$ supported on $\cX$, and for an input $\xb_{\new}=\xb_{*} \in \cX$, the unobserved outcome $y_{\new}$ follows  \eqref{eq:model}. Our goal is to estimate  $\EE(y_{\new}|\xb_{\new}=\xb_{*}) = \psi'(f_0(\xb_{*}))$ and its uncertainty.
For this, we introduce the Ensemble Subsampling Method (ESM), based on subsampling techniques \citep{wager2018estimation, fei2024u}, to construct ensemble estimators and confidence intervals.  ESM consists of the following components, which are summarized in Figure~\ref{fig:esmpipline} as well. 
\begin{itemize}
\item{(Subsampling)}
Let $\mathcal{I} = \{1, \dots, n\}$ denote the index set of the training dataset $\mathcal{D}_n$. We construct subsets of size $r (< n)$, yielding $B^* = \binom{n}{r}$ unique combinations. Denoting the $b$-th subset as $\mathcal{I}^b = \{i_1, \dots, i_r\}$, where $i_1 < \dots < i_r$ and $b \in [1 : B^*]$.
\item{(Estimator Construction)}
For each  $ 1 \le b \le B^*$, we aim to minimize the objective function in \eqref{eq:hf}  within $\cF(L,\pb,s,F)$ by  using a gradient based optimization algorithm such as Stochastic Gradient Descent (SGD) on the observations indexed by $\mathcal{I}^b$.  Let $\hat{f}^b$ denote the resulting approximate minimizer satisfying  
\begin{align}
    \Delta_r(\hf^b,\hf^{b,\opt})\leq \Delta^{\opt}_b,\label{eq:delta_opt}
\end{align}
where $\Delta^{\opt}_b$ quantifies the optimization error \citep{fan2024factor},
and 
$\hf^{b,\opt}$ and $ \Delta_r(\hf^b,\hf^{b,\opt})$ are defined respectively in \eqref{eq:hf}  and  \eqref{eq:obtained_estimator} on the subsample indexed by $\cI^b$.

\item{(Ensemble Estimation)}
Using the estimators $\{\hat{f}^b : b = 1, \dots, B^*\}$, we compute the ensemble estimation for any $\mathbf{x}_{*} \in \mathcal{X}$ as:
\begin{align*}
        \hf^{B^*}(\xb_{*})=\frac{1}{B^*}\sum_{b=1}^{B^*}\hf^{b}(\xb_{*}),
\end{align*}
and estimate 
$\EE(y_{\new}|\xb_{\new}=\xb_{*})$
as $\psi'(\hf^{B^*}(\xb_{*})))$.
\item{(Confidence Interval Construction)}
To quantify uncertainty, we estimate the variance, denoted  $\hat{\sigma}_*$, using the infinitesimal jackknife method. The resulting confidence interval (CI) is defined as:
\begin{align}
\text{CI}(\xb_{*})=\big[\psi'(\hf^{B^*}(\xb_{*})-c_{\alpha}\hsigma_*),\psi'(\hf^{B^*}(\xb_{*})+c_{\alpha}\hsigma_*  )\big].\label{eq:CI}
\end{align}
where $c_\alpha$ is a constant controlling the $1 - \alpha$ confidence level.
\end{itemize}
 
\begin{figure}[t]
    \centering
    \includegraphics[width=0.8\linewidth]{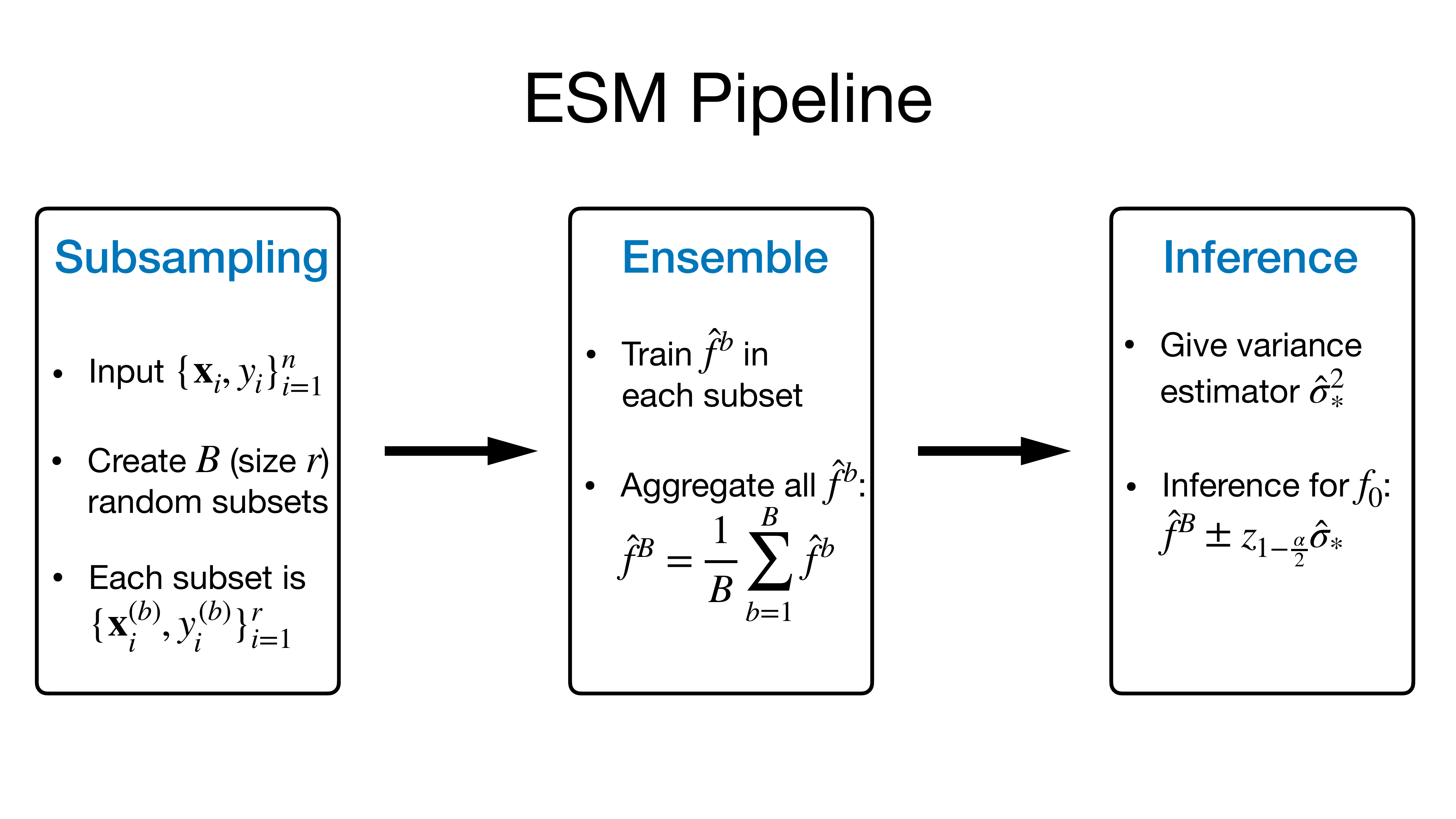}
    \vspace{-20pt}
    \caption{Overview of the Ensemble Subsampling Method (ESM).}
    \label{fig:esmpipline}
\end{figure}
The distribution of $\hat{f}^b(\xb_{*})$ for a fixed $\xb_{*}$ is  intractable. However, since $\hat{f}^b(\xb_{*})$ is permutation symmetric with respect to $\cI^b$, the ensemble estimator $\hat{f}^{B^*}(\xb_{*})$ forms a generalized U-statistic, which we later show to be asymptotically normal by leveraging U-statistics theory.
In practice, we allow $r$ to grow with $n$, but evaluating all $B^* = {{n}\choose{r}}$ models becomes computationally infeasible. To overcome this, we use a stochastic approximation by randomly drawing $B$ size-$r$ subsamples from $\cD_n$. This involves independently sampling indices $b_1, \dots, b_B$ from $1$ to ${{n}\choose{r}}$, where each subsample is indexed by $\cI^{b_j}$, and  computing 
\begin{align}
\hf^{B}(\xb_{*})=\frac{1}{B}\sum_{j=1}^B \hf^{b_j}(\xb_{*}).\label{eq:approx_hfB}
\end{align}
As \eqref{eq:approx_hfB} provides the estimator for $f_0$, the variance $\hsigma_*$  in \eqref{eq:CI} in the Confidence Interval Construction procedure   is given as 
\begin{align}
    \hsigma_*^2=\underbrace{\frac{n(n-1)}{(n-r)^2} \sum_{i=1}^n \hat{V}_i^2}_{\text{original }\hsigma^2}-{\underbrace{\frac{n(n-1)}{(n-r)^2}\cdot \frac{1}{B(B-1)}\sum_{i=1}^n\sum_{j=1}^B(Z_{b_{j}i}-\hat V_i)^2}_{\text{bias correction}}}.\label{eq:def_hsigma_*}
\end{align}
Here,   $Z_{b_{j}i}$ and $\hat{V}_i$ are defined as 
\begin{align*}
    Z_{b_{j}i}=(J_{b_ji}-J_{\cdot i})\big(\hf^{b_j}(\xb_{*})-\hf^{B}(\xb_{*})\big), \quad  \hat{V}_i=\frac{\sum_{j=1}^B Z_{b_j i}}{B},
\end{align*}
where $J_{b_ji}=I(i\in\cI^{b_j})$ and $J_{\cdot i}=\frac{1}{B}\sum_{j=1}^B J_{b_ji}$. 
Extending \citet{peng2024bias},  \eqref{eq:def_hsigma_*} jointly corrects Monte Carlo bias and finite-sample effects using the factor \( n(n-1)/(n-r)^2 \), enabling valid inference with \( B \succsim n \). Its construction  relies on  the  dominance of the first order term in the Hoeffding decomposition, and the bias correction term adequately captures its resampling variability, regardless of the specific scaling between \( r \) and \( n \).

\section{Theoretical Results}
\label{sec:theoretical_DNN}
 Let $\zb_i = (y_i, \xb_i)$ represent an independent observation of sample points.   For studying the asymptotic properties of $\hf^B$ in    \eqref{eq:approx_hfB}, we introduce
\begin{align} \label{cov1r}
\xi_{1,r}(\xb)=\Cov(\hf(\xb;\zb_1,\zb_2,...,\zb_{r}),\hf(\xb;\zb_1,\zb'_2,...,\zb'_r)).
\end{align}
Here $\xb\in \cX$ is in the support of probability measure $\PP_{\xb}$, $\hf(\xb; \zb_1, \zb_2, \dots, \zb_r)$ means that we obtain $\hf$ from the subsample $\zb_1, \zb_2, \dots, \zb_r$ and then apply it to the point $\xb$, and $\xi_{1,r}$ quantifies the covariance between estimates based on two overlapping subsets of size $r$ that share a common point $\zb_1$ but differ in the remaining elements. $\xi_{1,r}$ represents a second-order component in the Hoeffding decomposition of a U-statistic \citep{hoeffding1992class}, capturing the pairwise dependence structure within the subsamples used for constructing the estimator. 
The terms $\zb'_i$ and $\zb_i$ are independently generated from the same data generation process. 
We  impose some regularity assumptions for our theoretical guarantees.
\begin{assumption}
\label{assump:1}
    Suppose that \eqref{eq:model} holds and the true function $f_0\in\cG(q,\db,\tb,\mb,K)$. The support of $y$, denoted by $\supp(h(y))$, is fixed and does not depend on $\xb$. 
    Moreover, the sufficient statistic $T(y)=y$ is not degenerate and can take at least two distinct values in $\supp (h)$.  For any $\xb$, the conditional distribution $p(y | \xb)$ is sub-exponential. Specifically, with bounded $f_0$ there exists a universal constant $\kappa > 0$ such that for all $t > 0$,
\begin{align*}
   \mathbb{P}(|y - \mathbb{E}(y | \xb)| \geq t| \xb) \leq 2 \exp\left(-\frac{t}{\kappa}\right). 
\end{align*}
\end{assumption}
\begin{assumption}
\label{assump:network}
The network class $\cF(L,\pb,s,F)$ with $L$, $\pb$ and $s$ satisfies
i) $F\geq \max(K,1)$,
    ii) $\sum_{i=0}^q\log_2(4t_i\vee 4m_i)\log_2 n\leq L\precsim \log^{\alpha} n$ for some $\alpha>1$,
    iii) $n\phi_n\precsim \min_{i=1,...,L}p_i\leq \max_{i=1,...,L}p_i\precsim n$,
    iv) $s\precsim n\phi_n\log n$.
Here, $\phi_n=\max_{j=0,...,q} n^{-2m_j^*/(2m_j^*+t_j)}$, where $m_j^*=m_j\Pi_{l=i+1}^{q}(m_l\wedge 1 )$. The values  $t_i,m_i$ and $p_i$ are the $(i+1)$-th element in the vectors $\tb$, $\mb$ and $\pb$, respectively. 
\end{assumption}
\begin{assumption}
\label{assump:train}
For each subsample with index $\cI^b$ and  size $r$,  there exists a constant $C>0$ such that the optimization error $\Delta^{\opt}_b$ defined in \eqref{eq:delta_opt} satisfies $\Delta^{\opt}_b\leq  C \phi_r L\log^2r$. 
\end{assumption}
\begin{assumption}
\label{assump:4}
{The covariance $\xi_{1,r}(\xb)$ defined in \eqref{cov1r}  satisfies 
$\inf_{\xb\in \cX}\xi_{1,r}(\xb)\geq Cr^{-(1+\varepsilon)}$ for some $\varepsilon>0$. The  subsample size $r=n^{\gamma}$  satisfies $1/(1+\min_{j=0,...,q}(2m_j^*)/(2m_j^*+t_j))<\gamma<1$, and  
the number of subsampling times $B$ is large such that $B \succsim n$.}  
\end{assumption}
 Assumption~\ref{assump:1}  ensures that $p(y|\xb)$ not  degenerate, while the concentration of $y$ does not deviate too widely.
 Additionally, that $T(y) = y$  takes at least two distinct values in $\supp(h)$ guarantees a non-zero variance of $y$,  ensuring the log-partition function $\psi(f_0(\xb))$ exhibits local convexity.  The boundedness of $f_0$ guarantees the existence of $\kappa$. This condition,  weaker than the sub-Gaussian tail bound assumed in \citet{fan2024noise}, is satisfied by most common exponential family distributions, including Beta, Bernoulli, Poisson, and Gaussian.  
Assumption~\ref{assump:network} is a technical assumption similar to \citet{schmidt2020nonparametric}, which specifies the requirements on the network parameters relative to the sample size $n$ and smoothness parameters $\phi_n$, $t_i$ and $m_i$, including the number of layers $L$ and the sparsity $s$. 
Assumption~\ref{assump:train}, related to  the training dynamics of DNNs, stipulates that DNNs return  an estimator   sufficiently close to the global minimum.  
We note that the rate in Assumption~\ref{assump:train} pertains to the prediction error in Theorem~\ref{thm:bound_Rn}. Our intent is not to assert an explicit convergence rate for SGD in practice, but to impose a sufficient condition that separates optimization error from the inferential analysis. This perspective follows standard practice in the statistical literature \citep{Chernozhukov2018double,schmidt2020nonparametric,fan2024factor}, where such assumptions serve as technical devices to ensure valid inference rather than as primary objects of study. 
 Assumption~\ref{assump:4} on  $\xi_{1,r}(\xb)$ is not meant to serve as a precise characterization of neural networks. Rather, it ensures that the first-order Hoeffding projection remains non-degenerate and dominates the higher-order components, as in the random forest setting \citep{wager2018estimation}.   We choose the subsample size \( r \) much smaller than $n$ to ensure first-order dominance but large enough to limit bias, 
and let the number of subsamples 
$B$ scale with 
$n$ to control computation.


\begin{theorem}
\label{thm:bound_Rn}
    Suppose that $f_0\in \cG(q,\db,\tb,\mb,K)$ and  $\hf_n\in\cF(L,\pb,s,F)$. If Assumptions~\ref{assump:1}-\ref{assump:network} hold, then it holds that
\begin{align*}
\frac{1}{2}\Delta_{n}(\hf_n)-c'\phi_n L\log^2 n\leq R_n(\hf_n,f_0)\leq 2\Delta_{n}(\hf_n)+C'\phi_n L\log^2 n.
\end{align*}
for some constants $c',C'>0$. Here, $\Delta_{n}(\hf_n)$ is defined in \eqref{eq:obtained_estimator}. 
\end{theorem}
Theorem~\ref{thm:bound_Rn} establishes a bound on the out-of-sample estimation error. Unlike previous studies \citep{schmidt2020nonparametric,fan2024factor}, which primarily focus on bounding the in-sample mean squared error, i.e., $\frac{1}{n} \sum_{i=1}^n (y_i - f(\xb_i))^2$,
to derive bounds for the estimation error, our approach applies to a wide class of training losses \eqref{eq:obtained_estimator} under GNRMs.  The  heteroskedasticity inherent in GNRMs necessitates a more refined analysis, and a truncation technique is introduced to address this challenge by separating the noise into different regimes.  
Specifically, 
this result enables the analysis of  non-Gaussian outcomes with heteroskedasticity  depending on covariates, overcoming key limitations of prior methods  \citep{schmidt2020nonparametric, fan2024factor} that rely on the independence of $y_i-\EE(y_i|\xb_i)$ and $f_0(\xb_i)$ to get estimation error bounds.

\begin{proposition}
\label{prop:bias_honest} 
We suppose that   the conditions of Theorem~\ref{thm:bound_Rn} hold, the neural network fitted on subsamples belongs to $\cF(L,\pb,s,F)$,   Assumptions~\ref{assump:network}-\ref{assump:4} hold for each subsample with size $r$, and sparsity $s\asymp r^{1+\delta}\phi_r\log r$ for some constant $\delta>0$.   
Then for a fixed  $\xb_*$, if the pointwise bias is dominated by the best approximation bias, i.e., there exists constant $C(\xb_*)$ only depends on $\xb_*$ such that $\big| \EE\hf^b(\xb_*)-f_0(\xb_*)\big|\precsim C(\xb_*) \inf_{f\in\mathcal{F}}
\bigl|f(\mathbf{x}_*)-f_0(\mathbf{x}_*)\bigr|$, it holds that
$    \sqrt{\frac{n}{r^2\inf_{\xb\in\cX}\xi_{1,r}(\xb)}}\cdot  \big|\EE \hf^B(\xb_*)-f_0(\xb_*)\big|~\to~0$. 
\end{proposition}

Proposition~\ref{prop:bias_honest}  implies asymptotic unbiasedness, i.e., 
 $\EE \{\hf^B(\xb_*)\} -f_0(\xb_*)=o\Big(\sqrt{\frac{r^2\xi_{1,r}(\xb_*)}{n}} \Big)$  for a fixed $\xb_*$.  Thus, the bias of $\hat{f}^B$ is asymptotically negligible, and enables us to focus on $\hf^B(\xb_*)-\EE \{\hf^B(\xb_*)\}$ and establish the following results.

\begin{theorem}
\label{thm:Inference}
Suppose  the data $\cD_n$ are generated from Model \eqref{eq:model} with an unknown $f_0\in \cG(q,\db,\tb,\mb,K)$, the used neural networks belong to $\cF(L,\pb,s,F)$, and  Assumptions 
\ref{assump:1}--\ref{assump:4} hold (Assumption~\ref{assump:network}  holds for each subsample with size $r$).    Then with the results of Proposition~\ref{prop:bias_honest}, there exists a positive sequence, $\delta_1, \delta_2, \ldots,$ with $\delta_n\to0$ and a corresponding set $\cA_{\delta_n}\subseteq \cX$ with $\PP_{\xb}(\cA_{\delta_n})\geq 1-\delta_n$, such that 
 for any fixed $\xb_{*}\in \cA_{\delta_n}$, it holds that 
\begin{align*}
\sqrt{n}\cdot\frac{\hf^B(\xb_{*})-f_0(\xb_{*})}{\sqrt{r^2\xi_{1,r}(\xb_{*})}}~\cvd~\cN(0,1).
\end{align*}   
\end{theorem}

Theorem~\ref{thm:Inference} establishes the asymptotic normality of the ensembled estimator $\hf^B$ for any fixed \(\xb_*\in \cA_{\delta_n}\), i.e., 
the inference result holds with high probability for almost all \( \xb \in \cX \). Here, Theorem~\ref{thm:Inference} holds by the fact that the first-order projection of $\hf^B(\xb_*)-\EE \{\hf^B(\xb_*)\}$ dominates the other terms. Moreover, as the analytic form of  \(\xi_{1,r}(\xb_{*})\) is unavailable, the  following theorem provides a consistent estimator of it.

\begin{theorem}
\label{thm:estimate_xi1r}
Under the same setting as in Theorem~\ref{thm:Inference} and with  $\hsigma_*^2$ defined in \eqref{eq:def_hsigma_*}, it holds that for  any fixed \( \xb_{*} \in \cX \), 
\begin{align*}
    \sqrt{\frac{r^2\xi_{1,r}(\xb_{*})}{n\hsigma_*^2}}~\cvp~1. 
\end{align*}
\end{theorem}
Compared to \citet{wager2018estimation} and \citet{fei2024u}, since we allow the number of subsamples $B$ to grow at the same rate as the sample size $n$, the Monte Carlo bias in the variance estimation becomes non-negligible. As pointed out by \citet{wager2014confidence}, this bias is typically of order $n/B$, and does not vanish when $B \asymp n$.
Different from the correction proposed in \citet{wager2014confidence}, which assumes near asymptotic independence between the inclusion indicators $J_{b i}$ and the estimates $\hat{f}^b$, we propose a bias-correction term that directly accounts for the dependency between subsample structure and estimates. Specifically, we utilize the covariance form of the variance estimator and subtract a U-statistic–based empirical correction that estimates the within-subject variability due to repeated subsampling. This yields a debiased estimator that is asymptotically unbiased even when $B  \asymp n$.    For the confidence interval in \eqref{eq:CI},  Theorem~\ref{thm:estimate_xi1r} and the Slutsky theorem ensure that, as $ n \to \infty $, it holds that
\begin{align*}
    \PP\Big(\EE(y_{\new}|\xb_{*})\in \Big[\hat{L}(\xb_{*}),\hat{U}(\xb_{*})  \Big] \Big)\to 1-\alpha,
\end{align*}
where $
    \hat{L}(\xb_{*})=\psi'\Big(\hf^B(\xb_{*})- z_{1-\alpha/2}\hsigma_* \Big), \quad \hat{U}(\xb_{*})=\psi'\Big(\hf^B(\xb_{*})+z_{1-\alpha/2}\hsigma_* \Big), $ and $ z_{1-\alpha/2} $ is the $(1-\alpha/2)$-th quantile of the standard normal distribution.   In practice, we recommend using this back-transformed interval, as it preserves the parameter’s natural range.

\noindent{\em Remark.}    Our inference framework relies on a well-trained network; when training is insufficient or overly regularized, inference may become unreliable. For analytical tractability, we adopt an idealized $s$-sparse, $F$-bounded network class \citep{schmidt2020nonparametric}, while in practice, regularization techniques such as weight decay and dropout promote approximate sparsity, keeping the trained model close to this class. The optimization gap $\Delta_n(\hat{f}_n)$ in Theorem~\ref{thm:bound_Rn} quantifies the discrepancy between empirical and population risks, filters out degenerate training cases, and accounts for residual optimization and approximation errors. Our theory demonstrates that as long as $\Delta_n(\hat{f}_n)$ remains small, inference validity is preserved, consistent with \citet{Chernozhukov2018double,fan2024factor}.

\section{Numerical Experiments}
\label{sec:simulations}

We conduct simulations under logistic and Poisson generalized nonparametric regression models to evaluate the performance of the ensemble subsampling method (ESM) in point estimation, variance estimation ($\hat{\sigma}_*$), coverage, and interval length. For comparison, we include several alternative methods capable of producing confidence intervals, including HulC \citep{kuchibhotla2024hulc}, Bayesian Neural Networks \citep{sun2021sparse}, Ensemble Methods \citep{lakshminarayanan2017simple}, and the naive bootstrap. Additional experiments are presented in the Supplementary Material, namely,  Section~\ref{sec:simu_binomial} (binomial regression), Section~\ref{sec:simu_compareRF} (comparison with random forests \citep{breiman2001random}), Section~\ref{sec:simu_kernel} (kernel regression), Section~\ref{sec:simu_additionalcheck} (varying covariate dimensions and nonlinearities), and Sections~\ref{subsec:optimization}--\ref{subsec:simu_sensitive_dropout} (sensitivity analyses for optimization parameters, subsample size $r$, number of subsamples $B$, network depth $L$, and dropout rate).


\begin{table}[htbp]
\renewcommand{\arraystretch}{0.5}
\centering
\caption{Simulation results for different regressions under varying sample sizes $n$ and subsampling ratios $r$ using the following metrics:  absolute {\(\text{Bias}_f\)} and {\(\text{MAE}_f\)} are, respectively,  the average and mean absolute  bias between \( \widehat{f}_s^B \) and \( f_0 \);  {\(\text{Bias}_{\psi'}\)} and {\(\text{MAE}_{\psi'}\)}  are, respectively,   the average and mean absolute  bias in the transformed form \( \psi'(\widehat{f}_s^B) \) and \( \psi'(f_0) \);   {EmpSD} denotes the empirical standard deviation of \( \hat{f}_s^B(\mathbf{x}_{\text{test},i}) \) computed across the 300 repetitions and then averaged over the 80 test points;  {SE} and \( \text{SE}_c \) denote the average standard errors without and with bias correction, respectively, averaged over the 80 test points;  {CP} is the average, across test points, of the coverage probability that the 95\% confidence interval \eqref{eq:approx_hfB} contains the truth over 300 repetitions, and {AIL} is the mean length of these intervals.  Numbers in brackets indicate the standard deviation of each metric across test points.}
\label{table:simu_basic}
\setlength{\tabcolsep}{2.7pt} 
\begin{tabular}{rccccccccc}
\hline
\multicolumn{1}{l}{} & $\text{Bias}_f$  & $\text{MAE}_f$ & $\text{Bias}_{\psi'}$  &  $\text{MAE}_{\psi'}$   & EmpSD & SE &$\text{SE}_c$  & CP     & AIL  \\ \hline
\multicolumn{1}{l}{} & \multicolumn{7}{c}{Logistic Model}          \\ \hline
$n=400,r=n^{0.8}$  & 0.11(0.72) & 0.57(0.44) &0.02(0.14) &0.11(0.08) & 0.65 & 0.78 & 0.66 &91.7\% & 0.46(0.13)                   \\ 
$r=n^{0.85}$ & 0.06(0.59) & 0.47(0.36) &0.01(0.12) & 0.10(0.07) &0.55 & 0.68 &0.58& 94.0\% & 0.43(0.11)                  \\ 
$r=n^{0.9}$ & 0.07(0.50) & 0.40(0.31) &0.01(0.10) & 0.08(0.06) & 0.48 & 0.59 & 0.50 & 94.6\% & 0.39(0.09)    \\ 
$n=700,r=n^{0.8}$ & 0.08(0.41)& 0.33(0.26) & 0.02(0.09) & 0.07(0.05) & 0.39 & 0.51 & 0.39 & 93.6\% & 0.31(0.08)  \\
$r=n^{0.85}$ & 0.05(0.36) & 0.28(0.22) & 0.01(0.08) & 0.06(0.05) & 0.34 & 0.44 & 0.34 & 93.9\% &0.28(0.06)    \\
$r=n^{0.9}$  & 0.06(0.33) & 0.26(0.21) & 0.01(0.07) & 0.06(0.05) & 0.31   & 0.40 & 0.31 & 92.9\% & 0.26(0.06)          \\
\hline
\multicolumn{1}{l}{} & \multicolumn{7}{c}{Poisson Model}          \\ \hline
$n=400,r=n^{0.8}$  & -0.32(0.38) & 0.38(0.31) & -0.16(0.23) & 0.22(0.18) & 0.35 & 0.43 & 0.36 & 88.6\% & 0.86(0.45)                \\ 
$r=n^{0.85}$  & -0.23(0.36) & 0.33(0.27) & -0.12(0.25) & 0.21(0.18) & 0.34 & 0.43 & 0.36 & 91.9\% & 0.95(0.52)            \\ 
$r=n^{0.9}$  & -0.14(0.34) & 0.29(0.24) & -0.06(0.25) & 0.19(0.18) & 0.33 & 0.42 & 0.35 & {94.3\%} & 1.02(0.59)             \\
$n=700,r=n^{0.8}$ & -0.17(0.26) & 0.25(0.20)& -0.10(0.19) & 0.16(0.14) & 0.25 & 0.34 & 0.26 & 90.7\% & 0.67(0.34)       \\ 
$r=n^{0.85}$  & -0.10(0.26) & 0.21(0.17) & -0.06(0.19) & 0.15(0.14) & 0.24 & 0.33 & 0.25 & 93.2\% & 0.71(0.38)         \\ 
$r=n^{0.9}$  & -0.04(0.24) & 0.19(0.15) & -0.02(0.19) & 0.14(0.13) & 0.23 & 0.32 & 0.24 & {94.5\%} & 0.72(0.38)       \\ \hline
\end{tabular}
\end{table}
We present our basic settings.    Define 
\(
    g(\xb)=x_1+ 0.25x_2^2 + 0.1\arctan\big(0.5x_3-0.3\big).
\)
Here, $x_j$ is the $j$-th element in the vector $\xb$. We   generate $\xb_i$ from  $\cN(\mathbf{0},\Ib_p)$ with $p=10$, i.e., the first 3 covariates are signals and the rest are noise variables.  For each $i=1, \ldots, n$ and given $\xb_i $, we independently simulate $y_i$ under   the following  models. 

\noindent (logistic model)   
\(
    \PP(y_i=1|\xb_i)=\frac{1}{1+\exp\{-g(\xb_i)\}}, 
\) with $f_0(\xb)=g(\xb)$ and $\psi'(x)=(1+\exp(-x))^{-1}$.

\noindent (Poisson model) With $ k  \in \{0,1,2,\ldots\}$,    \(
    \PP(y_i=k|\xb_i)=\frac{e^{-\lambda(\xb_i)}(\lambda(\xb_i))^k}{k!}, 
\)
where we set $\lambda(\xb_i)=\log(1+e^{g(\xb_i)})$ to ensure that $\lambda(\xb_i)>0$. In this model, we have $f_0(\xb_i)=\log\{ \lambda(\xb_i)\}$ and $\psi'(x)=\exp(x)$. 
For evaluation, we generate a  total of 80 independent test points 
\( \xb_{\text{test},i} \sim \cN(\mathbf{0}, \Ib_{10}), \)  \( i = 1, \dots, 80 \), 
which remain fixed throughout the following experiment. We then apply \eqref{eq:approx_hfB} and \eqref{eq:def_hsigma_*} to obtain the point estimate \( \hat{f}^B(\xb_{\text{test},i}) \) 
and its associated variance estimate \( \hat{\sigma}_{*}(\xb_{\text{test},i}) \). 
To assess the accuracy of the variance estimates and the empirical coverage of the confidence intervals \eqref{eq:CI}, 
we repeat the entire procedure 300 times, each time independently regenerating the training data. 
This yields replicated estimates \( \hat{f}^B_s(\xb_{\text{test},i}) \) and \( \hat{\sigma}_{*,s}(\xb_{\text{test},i}) \) 
for \( s = 1, \ldots, 300 \). Final evaluation is based on the average performance across the fixed test set, 
using the metrics summarized in Table~\ref{table:simu_basic}.

\begin{figure}[t]
    \centering
    \subfloat[Point Estimates and Variations]{\includegraphics[width=0.48\textwidth]{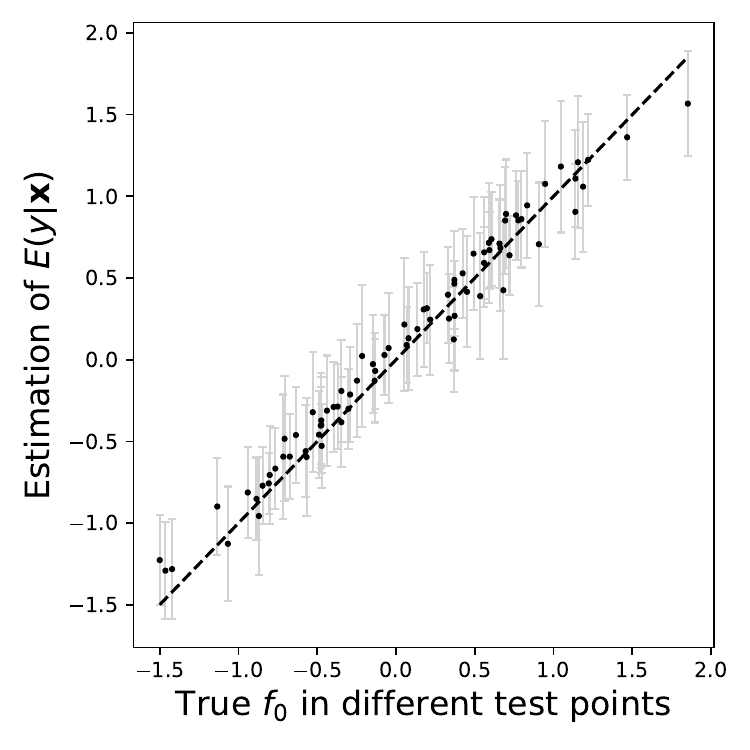}\label{fig_simu:part1CI}}
    \subfloat[Estimated Standard Errors]{\includegraphics[width=0.48\textwidth]{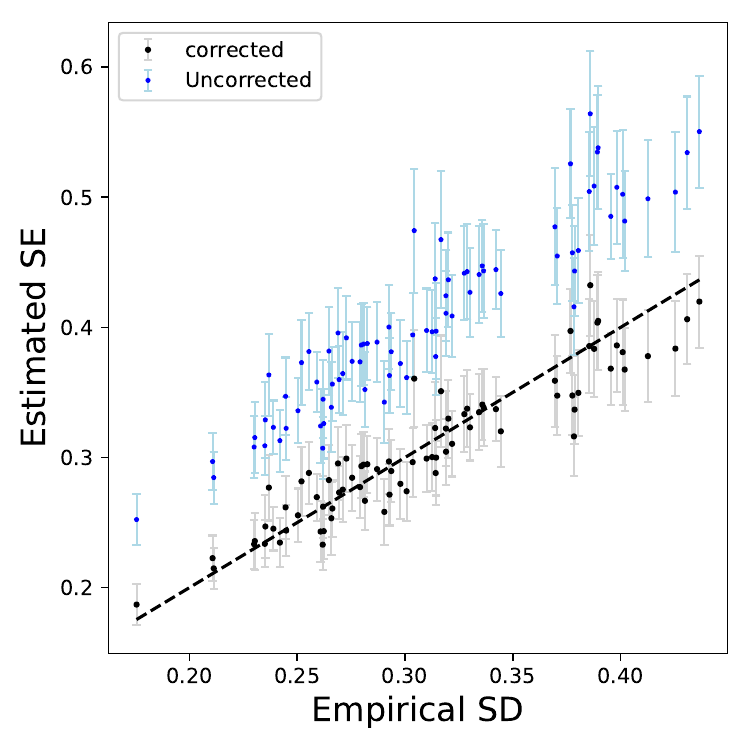}\label{fig_simu:part1SD}}
    \caption{Estimation and inference in simulation samples:  Logistic Model with $n=700$,  $r=n^{0.9}$ and $B=1400$. 
    Figure~\ref{fig_simu:part1CI} shows the average estimated $\EE(y|\xb)$ with variability across 300 repetitions (gray band) over test points. Figure~\ref{fig_simu:part1SD} compares corrected and uncorrected standard errors and their variability (gray and blue bands, respectively) to the empirical standard deviations of estimates across all test samples.}
        \label{fig_simu:part1}
\end{figure}

 In the experiment, we randomly select sub-datasets of size \( r \) to train neural networks, varying \( r \) from \( n^{0.8} \) to \( n^{0.9} \), with \( n = 400 \) and \( 700 \). These choices of $r$ are based on our experiments as they provide enough data for neural network convergence while remaining asymptotically smaller than the total sample size $n$. 
 To ensure computational feasibility, we set \( B = 1400 \)  and use a three-layer DNN with architecture \( \pb = (p, 128, 64, 1) \),  a learning rate of 0.1, and ReLU activation. 
 We apply early stopping (500 epochs), a 10\% dropout rate, $\ell_2$ weight decay with rate 0.02, and output truncation ($F = 3$) to effectively control the sparsity $s$ and boundedness $F$ in the function class  defined in~\eqref{eq:class_F2}, aligning the implemented network with theoretical assumptions.

Table~\ref{table:simu_basic} shows that increasing the sample size \( n \) reduces biases with decreased mean absolute errors (MAE) for both \( f \) and its transformed counterpart \( \psi' \). Figure~\ref{fig_simu:part1CI} displays the average estimated $\mathbb{E}(y|\xb)$ and its variation across 300 repetitions for various test points, in the Bernoulli case with $n = 700$ and $r = n^{0.9}$. The estimated means closely align with the true values.  To further assess the quality of uncertainty quantification, we also examine the empirical standard deviation across repetitions. The corrected standard errors \( \text{SE}_c \) align much more closely with the empirical standard deviations ({EmpSD}) than the uncorrected standard errors ({SE}), confirming the effectiveness of the correction term introduced in \eqref{eq:def_hsigma_*}. This is further illustrated in Figure~\ref{fig_simu:part1SD}, where the variance estimator with bias correction closely matches the empirical standard deviation, while the uncorrected version overestimates the uncertainty. As a result, our method  achieves coverage probabilities (CP) close to the nominal level.

We present the comparison results in Table~\ref{table:comparsion_methods}. For  ESM, we select our results with $\gamma=0.9$. 
For HulC method proposed by \citet{kuchibhotla2024hulc}, although the coverage probability is close to the nominal level of $95\%$, its average interval length is very wide. For  Bayesian neural networks \citep{sun2021sparse}, we observe that the performance is highly sensitive to the tuning of hyperparameters. Due to the difficulty of selecting appropriate priors and optimization settings, the coverage probabilities and average interval lengths in our experiments are worse than those of the proposed ESM, with noticeably wider intervals and less stable coverage. For the ensemble methods \citep{lakshminarayanan2017simple} and the naive bootstrap approach,   the intervals are wider than those of ESM. As the sample size $n$ increases, and the coverage probability tends to exceed the nominal $95\%$. \xr{Overall, ESM achieves competitive coverage performance with narrowest intervals,  although it may not always provide   coverage closest to the nominal level among the compared methods.} \xr{
So, in some cases, using estimators without bias correction or using bootstrap procedures may also have merit, as the presence of bias-correction terms in finite samples can lead to slight undercoverage. }
\begin{table}[t]
\renewcommand{\arraystretch}{0.5}
\centering
\caption{Results of comparison methods.}
\label{table:comparsion_methods}
\setlength{\tabcolsep}{2.7pt} 
\begin{tabular}{cccccccccc}
\hline
\multicolumn{1}{l}{} & $\text{Bias}_f$  & $\text{MAE}_f$ & $\text{Bias}_{\psi'}$  &  $\text{MAE}_{\psi'}$   & EmpSD & SE &$\text{SE}_c$  & CP     & AIL  \\ \hline
\multicolumn{1}{l}{} & \multicolumn{7}{c}{Logistic Model, $n=400$}          \\ \hline
HulC & - & - & - & - & - & -& - & 94.8\% & 0.71(0.21)        \\ 
BNN  & 0.06(0.61) & 0.50(0.35) & 0.02(0.14) & 0.12(0.08) & 0.27 & - & - & 97.8\% & 0.68(0.13)      \\ 
Ensemble & 0.06(0.84) & 0.67(0.52) & 0.01(0.16) & 0.13(0.09) & 0.80 & 0.90 & - & 95.5\% & 0.59(0.15)    \\ 
Naive bp & 0.05(0.81) & 0.64(0.50) & 0.01(0.16) & 0.13(0.09) & 0.77 & 0.85 & - & 95.4\% & 0.57(0.15)  \\ 
\textbf{ESM} & 0.07(0.50) & 0.40(0.31) &0.01(0.10) & 0.08(0.06) & 0.48 & 0.59 & 0.50 & \textbf{94.6\%} & \textbf{0.39(0.09)}    \\ 
\hline
\multicolumn{1}{l}{} & \multicolumn{7}{c}{Logistic Model, $n=700$}          \\ \hline
HulC   & - & - & - & - & - & -& - & 94.9\% & 0.67(0.21)              \\ 
BNN & 0.01(0.64) & 0.55(0.34) & 0.00(0.15) & 0.13(0.07) & 0.11 & - & - & 65.7\% & 0.31(0.10)      \\ 
Ensemble  & 0.04(0.48) & 0.38(0.30) & 0.01(0.10) & 0.08(0.06) & 0.47 & 0.53 & - & 96.8\% & 0.41(0.10)   \\ 
Naive bp  & 0.03(0.48) & 0.37(0.29) & 0.01(0.10) & 0.08(0.06) & 0.46 & 0.52 & - & 96.7\% & 0.40(0.09) \\ 
\textbf{ESM}  & 0.06(0.33) & 0.26(0.21) & 0.01(0.07) & 0.06(0.05) & 0.31   & 0.40 & 0.31 & 92.9\% & \textbf{0.26(0.06)}    \\ 
\hline
\multicolumn{1}{l}{} & \multicolumn{7}{c}{Poisson Model, $n=400$}          \\ \hline
HulC  & - & - & - & - & - & -& - & 90.0\% & 1.40(0.89)               \\ 
BNN & 0.15(0.49) & 0.41(0.30) & 0.03(0.36) & 0.30(0.21) & 0.20 & - & - & 86.3\% & 1.55(1.18)       \\ 
Ensemble & -0.22(0.44) & 0.37(0.31) & -0.09(0.30) & 0.24(0.21) & 0.42 & 0.50 & - & 96.0\% & 1.33(0.71)    \\ 
Naive bp & -0.21(0.44) & 0.37(0.31)  & -0.08(0.30) & 0.24(0.21)  & 0.42 & 0.48 & - & 95.3\% & 1.28(0.69)  \\ 
\textbf{ESM} & -0.14(0.34) & 0.29(0.24) & -0.06(0.25) & 0.19(0.18) & 0.33 & 0.42 & 0.35 & \textbf{94.3\%} & \textbf{1.02(0.59)}       \\ 
\hline
\multicolumn{1}{l}{} & \multicolumn{7}{c}{Poisson Model, $n=700$}          \\ \hline
HulC   & - & - & - & - & - & -& - & 92.8\% & 1.37(0.84)       \\ 
BNN   & 0.26(0.43) & 0.40(0.31)  & 0.12(0.31) & 0.28(0.18) & 0.09 & - & - & 73.2\% & 0.92(0.34)    \\ 
Ensemble & -0.10(0.30) & 0.25(0.20) & -0.04(0.23) & 0.17(0.16) & 0.29 & 0.37 & - & 97.9\% & 1.06(0.56) \\ 
Naive bp & -0.10(0.31)  & 0.25(0.20) & -0.04(0.24) & 0.18(0.16) & 0.30 & 0.35 & - & 97.1\% & 1.03(0.54)   \\ 
\textbf{ESM}  & -0.04(0.24) & 0.19(0.15) & -0.02(0.19) & 0.14(0.13) & 0.23 & 0.32 & 0.24 & \textbf{94.5\%} & \textbf{0.72(0.38)}    \\ 
\hline
\end{tabular}
\end{table}

\begin{table}[t]
\setlength{\tabcolsep}{2.7pt}
\renewcommand{\arraystretch}{0.5}
\centering
\caption{Simulation results for different models under varying sample sizes $n$ and subsampling ratios $r$ with different DNNs.}
\label{table:simu_changeNN}
\begin{tabular}{rccccccccc}
\hline
\multicolumn{1}{l}{} & $\text{Bias}_f$  & $\text{MAE}_f$ & $\text{Bias}_{\psi'}$  &  $\text{MAE}_{\psi'}$   & EmpSD & SE &$\text{SE}_c$  & CP     & AIL  \\ \hline
\multicolumn{1}{l}{} & \multicolumn{7}{c}{Logistic Model}          \\ \hline
$n=400,r=n^{0.8}$ & 0.20(0.51) & 0.44(0.33) & 0.04(0.11) & 0.10(0.07) & 0.49 & 0.80 & 0.51 & 90.2\% & 0.40(0.11)       \\ 
$r=n^{0.85}$  & 0.20(0.52)& 0.44(0.34) & 0.04(0.11) & 0.09(0.07) & 0.50 & 0.86 & 0.52 & 90.2\% & 0.41(0.11)                 \\ 
$r=n^{0.9}$  & 0.19(0.53) & 0.45(0.34) & 0.04(0.11) & 0.10(0.07) & 0.51 & 0.95 & 0.53 & 91.0\% & 0.41(0.11)   \\      
$n=700,r=n^{0.8}$ & 0.16(0.39) & 0.34(0.26) & 0.04(0.09) & 0.08(0.06) & 0.37 & 0.84 & 0.38  & 90.5\% & 0.31(0.08)   \\
$r=n^{0.85}$ & 0.20(0.39) & 0.35(0.27)  & 0.04(0.09) & 0.08(0.06) & 0.37 & 0.78 & 0.38 & 91.1\% & 0.31(0.08)   \\
$r=n^{0.9}$ & 0.16(0.39) &  0.34(0.26) &  0.04(0.09) &  0.07(0.06)  & 0.37 & 0.84 & 0.38 & 91.4\%  & 0.31(0.08)    \\  
\hline
\multicolumn{1}{l}{} & \multicolumn{7}{c}{Poisson Model}          \\ \hline
$n=400,r=n^{0.8}$  & -0.35(0.38) & 0.38(0.28) & -0.19(0.21) & 0.23(0.18) & 0.30 & 0.50 & 0.31 & 86.1\% & 0.70(0.35)         \\ 
$r=n^{0.85}$ & -0.26(0.33) & 0.33(0.26) & -0.14(0.21) & 0.20(0.16) & 0.31 & 0.55 & 0.32 & 90.0\% & 0.78(0.39)     \\ 
$r=n^{0.9}$ & -0.17(0.32) & 0.28(0.23) & -0.08(0.23) & 0.18(0.16) & 0.31 & 0.56 &  0.32 & 91.8\% & 0.86(0.45)       \\ 
$n=700,r=n^{0.8}$ & -0.22(0.24) & 0.26(0.20) & -0.13(0.17) & 0.16(0.14) & 0.23 & 0.47 & 0.23 & 88.0\% & 0.57(0.27)       \\ 
$r=n^{0.85}$ & -0.13(0.25) & 0.21(0.18) & -0.07(0.18) & 0.14(0.13) & 0.23 & 0.51 & 0.23 & 92.3\% & 0.63(0.29)     \\ 
$r=n^{0.9}$ & -0.05(0.25) & 0.20(0.16) & -0.02(0.19) & 0.14(0.13) & 0.23 & 0.56 & 0.24 & 93.4\% & 0.68(0.34)  \\ 
\hline
\end{tabular}
\end{table}
We also investigate  the performance of  narrower yet deeper architectures  under our method. We design a deeper network with \( L=5 \) and \( \pb=(p,10,15,20,30,1) \). As deep networks may cause gradient explosion or vanishing during backpropagation \citep{glorot2010understanding},  we incorporate batch normalization layers after the second and second-to-last layers,  aiming to mitigate gradient instability while maintaining the network's ability to learn complex representations.  
Table~\ref{table:simu_changeNN} shows that key metrics, including empSD, SE, and $\text{SE}_c$, remain largely consistent with those observed for shallower network architectures.
Nevertheless, coverage probabilities drop to around 91\%.    Figure~\ref{fig_simu:part3} presents representative results for $n = 400$ and $r = n^{0.8}$. As shown in Figure~\ref{fig_simu:part3CI}, the fitted values exhibit a slight systematic deviation from the diagonal, indicating bias relative to the true function $f_0$. Nevertheless, Figure~\ref{fig_simu:part3SD} confirms that the corrected variance estimates remain accurate. This bias shifts the confidence intervals, leading to reduced coverage despite precise variance estimation.
 Similar bias-related undercoverage has been reported for random forests in \citet{wager2018estimation}.    In our case, greater network depth may induce complex interactions that, without sufficient  data, amplify estimation bias.
\begin{figure}[t]
    \centering
    \subfloat[Point Estimates with  Intervals]{\includegraphics[width=0.48\textwidth]{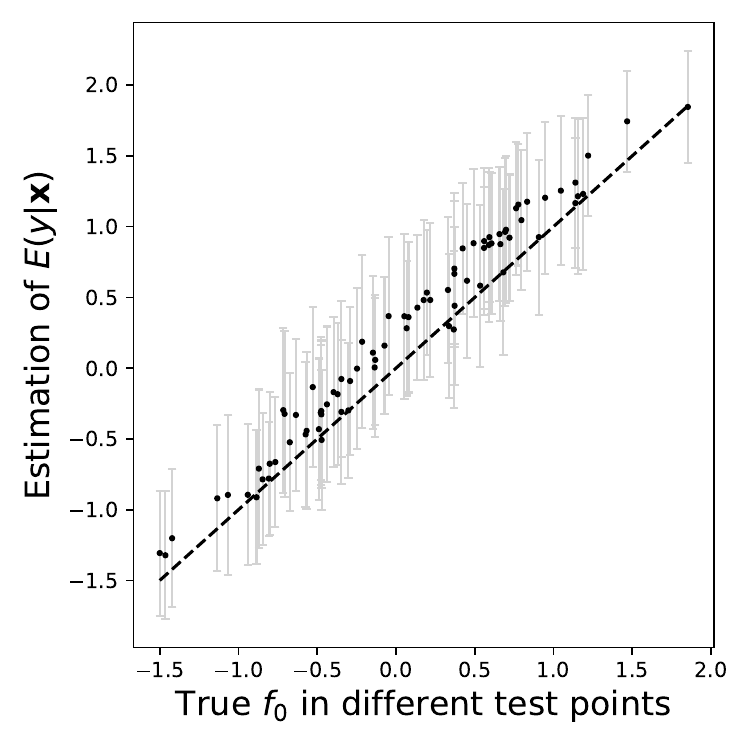}\label{fig_simu:part3CI}}
    \subfloat[Estimated Standard Errors and Variations]{\includegraphics[width=0.48\textwidth]{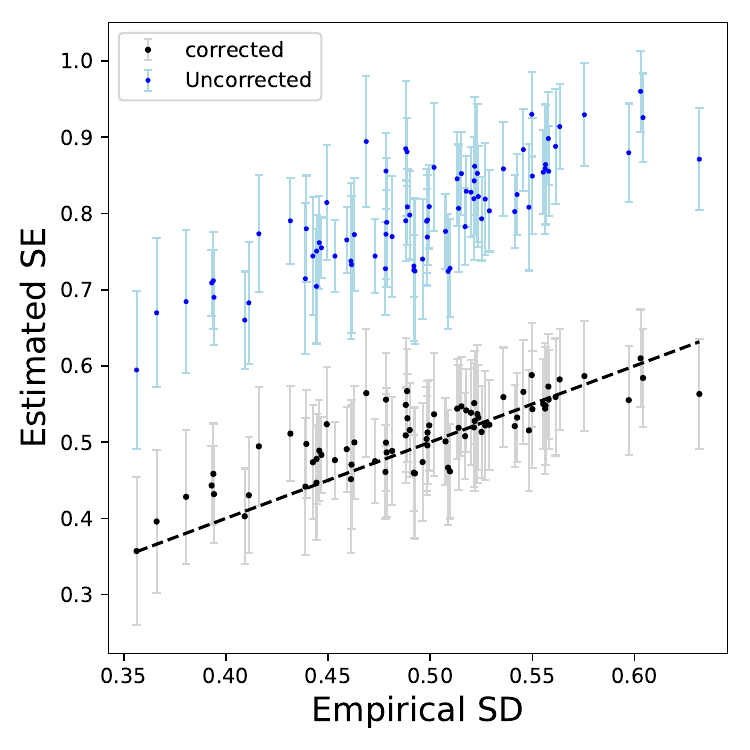}\label{fig_simu:part3SD}}
    \caption{Estimation and inference in simulation samples:  Logistic Model with $n=400$ and $r=n^{0.8}$ under a deeper network.      Figure~\ref{fig_simu:part3CI} displays the average estimated $\EE(y|\xb)$ and its deviation (gray band) across 300 runs at various test points. Figure~\ref{fig_simu:part3SD} compares corrected and uncorrected SEs and their deviations (gray and blue bands) across 300 runs with empirical SDs across test samples.}
        \label{fig_simu:part3}
\end{figure}

\section{Analysis of eICU data}
\label{sec:realdata}
Preventing ICU readmissions is vital as readmitted patients face higher mortality risks and increased expenses.  Identifying high-risk individuals enables more personalized care and optimized resource use. The eICU Collaborative Research Database (eICU-CRD) \citep{pollard2018eICU}  is  a large-scale, multi-center collection of anonymized ICU patient data, created through collaboration between the MIT Laboratory for Computational Physiology and healthcare organizations. It includes detailed information on patients' physiological measurements, clinical data, lab results, medication usage, and outcomes during their ICU stays. Our analysis focuses on  Hospital \#188, a major hospital with 2,632 patients admitted to ICU during 2014--2015. We include ten covariates in our analysis, selected for their clinical relevance and frequent use in prior predictive modeling studies.  The features include six lab tests (Hct, chloride, WBC, Hgb, RBC, glucose), patient age, Emergency Department admission status, and two APACHE-based severity indicators: anticipated  hospital stay and mortality likelihood.
 Among the 2,632 patients admitted to the ICU, 1,476 were not readmitted, 1,046 were readmitted once, 77 were readmitted twice, 26 were readmitted three times, and 7 were readmitted five or more times.
 We apply nonparametric logistic and Poisson models within our ESM framework to estimate, respectively, the probability of ICU readmission and   total number of ICU readmissions using the laboratory and  demographic information. We validate our results using 5-fold cross-validation. The dataset is split into five equal-sized folds;  the model is trained on four folds ($n=2105$), and the ESM method is applied to estimate the subject-specific means and confidence intervals for each subject in the held-out fold, using only covariates while withholding  true outcomes. Repeating this process across all folds yields point estimates and inference intervals for all observations.

\begin{figure}[htbp]
    \centering
   \subfloat[ROC curve in DNNs]{\includegraphics[width=0.48\textwidth]{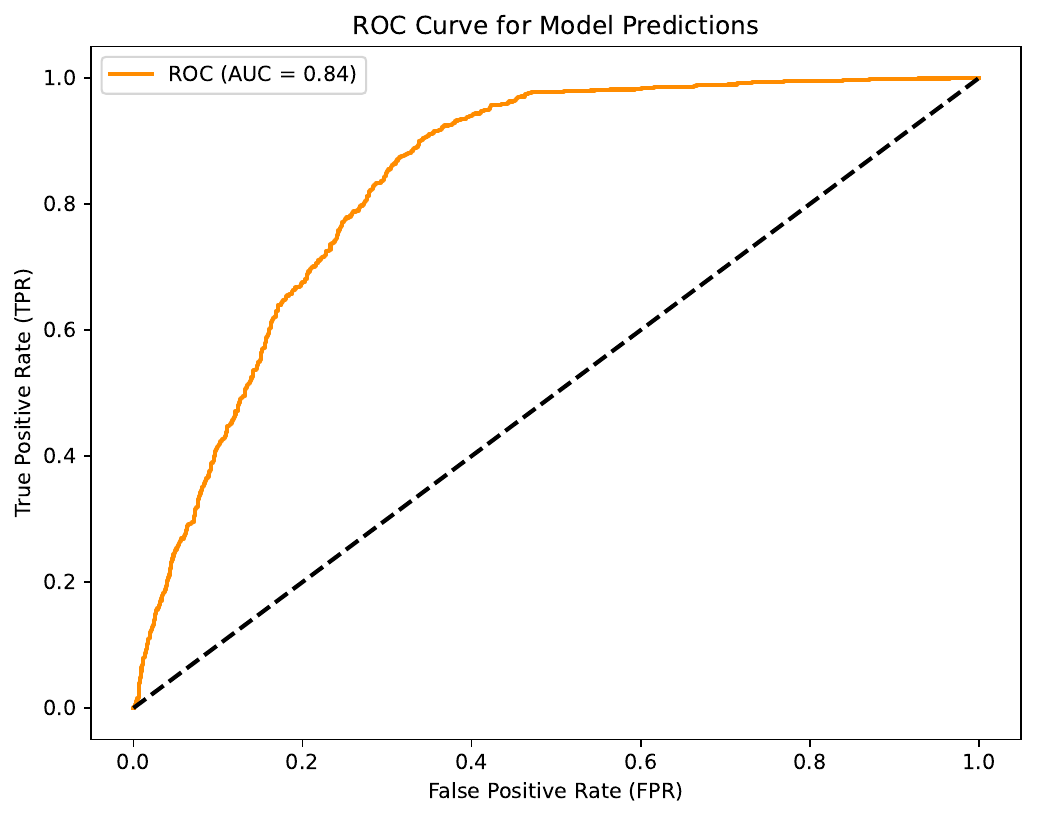}\label{fig_real:rocNNlog}}
    \subfloat[Individual estimates and CIs in DNNs]{\includegraphics[width=0.48\textwidth]{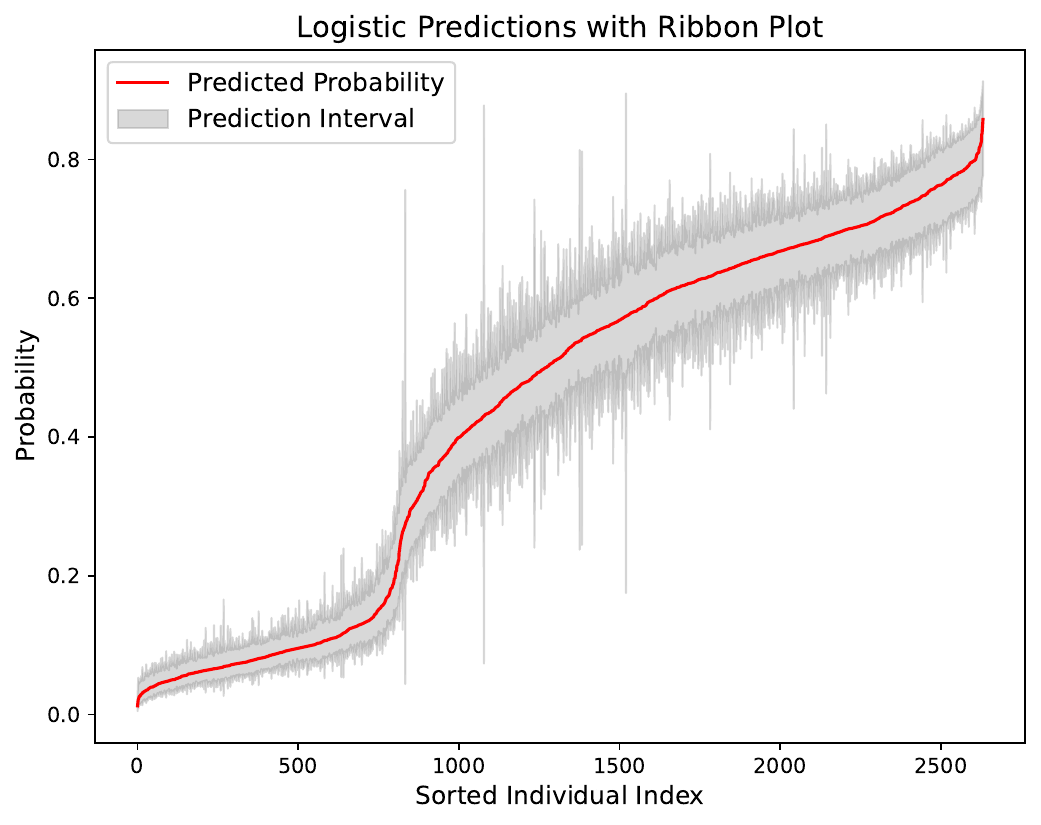}\label{fig_real:calibrationLogisticNN}}

    \subfloat[ROC curve in RFs]{\includegraphics[width=0.48\textwidth]{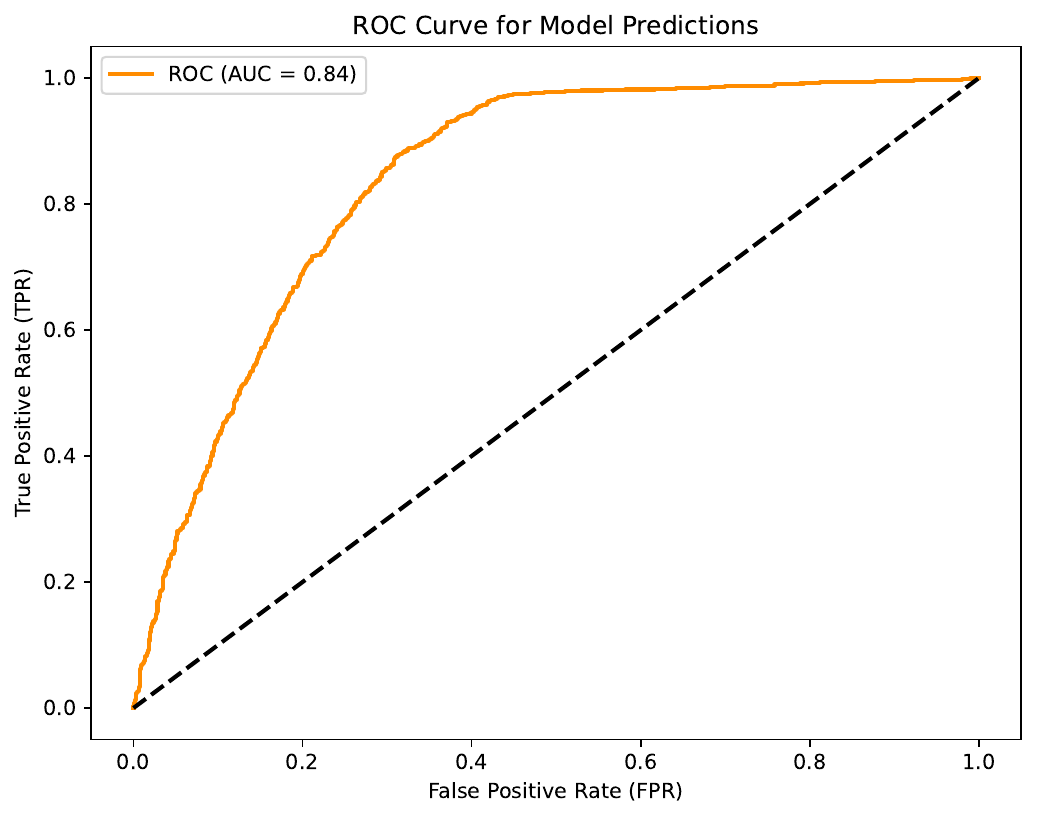}\label{fig_real:rocRFlog}}
    \subfloat[Individual estimates and CIs in RFs]{\includegraphics[width=0.48\textwidth]{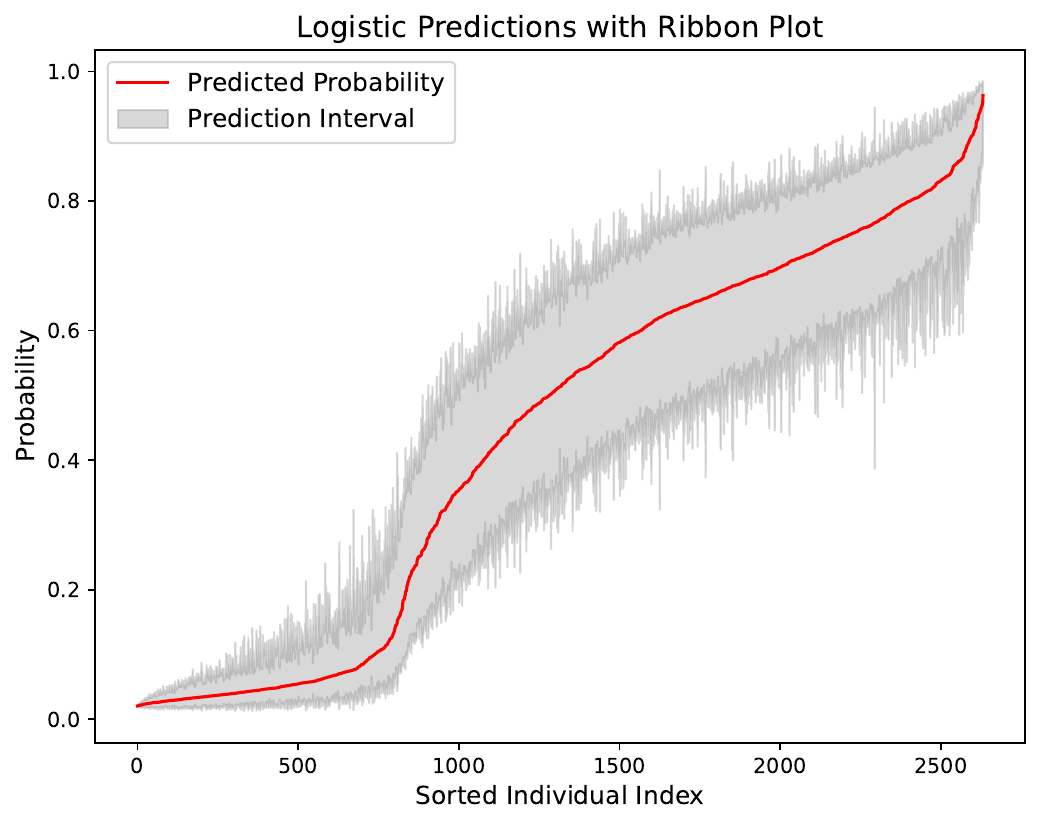}\label{fig_real:calibrationLogisticRF}}
    \caption{ Evaluation of the nonparametric logistic model estimator. Figure~\ref{fig_real:rocNNlog} and \ref{fig_real:rocRFlog} show the ROC curve with both AUC of  0.84. Figure~\ref{fig_real:calibrationLogisticNN} and \ref{fig_real:calibrationLogisticRF} display estimated subject-level probabilities of ICU readmission and confidence intervals, illustrating  heteroskedasticity across individuals.}
    \label{figreal:logisticpart}
\end{figure}

\begin{figure}[htbp]
    \centering
   \subfloat[Lift curve in DNNs]{\includegraphics[width=0.48\textwidth]{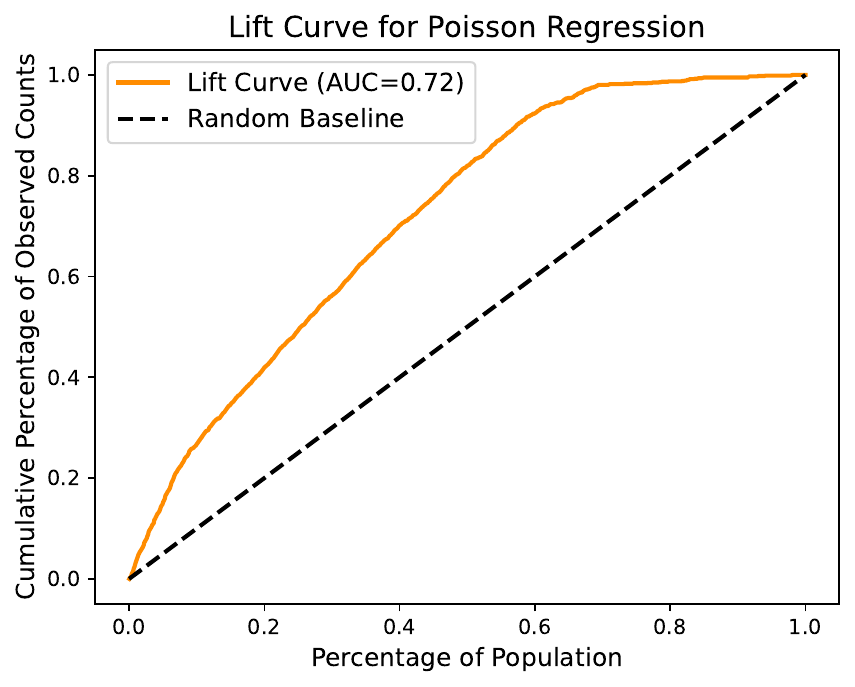}\label{fig_real:rocNNpoi}}
    \subfloat[Individual estimates and CIs in DNNs]{\includegraphics[width=0.48\textwidth]{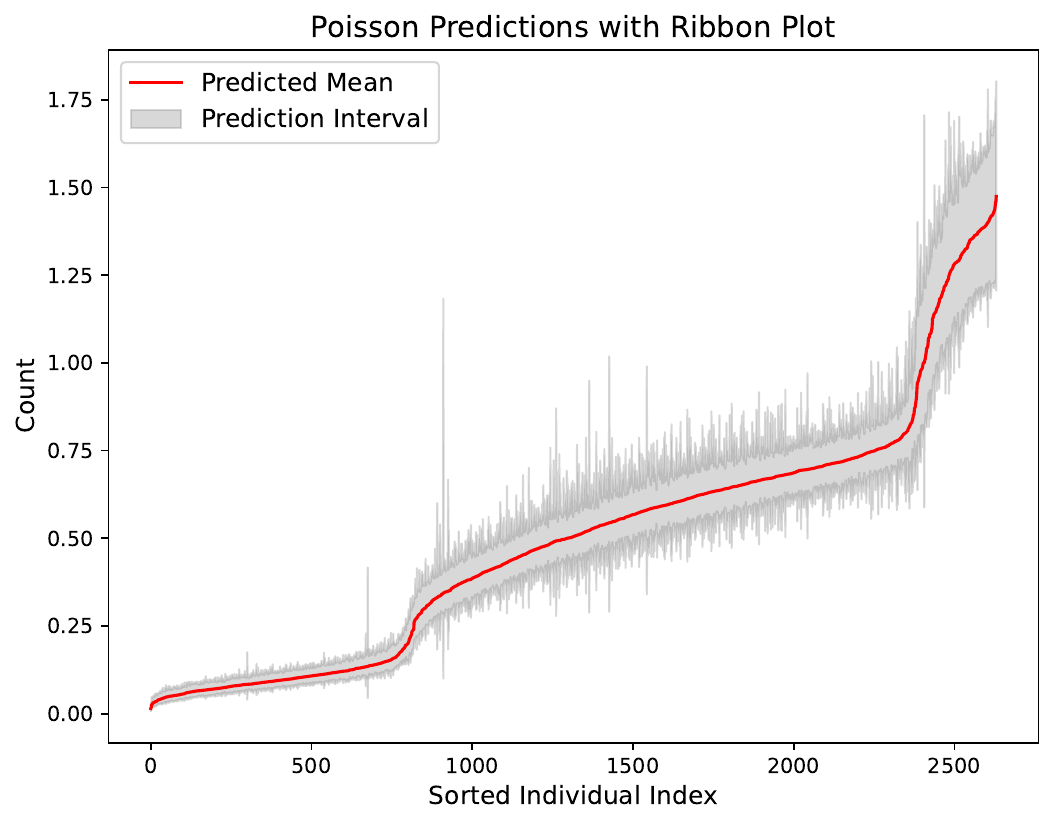}\label{fig_real:calibrationPoissonNN}}

    \subfloat[Lift curve in RFs]{\includegraphics[width=0.48\textwidth]{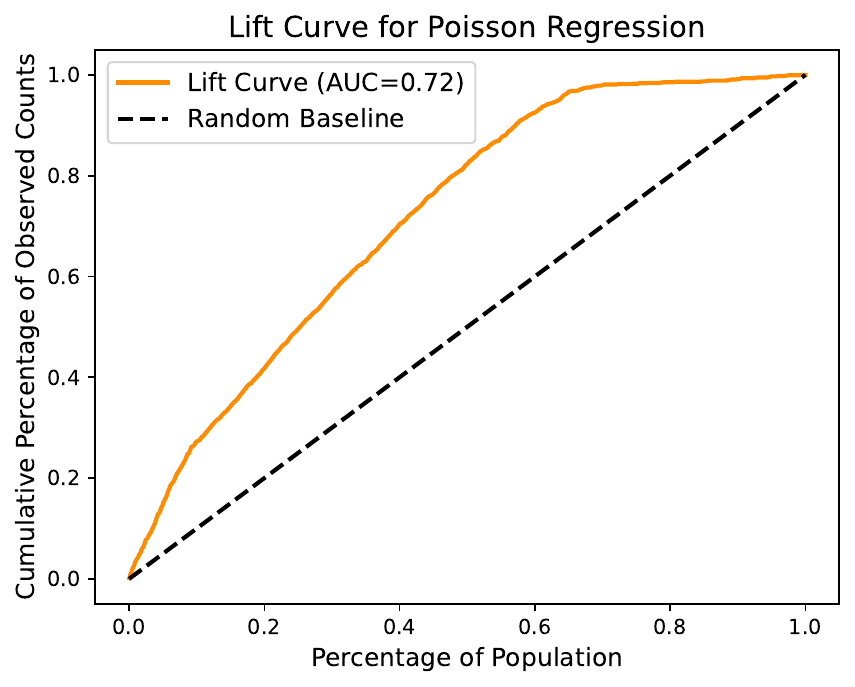}\label{fig_real:rocRFpoi}}
    \subfloat[Individual estimates and CIs in RFs]{\includegraphics[width=0.48\textwidth]{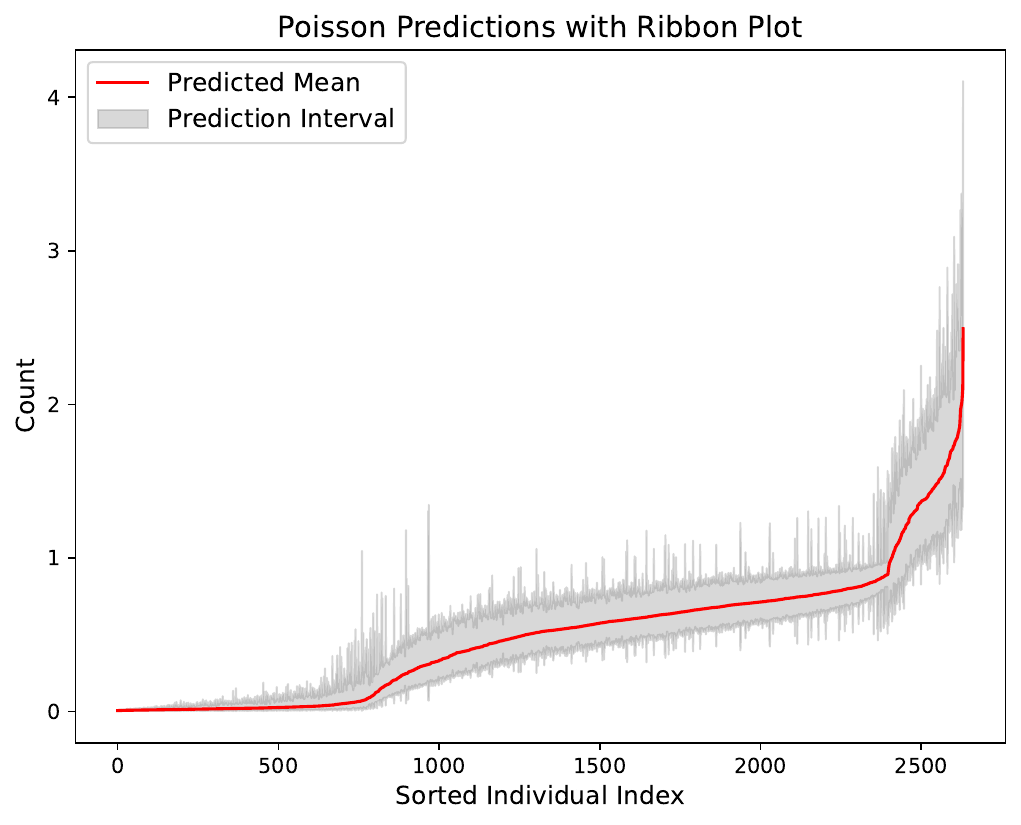}\label{fig_real:calibrationPoissonRF}}
    \caption{ Evaluation of the nonparametric logistic model estimator. Figure~\ref{fig_real:rocNNpoi} and \ref{fig_real:rocRFpoi} show the ROC curve with both AUC of  0.72. Figure~\ref{fig_real:calibrationPoissonNN} and \ref{fig_real:calibrationPoissonRF} display estimated subject-level probabilities of ICU readmission and confidence intervals, illustrating  heteroskedasticity across individuals.}
    \label{figreal:poissonpart}
\end{figure}

To estimate the probability of ICU readmission, we define a binary outcome variable: $y=1$ if any readmission occurs during follow-up, i.e., 2014--2015,   and $y=0$ otherwise. When implementing our method,  we set the subsampling size to $r = n^{0.9} = 979$ and the number of subsampling iterations to $B = 3000$.  We use a $(10, 128, 64, 1)$ architecture with ReLU activation, selected for its balance of accuracy and efficiency via simulations. Training is performed over 300 epochs with a tuned learning rate of $\eta = 0.1$.

 As shown in Figure~\ref{figreal:logisticpart}, both random forests and neural networks achieve similar predictive performance (AUC = 0.84), but their uncertainty estimates differ much. Neural networks yield shorter average confidence intervals (0.123 vs. 0.225 for RFs). Figures~\ref{fig_real:calibrationLogisticNN} and \ref{fig_real:calibrationLogisticRF} highlight these differences by displaying predicted readmission probabilities with 95\% intervals across patients sorted by risk. The variability in interval widths suggests heteroskedasticity, likely driven by differences in data quality or patient characteristics. Notably, Figure~\ref{fig_real:calibrationLogisticNN} reveals a subset of patients (indices between 800–1600) with unusually high uncertainty under DNNs; further details are provided in Section~\ref{subsec:additionalcheck}.

 To estimate subject-level ICU readmissions during 2014–2015, we fit a nonparametric Poisson regression using a DNN to model \(\mathbb{E}(y|\mathbf{x}_{*})\), offering a complementary view of patient well-being. Figure~\ref{figreal:poissonpart} shows a lift curve around 0.72, indicating comparable predictive accuracy for DNNs and RFs. However, uncertainty estimates differ: DNNs yield average interval lengths of 0.124, compared to 0.274 for RFs. Figures~\ref{fig_real:calibrationPoissonNN} and \ref{fig_real:calibrationPoissonRF} illustrate individual predictions with 95\% intervals, again showing heteroskedasticity. Additional analyses (Section~\ref{subsec:realdata_addfeature}) reveal that in some settings, RFs produce narrower intervals than DNNs, highlighting model sensitivity to data structure and signal-to-noise characteristics.   We also assess the transferability of ESM by applying models trained on hospital~\#188 directly to data from hospital~\#458; see Section~\ref{sec:realdata_transfer} in the Supplementary Material. The results indicate that the model remains transferable in this setting.
\section{Discussion}
\label{sec:discussion}
This paper presents a general framework for inference on subject-level means estimated by deep neural networks for categorical and exponential family outcomes. We construct a DNN estimator by minimizing the loss function induced from the generalized nonparametric regression model (GNRM). To address a key gap in leveraging this general loss function, we analyze the convergence rates of DNNs under the GNRM framework and establish connections with U-statistics and Hoeffding decomposition theory.
  Building on these results, we introduce an ensemble subsampling method (ESM) to enable valid inference. Numerical studies and an application to the eICU dataset demonstrate ESM’s utility in quantifying predictive uncertainty and supporting personalized care.

 Our  methodology may have broad applicability  as it enables valid inference for any continuously-differentiable functional of $f_0(x)$, such
 as  
$\Var (y | \xb)
=\psi''(f_0(x))
=:g(f_0(x))$, 
where $g(u)=\psi''(u)$, the second derivative of $\psi(u)$.  We can estimate $\Var (y | \xb)$ via the plug-in estimator $g(\hf^{B}(\xb))$ and apply the delta method to give $\sqrt{\frac{n}{r^2\xi_{1,r}(\xb)}}\cdot (g(\hf^{B}(\xb))-g(f_0(\xb))) =\sqrt{\frac{n}{r^2\xi_{1,r}(\xb)}}\cdot g'(f_0(\xb))(\hf^{B}(\xb)-f_0(\xb))+o_p(1)$, which will lead to estimated confidence intervals.  Our method may also naturally extend to semiparametric causal inference settings where the target parameter can be expressed as a functional of the conditional means. 
Specifically, to estimate the Conditional Average Treatment Effect (CATE),  
\(
\tau(\xb) = \EE[y(1)|\xb] - \EE[y(0)|\xb]
\), 
where $y(1)$ and $y(0)$ denote the potential outcomes under treatment and control,
we apply ESM separately to treated and control groups under standard unconfoundedness. Multiple subsamples are drawn for each group, predictive models are trained to estimate $\EE[y(g)|\xb]$, $g=0,1$, and ensemble means $\hat{\mu}_1(\xb)$ and $\hat{\mu}_0(\xb)$ are obtained. The CATE estimator is  
\(
\hat{\tau}(\xb) = \hat{\mu}_1(\xb) - \hat{\mu}_0(\xb).
\)
ESM’s variance framework provides inference by obtaining $\hat{\sigma}_g^2(\xb)=\hat{\Var}\{\hat{\mu}_g(\xb)\}$ for $g=0,1$, yielding pointwise confidence intervals
\(
\hat{\tau}(\xb) \pm z_{1-\alpha/2} \sqrt{\hat{\sigma}_1^2(\xb) + \hat{\sigma}_0^2(\xb)}.
\) Detailed justifications warrant  investigation.

Our theoretical analysis focuses on ReLU activations, where monomial-based approximations are tractable.  
We envision that ESM can be extended to networks with smooth activation functions such as tanh or GELU, although this requires careful consideration. Unlike ReLU, smooth activations are less flexible in local adaptations,  which may result in slower approximation rates and influence convergence behavior. To preserve inference validity in this setting, it may be necessary to employ alternative analytical tools or apply stronger regularization to control model complexity and prevent saturation.  Moreover, Our theoretical analysis relies on several assumptions that are broadly consistent with recent work on implicit bias and convergence in overparameterized models \citep{chizat2020implicit,chatterji2021does,cao2022benign}. We acknowledge, however, that some of these assumptions may appear strong. In particular, the conditions imposed on $\xi_{1,r}(\mathbf{x})$, which ensure the validity of the first-order Hoeffding projection, can be further improved. An important direction for future work is to relax these conditions by linking them more directly to of the neural network class, such as its sparsity or other structural characteristics.

\noindent
{\bf Data Availability Statement.} The data used in this study are publicly available from the eICU Collaborative Research Database (https://eicu-crd.mit.edu/) upon completion of a data use agreement and approval through the PhysioNet credentialed access process.

\noindent
{\bf Acknowledgements.} 
 We are deeply grateful to the Editor, Associate Editor,  two referees and one reproducibility reviewer for insightful reviews that have greatly improved the manuscript. This work was supported by NIH grants 2R01CA249096-5 and R01CA269398. 
\bibliographystyle{agsm}
\begingroup
\baselineskip=17.5pt
\bibliography{JASAref}

@article{brown1986fundamentals,
  title={Fundamentals of statistical exponential families: with applications in statistical decision theory},
  author={Brown, Lawrence D},
  journal={Lecture Notes-Monograph Series},
  volume={9},
  pages={i--279},
  year={1986},
  publisher={JSTOR}
}

@article{van2019calibration,
  title={Calibration: the Achilles heel of predictive analytics},
  author={Van Calster, Ben and McLernon, David J and Van Smeden, Maarten and Wynants, Laure and Steyerberg, Ewout W and Topic Group `Evaluating diagnostic tests and prediction models' of the STRATOS initiative Bossuyt Patrick Collins Gary S. Macaskill Petra McLernon David J. Moons Karel GM Steyerberg Ewout W. Van Calster Ben van Smeden Maarten Vickers Andrew J.},
  journal={BMC Medicine},
  volume={17},
  number={1},
  pages={230},
  year={2019},
  publisher={Springer}
}

@article{lundberg2018explainable,
  title={Explainable machine-learning predictions for the prevention of hypoxaemia during surgery},
  author={Lundberg, Scott M and Nair, Bala and Vavilala, Monica S and Horibe, Mayumi and Eisses, Michael J and Adams, Trevor and Liston, David E and Low, Daniel King-Wai and Newman, Shu-Fang and Kim, Jerry and others},
  journal={Nature Biomedical Engineering},
  volume={2},
  number={10},
  pages={749--760},
  year={2018},
  publisher={Nature Publishing Group UK London}
}

@inproceedings{chizat2020implicit,
  title={Implicit bias of gradient descent for wide two-layer neural networks trained with the logistic loss},
  author={Chizat, Lenaic and Bach, Francis},
  booktitle={Conference on learning theory},
  pages={1305--1338},
  year={2020}
}

@article{chatterji2021does,
  title={When does gradient descent with logistic loss find interpolating two-layer networks?},
  author={Chatterji, Niladri S and Long, Philip M and Bartlett, Peter L},
  journal={Journal of Machine Learning Research},
  volume={22},
  number={159},
  pages={1--48},
  year={2021}
}

@article{yarotsky2017error,
  title={Error bounds for approximations with deep {ReLU} networks},
  author={Yarotsky, Dmitry},
  journal={Neural Networks},
  volume={94},
  pages={103--114},
  year={2017},
  publisher={Elsevier},
  doi={10.1016/j.neunet.2017.07.005}
}

@article{peng2024bias,
  title={Bias, Consistency, and Alternative Perspectives of the Infinitesimal Jackknife},
  author={Peng, Wei and Mentch, Lucas and Stefanski, Len},
  journal={Statistica Sinica},
  year={2024},
  publisher={Statistica Sinica}
}

@article{drucker1996support,
  title={Support vector regression machines},
  author={Drucker, Harris and Burges, Christopher J and Kaufman, Linda and Smola, Alex and Vapnik, Vladimir},
  journal={Advances in Neural Information Processing Systems},
  volume={9},
  pages={155-161},
  year={1996}
}

@article{athey2018approximate,
  title={Approximate residual balancing: debiased inference of average treatment effects in high dimensions},
  author={Athey, Susan and Imbens, Guido W and Wager, Stefan},
  journal={Journal of the Royal Statistical Society Series B: Statistical Methodology},
  volume={80},
  number={4},
  pages={597--623},
  year={2018},
  publisher={Oxford University Press}
}

@article{angelopoulos2021gentle,
  title={A gentle introduction to conformal prediction and distribution-free uncertainty quantification},
  author={Angelopoulos, Anastasios N and Bates, Stephen},
  journal={arXiv preprint arXiv:2107.07511},
  year={2021}
}

@inproceedings{huang2024conformal,
  title={Conformal Prediction for Deep Classifier via Label Ranking},
  author={Huang, Jianguo and Xi, Huajun and Zhang, Linjun and Yao, Huaxiu and Qiu, Yue and Wei, Hongxin},
  booktitle={International Conference on Machine Learning},
  pages={20331--20347},
  year={2024},
}

@article{guo2021inference,
  title={Inference for the case probability in high-dimensional logistic regression},
  author={Guo, Zijian and Rakshit, Prabrisha and Herman, Daniel S and Chen, Jinbo},
  journal={Journal of Machine Learning Research},
  volume={22},
  number={254},
  pages={1--54},
  year={2021}
}

@article{zhang2023bootstrap,
  title={Bootstrap prediction intervals with asymptotic conditional validity and unconditional guarantees},
  author={Zhang, Yunyi and Politis, Dimitris N},
  journal={Information and Inference: A Journal of the IMA},
  volume={12},
  number={1},
  pages={157--209},
  year={2023},
  publisher={Oxford University Press}
}

@article{lecun2015deep,
  title={Deep learning},
  author={LeCun, Yann and Bengio, Yoshua and Hinton, Geoffrey},
  journal={Nature},
  volume={521},
  number={7553},
  pages={436--444},
  year={2015},
  publisher={Nature Publishing Group UK London}
}

@article{nadaraya1964estimating,
  title={On estimating regression},
  author={Nadaraya, Elizbar A},
  journal={Theory of Probability \& Its Applications},
  volume={9},
  number={1},
  pages={141--142},
  year={1964},
  publisher={SIAM}
}

@article{pollard2018eicu,
  title={The {eICU collaborative research database}, a freely available multi-center database for critical care research},
  author={Pollard, Tom J and Johnson, Alistair E W and Raffa, Jesse D and Celi, Leo A and Mark, Roger G and Badawi, Omar},
  journal={Scientific Data},
  volume={5},
  number={1},
  pages={1--13},
  year={2018},
  publisher={Nature Publishing Group}
}

@article{fei2024u,
  title={U-learning for Prediction Inference via Combinatory Multi-Subsampling: With Applications to LASSO and Neural Networks},
  author={Fei, Zhe and Li, Yi},
  journal={arXiv preprint arXiv:2407.15301},
  year={2024}
}

@article{breiman2001random,
  title={Random forests},
  author={Breiman, Leo},
  journal={Machine Learning},
  volume={45},
  number={1},
  pages={5--32},
  year={2001},
  publisher={Springer}
}

@article{hoeffding1992class,
  title={A class of statistics with asymptotically normal distribution},
  author={Hoeffding, Wassily},
  journal={Breakthroughs in Statistics: Foundations and Basic Theory},
  pages={308--334},
  year={1992},
  publisher={Springer}
}

@article{jiao2023deep,
  title={Deep nonparametric regression on approximate manifolds: Nonasymptotic error bounds with polynomial prefactors},
  author={Jiao, Yuling and Shen, Guohao and Lin, Yuanyuan and Huang, Jian},
  journal={The Annals of Statistics},
  volume={51},
  number={2},
  pages={691--716},
  year={2023},
  publisher={Institute of Mathematical Statistics}
}

@article{bhattacharya2024deep,
  title={Deep neural networks for nonparametric interaction models with diverging dimension},
  author={Bhattacharya, Sohom and Fan, Jianqing and Mukherjee, Debarghya},
  journal={The Annals of Statistics},
  volume={52},
  number={6},
  pages={2738--2766},
  year={2024},
  publisher={Institute of Mathematical Statistics}
}

@article{shen2021robust,
  title={Robust nonparametric regression with deep neural networks},
  author={Shen, Guohao and Jiao, Yuling and Lin, Yuanyuan and Huang, Jian},
  journal={arXiv preprint arXiv:2107.10343},
  year={2021}
}

@article{fan2024noise,
  title={How do noise tails impact on deep {ReLU} networks?},
  author={Fan, Jianqing and Gu, Yihong and Zhou, Wen-Xin},
  journal={The Annals of Statistics},
  volume={52},
  number={4},
  pages={1845--1871},
  year={2024},
  publisher={Institute of Mathematical Statistics}
}

@article{wang2024deep,
  title={Deep regression learning with optimal loss function},
  author={Wang, Xuancheng and Zhou, Ling and Lin, Huazhen},
  journal={Journal of the American Statistical Association},
  pages={1--13},
  year={2024},
  publisher={Taylor \& Francis}
}

@book{goodfellow2016deep,
  title={Deep Learning},
  author={Goodfellow, Ian and Bengio, Yoshua and Courville, Aaron and Bengio, Yoshua},
  volume={1},
  number={2},
  year={2016},
  publisher={MIT press Cambridge}
}

@book{boroskikh2020u,
  title={U-statistics in Banach Spaces},
  author={Boroskikh, Yu V},
  year={2020},
  publisher={Walter de Gruyter GmbH \& Co KG}
}

@article{frees1989infinite,
  title={Infinite order {U}-statistics},
  author={Frees, Edward W},
  journal={Scandinavian Journal of Statistics},
  volume={16},
  number={1},
  pages={29--45},
  year={1989},
  publisher={JSTOR}
}

@book{lee2019u,
  title={U-statistics: Theory and Practice},
  author={Lee, A J},
  year={2019},
  publisher={Routledge}
}

@article{lu2021deep,
  title={Deep network approximation for smooth functions},
  author={Lu, Jianfeng and Shen, Zuowei and Yang, Haizhao and Zhang, Shijun},
  journal={SIAM Journal on Mathematical Analysis},
  volume={53},
  number={5},
  pages={5465--5506},
  year={2021},
  publisher={SIAM}
}

@article{kohler2021rate,
  title={On the rate of convergence of fully connected deep neural network regression estimates},
  author={Kohler, Michael and Langer, Sophie},
  journal={The Annals of Statistics},
  volume={49},
  number={4},
  pages={2231--2249},
  year={2021},
  publisher={JSTOR}
}

@article{shen2019nonlinear,
  title={Nonlinear approximation via compositions},
  author={Shen, Zuowei and Yang, Haizhao and Zhang, Shijun},
  journal={Neural Networks},
  volume={119},
  pages={74--84},
  year={2019},
  publisher={Elsevier}
}

@article{kohler2005adaptive,
  title={Adaptive regression estimation with multilayer feedforward neural networks},
  author={Kohler, Michael and Krzy{\.z}ak, Adam},
  journal={Nonparametric Statistics},
  volume={17},
  number={8},
  pages={891--913},
  year={2005},
  publisher={Taylor \& Francis}
}

@article{mccaffrey1994convergence,
  title={Convergence rates for single hidden layer feedforward networks},
  author={McCaffrey, Daniel F and Gallant, A Ronald},
  journal={Neural Networks},
  volume={7},
  number={1},
  pages={147--158},
  year={1994},
  publisher={Elsevier}
}

@book{fan2018local,
  title={Local Polynomial Modelling and Its Applications: Monographs on Statistics and Applied Probability 66},
  author={Fan, Jianqing},
  year={2018},
  publisher={Routledge}
}

@book{takezawa2005introduction,
  title={Introduction to Nonparametric Regression},
  author={Takezawa, Kunio},
  year={2005},
  publisher={John Wiley \& Sons}
}

@book{gyorfi2006distribution,
  title={A Distribution-Free Theory of Nonparametric Regression},
  author={Gy{\"o}rfi, L{\'a}szl{\'o} and Kohler, Michael and Krzyzak, Adam and Walk, Harro},
  year={2006},
  publisher={Springer Science \& Business Media}
}

@article{mentch2016quantifying,
  title={Quantifying uncertainty in random forests via confidence intervals and hypothesis tests},
  author={Mentch, Lucas and Hooker, Giles},
  journal={Journal of Machine Learning Research},
  volume={17},
  number={26},
  pages={1--41},
  year={2016}
}

@article{wang2022quantifying,
  title={Quantifying uncertainty of subsampling-based ensemble methods under a {U}-statistic framework},
  author={Wang, Qing and Wei, Yujie},
  journal={Journal of Statistical Computation and Simulation},
  volume={92},
  number={17},
  pages={3706--3726},
  year={2022},
  publisher={Taylor \& Francis}
}

@inproceedings{schupbach2020quantifying,
  title={Quantifying uncertainty in neural network ensembles using {U}-statistics},
  author={Schupbach, Jordan and Sheppard, John W and Forrester, Tyler},
  booktitle={2020 International Joint Conference on Neural Networks (IJCNN)},
  pages={1--8},
  year={2020},
  organization={IEEE}
}

@inproceedings{alaa2020discriminative,
  title={Discriminative jackknife: Quantifying uncertainty in deep learning via higher-order influence functions},
  author={Alaa, Ahmed and Van Der Schaar, Mihaela},
  booktitle={International Conference on Machine Learning},
  pages={165--174},
  year={2020},
}

@article{kim2020predictive,
  title={Predictive inference is free with the jackknife+-after-bootstrap},
  author={Kim, Byol and Xu, Chen and Barber, Rina},
  journal={Advances in Neural Information Processing Systems},
  volume={33},
  pages={4138--4149},
  year={2020}
}

@article{barber2021predictive,
  title={Predictive inference with the jackknife+},
  author={Barber, Rina Foygel and Candes, Emmanuel J and Ramdas, Aaditya and Tibshirani, Ryan J},
  journal={The Annals of Statistics},
  volume={49},
  number={1},
  pages={486--507},
  year={2021},
  publisher={JSTOR}
}

@article{lei2018distribution,
  title={Distribution-free predictive inference for regression},
  author={Lei, Jing and G’Sell, Max and Rinaldo, Alessandro and Tibshirani, Ryan J and Wasserman, Larry},
  journal={Journal of the American Statistical Association},
  volume={113},
  number={523},
  pages={1094--1111},
  year={2018},
  publisher={Taylor \& Francis}
}

@article{fan2024factor,
  title={Factor augmented sparse throughput deep {ReLU} neural networks for high dimensional regression},
  author={Fan, Jianqing and Gu, Yihong},
  journal={Journal of the American Statistical Association},
  volume={119},
  number={548},
  pages={2680--2694},
  year={2024},
  publisher={Taylor \& Francis}
}

@article{hornik1989multilayer,
  title={Multilayer feedforward networks are universal approximators},
  author={Hornik, Kurt and Stinchcombe, Maxwell and White, Halbert},
  journal={Neural Networks},
  volume={2},
  number={5},
  pages={359--366},
  year={1989},
  publisher={Elsevier}
}

@inproceedings{nair2010rectified,
  title={Rectified linear units improve restricted boltzmann machines},
  author={Nair, Vinod and Hinton, Geoffrey E},
  booktitle={International Conference on Machine Learning},
  pages={807--814},
  year={2010}
}

@article{yan2025deep,
  title={Deep Regression for Repeated Measurements},
  author={Yan, Shunxing and Yao, Fang and Zhou, Hang},
  journal={Journal of the American Statistical Association},
  pages={1--12},
  year={2025},
  publisher={Taylor \& Francis}
}

@article{wager2018estimation,
  title={Estimation and inference of heterogeneous treatment effects using random forests},
  author={Wager, Stefan and Athey, Susan},
  journal={Journal of the American Statistical Association},
  volume={113},
  number={523},
  pages={1228--1242},
  year={2018},
  publisher={Taylor \& Francis}
}

@article{sun2021sparse,
  title={Sparse deep learning: A new framework immune to local traps and miscalibration},
  author={Sun, Yan and Xiong, Wenjun and Liang, Faming},
  journal={Advances in Neural Information Processing Systems},
  volume={34},
  pages={22301--22312},
  year={2021}
}

@article{lakshminarayanan2017simple,
  title={Simple and scalable predictive uncertainty estimation using deep ensembles},
  author={Lakshminarayanan, Balaji and Pritzel, Alexander and Blundell, Charles},
  journal={Advances in Neural Information Processing Systems},
  volume={30},
  year={2017}
}

@article{kuchibhotla2024hulc,
  title={The {HulC}: Confidence regions from convex hulls},
  author={Kuchibhotla, Arun Kumar and Balakrishnan, Sivaraman and Wasserman, Larry},
  journal={Journal of the Royal Statistical Society Series B: Statistical Methodology},
  volume={86},
  number={3},
  pages={586--622},
  year={2024},
  publisher={Oxford University Press UK}
}

@article{schmidt2020nonparametric,
  title={Nonparametric regression using deep neural networks with {ReLU} activation function},
  author={Schmidt-Hieber, Anselm Johannes},
  journal={The Annals of Statistics},
  volume={48},
  number={4},
  pages={1875--1897},
  year={2020},
  publisher={Institute of Mathematical Statistics}
}

@article{cao2022benign,
  title={Benign overfitting in two-layer convolutional neural networks},
  author={Cao, Yuan and Chen, Zixiang and Belkin, Misha and Gu, Quanquan},
  journal={Advances in Neural Information Processing Systems},
  volume={35},
  pages={25237--25250},
  year={2022}
}

@article{Chernozhukov2018double,
    author = {Chernozhukov, Victor and Chetverikov, Denis and Demirer, Mert and Duflo, Esther and Hansen, Christian and Newey, Whitney and Robins, James},
    title = {Double/debiased machine learning for treatment and structural parameters},
    journal = {The Econometrics Journal},
    volume = {21},
    number = {1},
    pages = {C1-C68},
    year = {2018},
}

@book{hastie2009elements,
  title={The Elements of Statistical Learning: Data Mining, Inference, and Prediction},
  author={Hastie, Trevor and Tibshirani, Robert and Friedman, Jerome H and Friedman, Jerome H},
  year={2009},
  publisher={Springer}
}

@string{JMLR  = {Journal of Machine Learning Research}}

@article{fei2021estimation,
  title={Estimation and inference for high dimensional generalized linear models: A splitting and smoothing approach},
  author={Fei, Zhe and Li, Yi},
  journal={Journal of Machine Learning Research},
  volume={22},
  number={58},
  pages={1--32},
  year={2021}
}

@article{wager2014confidence,
  title={Confidence intervals for random forests: The jackknife and the infinitesimal jackknife},
  author={Wager, Stefan and Hastie, Trevor and Efron, Bradley},
  journal={Journal of Machine Learning Research},
  volume={15},
  number={1},
  pages={1625--1651},
  year={2014},
  publisher={JMLR. org}
}

@inproceedings{kou2023benign,
  title={Benign overfitting in two-layer {ReLU} convolutional neural networks},
  author={Kou, Yiwen and Chen, Zixiang and Chen, Yuanzhou and Gu, Quanquan},
  booktitle={International Conference on Machine Learning},
  pages={17615--17659},
  year={2023}
}

@article{bartlett2020benign,
  title={Benign overfitting in linear regression},
  author={Bartlett, Peter L and Long, Philip M and Lugosi, G{\'a}bor and Tsigler, Alexander},
  journal={Proceedings of the National Academy of Sciences},
  volume={117},
  number={48},
  pages={30063--30070},
  year={2020},
  publisher={National Academy of Sciences}
}

@inproceedings{glorot2010understanding,
  title={Understanding the difficulty of training deep feedforward neural networks},
  author={Glorot, Xavier and Bengio, Yoshua},
  booktitle={Proceedings of the thirteenth international conference on artificial intelligence and statistics},
  pages={249--256},
  year={2010},
  organization={JMLR Workshop and Conference Proceedings}
}
\endgroup

\ifincludeappendix
\appendix
\bigskip
\renewcommand{\thetable}{S.\arabic{table}}
\renewcommand{\thefigure}{S.\arabic{figure}}
\setcounter{table}{0}
\setcounter{figure}{0}

\begin{center}
{\large\bf Supplementary Material for  ``Inference for Deep Neural Network Estimators in Generalized Nonparametric Models" }
\end{center}

\section{Additional Numerical Results}
 We conduct additional numerical experiments to further evaluate our proposed method. In Section~\ref{sec:simu_binomial}, we extend the approach to the binomial regression setting. In Section~\ref{sec:simu_compareRF}, we compare neural networks and random forests, showing that their relative performance varies across different scenarios. In Section~\ref{sec:simu_kernel}, we further examine simpler kernel regression models to assess ESM under smoother nonlinear structures.  In Section~\ref{sec:simu_additionalcheck}, we present experiments illustrating how our method performs as data dimensionality and nonlinearity increase.  We also investigate the effects of key optimization and design parameters, including the subsample size $r$, the number of subsamples $B$, the network depth $L$, and the dropout rate, in Sections~\ref{subsec:optimization}, \ref{subsec:simu_sensitive_r}, \ref{subsec:simu_sensitive_B}, \ref{subsec:simu_sensitive_L}, and~\ref{subsec:simu_sensitive_dropout}, respectively.

\subsection{Binomial Regression}
\label{sec:simu_binomial}
We show the results under binomial regression, that is,   with $ k  \in \{0,1,...,n_{\text{trial}}\}$,  
\[
    \PP(y_i=k|\xb_i)={n_{\text{trial}}\choose{k}}p^k\big(\xb^{(i)}\big)\big(1-p\big(\xb^{(i)}\big)\big)^k,\]
    with \(p\big(\xb^{(i)}\big)=\frac{1}{1+\exp\{-g(\xb_i)\}}.
\)
Under this model,  $f_0(\xb)=g(\xb)$ and $\psi'(x)=n_{\text{trial}}(1+\exp(-x))^{-1}$. 

We   use the same simulation framework as described in Section~\ref{sec:simulations}. We set the number of trials in the binomial distribution to $n_{\text{trial}} = 5$, while keeping all other simulation parameters and settings consistent with the baseline experiments. At the same time, we also vary the resampling number $B$ from 1400 to 3000, and present the corresponding results to examine the effect of $B$ on the variance estimation and interval construction.

\begin{table}[htbp]
\renewcommand{\arraystretch}{0.5}
\centering
\caption{Simulation results for Binomial regression under varying sample sizes $n$ and subsampling ratios $r$.}
\label{table:simu_binomial}
\setlength{\tabcolsep}{2.7pt} 
\begin{tabular}{rccccccccc}
\hline
\multicolumn{1}{l}{} & $\text{Bias}_f$  & $\text{MAE}_f$ & $\text{Bias}_{\psi'}$  &  $\text{MAE}_{\psi'}$   & EmpSD & SE &$\text{SE}_c$  & CP     & AIL  \\ \hline
\multicolumn{1}{l}{} & \multicolumn{7}{c}{$B=1400$}          \\ \hline
$n=400,r=n^{0.8}$ & 0.03(0.34) & 0.27(0.21) & 0.01(0.07) & 0.06(0.04) & 0.31 & 0.39 & 0.32 & 94.0\% & 0.26(0.07)                 \\ 
$r=n^{0.85}$ & 0.04(0.36) & 0.28(0.23) & 0.01(0.07) & 0.06(0.05) & 0.34 & 0.41 &0.34 & {93.5\%} & 0.27(0.08)                  \\ 
$r=n^{0.9}$ & 0.04(0.34) & 0.27(0.21) & 0.01(0.07) & 0.06(0.04) & 0.32 & 0.43 & 0.35 & 95.0\% & 0.28(0.08)        \\ 
$n=700,r=n^{0.8}$ & 0.02(0.27) & 0.21(0.17) & 0.00(0.06) & 0.05(0.03) & 0.25 & 0.34 & 0.25 & 93.2\% & 0.21(0.05)        \\ 
$r=n^{0.85}$  & 0.03(0.25) & 0.20(0.16) & 0.01(0.05) & 0.04(0.03) & 0.24 & 0.34 & 0.25 & {94.3\%} & 0.21(0.05)              \\ 
$r=n^{0.9}$ &0.03(0.26) & 0.20(0.16) & 0.01(0.06) & 0.04(0.03) & 0.25& 0.34 & 0.25 & {93.5\%} & 0.21(0.05)   \\ \hline 
\multicolumn{1}{l}{} & \multicolumn{7}{c}{$B=3000$}          \\ \hline
$n=400,r=n^{0.8}$ & 0.05(0.34) & 0.27(0.21) & 0.01(0.07) & 0.06(0.04) & 0.31 & 0.36& 0.32 & 93.9\% & 0.26(0.07)                 \\ 
$r=n^{0.85}$ & 0.03(0.35) & 0.27(0.22) & 0.01(0.07) & 0.06(0.05) & 0.33 & 0.38 & 0.34 & {94.3\%} & 0.27(0.07)        \\ 
$r=n^{0.9}$ & 0.03(0.34) & 0.27(0.21) & 0.00(0.07) & 0.06(0.04) & 0.32 & 0.38 & 0.34 & 95.0\% & 0.28(0.07)     \\
$n=700,r=n^{0.8}$ & 0.04(0.26) & 0.20(0.16) & 0.01(0.06) & 0.04(0.03) &0.24 & 0.29 & 0.25 & 93.3\% & 0.20(0.05)          \\ 
$r=n^{0.85}$ & 0.01(0.25) & 0.19(0.15) & 0.00(0.05) & 0.04(0.03) & 0.23 & 0.29 & 0.24 & {94.4\%} & 0.20(0.05)          \\ 
$r=n^{0.9}$ & 0.02(0.24) & 0.19(0.15) & 0.00(0.05) & 0.04(0.03) & 0.23 & 0.29 & 0.24 & {94.1\%} & 0.20(0.05)  \\ \hline 
\end{tabular}
\end{table}

\begin{figure}[H]
    \centering
    \subfloat[Estimated Standard Errors]{\includegraphics[width=0.48\textwidth]{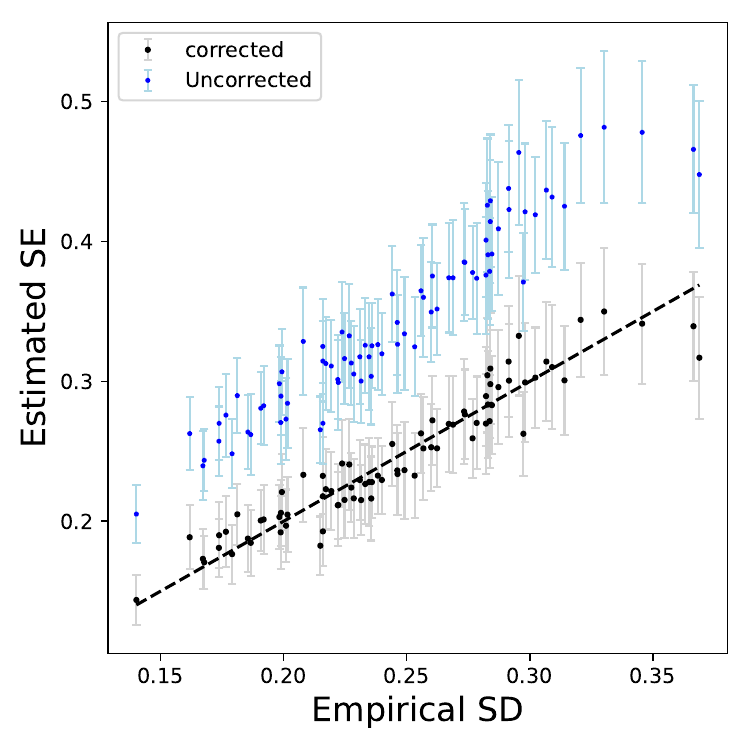}\label{fig_simu:part1sdbino}}
    \subfloat[Estimated Standard Errors]{\includegraphics[width=0.48\textwidth]{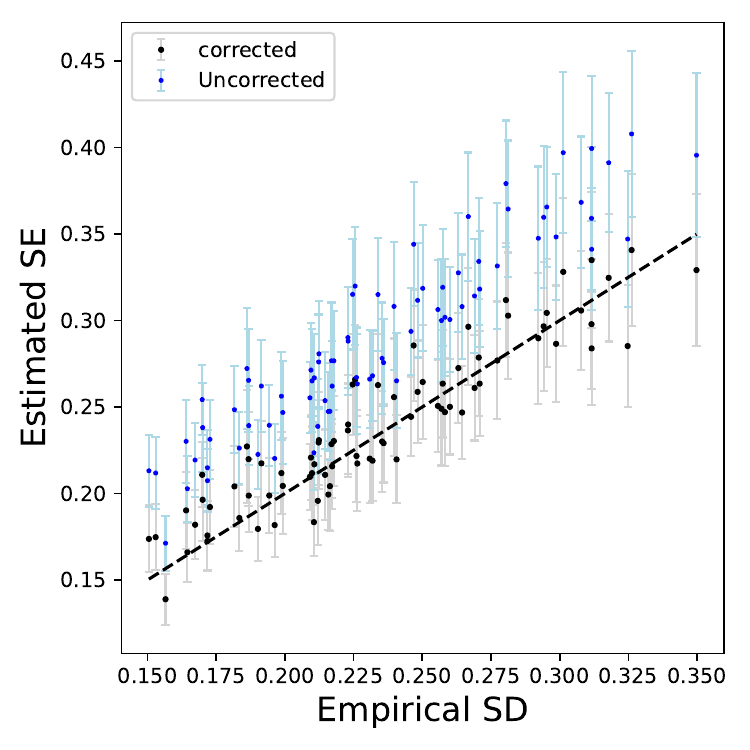}\label{fig_simu:part2sdbino}}
    \caption{Estimation and inference in simulation samples:  Binomial Model with $n=700$ and $r=n^{0.9}$. Figure~\ref{fig_simu:part1sdbino} and \ref{fig_simu:part2sdbino} show corrected and uncorrected estimation SE versus empirical SD of all test samples with $B=1400$ and $B=3000$, respectively.}
        \label{fig_simu:partbino}
\end{figure}

As shown in Table~\ref{table:simu_binomial}, our method achieves reasonable coverage probabilities and stable interval lengths under the binomial regression setting, across various values of $B$ and $r$. In binomial regression, the corrected standard errors $\text{SE}_c$ closely track the empirical variability, supporting the validity of our variance estimation procedure. We also plot the corrected and uncorrected estimation SE versus empirical SD with different $B$ in Figure~\ref{fig_simu:partbino}. Figure~\ref{fig_simu:partbino} also shows that with a larger $B$, the uncorrected estimator becomes more accurate and gradually approaches the empirical SD. In contrast, the corrected SE consistently aligns well with the empirical SD across different values of $B$, demonstrating the effectiveness of our correction procedure, especially in finite-sample settings.

\subsection{Comparison to Random Forests}
\label{sec:simu_compareRF}
Although structurally different from deep neural networks (DNNs), random forests (RFs) are another widely used tool for nonparametric estimation. As a benchmark, we compare the performance of neural networks and random forests within our simulation framework to examine their respective behaviors across varying settings, rather than to identify a universally superior method. In addition, we investigate the use of our corrected variance estimator within the RF framework to enhance the accuracy of empirical standard deviation estimation.

\begin{table}[t]
\renewcommand{\arraystretch}{0.5}
\centering
\caption{Comparison under the basic setting  in logistic model.}
\label{table:simu_compare1logistic}
\setlength{\tabcolsep}{2.7pt} 
\begin{tabular}{rccccccccc}
\hline
\multicolumn{1}{l}{} & $\text{Bias}_f$  & $\text{MAE}_f$ & $\text{Bias}_{\psi'}$  &  $\text{MAE}_{\psi'}$   & EmpSD & SE &$\text{SE}_c$  & CP     & AIL  \\ \hline
\multicolumn{1}{l}{} & \multicolumn{7}{c}{Logistic Model with DNNs}          \\ \hline
$n=400,r=n^{0.8}$  & 0.11(0.72) & 0.57(0.44) &0.02(0.14) &0.11(0.08) & 0.65 & 0.78 & 0.66 &91.7\% & 0.46(0.13)                   \\ 
$r=n^{0.85}$ & 0.06(0.59) & 0.47(0.36) &0.01(0.12) & 0.10(0.07) &0.55 & 0.68 &0.58& 94.0\% & 0.43(0.11)                  \\ 
$r=n^{0.9}$ & 0.07(0.50) & 0.40(0.31) &0.01(0.10) & 0.08(0.06) & 0.48 & 0.59 & 0.50 & 94.6\% & 0.39(0.09)   \\     
$n=700,r=n^{0.8}$ & 0.08(0.41)& 0.33(0.26) & 0.02(0.09) & 0.07(0.05) & 0.39 & 0.51 & 0.39 & 93.6\% & 0.31(0.08)  \\
$r=n^{0.85}$ & 0.05(0.36) & 0.28(0.22) & 0.01(0.08) & 0.06(0.05) & 0.34 & 0.44 & 0.34 & 93.9\% &0.28(0.06)    \\
$r=n^{0.9}$  & 0.06(0.33) & 0.26(0.21) & 0.01(0.07) & 0.06(0.05) & 0.31   & 0.40 & 0.31 & 93.5\% & 0.26(0.06)         \\  
\hline
\multicolumn{1}{l}{} & \multicolumn{7}{c}{Logistic Model with RFs}          \\ \hline
$n=400,r=n^{0.8}$  & 0.07(0.39) & 0.31(0.24) & 0.02(0.09) & 0.07(0.05) & 0.31 & 0.38 & 0.33 & 89.1\% &0.28(0.03)     \\ 
$r=n^{0.85}$ & 0.07(0.41) & 0.33(0.25) & 0.02(0.09) & 0.07(0.06) & 0.36 & 0.43 & 0.36 & 91.1\% & 0.31(0.04)                \\ 
$r=n^{0.9}$ & 0.08(0.44) & 0.35(0.27) & 0.02(0.10) & 0.08(0.06) & 0.40 & 0.48 & 0.41 & 93.0\% & 0.35(0.05)  \\ 
$n=700,r=n^{0.8}$ & 0.07(0.33) & 0.27(0.21) & 0.02(0.07) & 0.06(0.05) & 0.28 & 0.36 & 0.28 & 89.1\% & 0.25(0.03)   \\
$r=n^{0.85}$ & 0.06(0.36) & 0.29(0.22) & 0.01(0.08) & 0.06(0.05) & 0.32 & 0.41 & 0.32 & 91.8\% & 0.27(0.03)  \\
$r=n^{0.9}$  & 0.06(0.37) & 0.30(0.23) & 0.01(0.08) & 0.07(0.05) & 0.35 & 0.47 & 0.37 & 94.2\% & 0.31(0.04)  \\   \hline
\end{tabular}
\end{table}

\begin{table}[H]
\renewcommand{\arraystretch}{0.5}
\centering
\caption{Comparison under the basic setting   in Poisson model.}
\label{table:simu_compare1poisson}
\setlength{\tabcolsep}{2.7pt} 
\begin{tabular}{rccccccccc}
\hline
\multicolumn{1}{l}{} & $\text{Bias}_f$  & $\text{MAE}_f$ & $\text{Bias}_{\psi'}$  &  $\text{MAE}_{\psi'}$   & EmpSD & SE &$\text{SE}_c$  & CP     & AIL  \\ \hline
\multicolumn{1}{l}{} & \multicolumn{7}{c}{Poisson Model with DNNs}          \\ \hline
$n=400,r=n^{0.8}$  & -0.32(0.38) & 0.38(0.31) & -0.16(0.23) & 0.22(0.18) & 0.35 & 0.43 & 0.36 & 88.6\% & 0.86(0.45)                \\ 
$r=n^{0.85}$  & -0.23(0.36) & 0.33(0.27) & -0.12(0.25) & 0.21(0.18) & 0.34 & 0.43 & 0.36 & 91.9\% & 0.95(0.52)            \\ 
$r=n^{0.9}$  & -0.14(0.34) & 0.29(0.24) & -0.06(0.25) & 0.19(0.18) & 0.33 & 0.42 & 0.35 & {94.3\%} & 1.02(0.59)                \\
$n=700,r=n^{0.8}$ & -0.17(0.26) & 0.25(0.20)& -0.10(0.19) & 0.16(0.14) & 0.25 & 0.34 & 0.26 & 90.7\% & 0.67(0.34)       \\ 
$r=n^{0.85}$  & -0.10(0.26) & 0.21(0.17) & -0.06(0.19) & 0.15(0.14) & 0.24 & 0.33 & 0.25 & 93.2\% & 0.71(0.38)         \\ 
$r=n^{0.9}$  & -0.04(0.24) & 0.19(0.15) & -0.02(0.19) & 0.14(0.13) & 0.23 & 0.32 & 0.24 & {94.5\%} & 0.72(0.38)        \\   \hline
\multicolumn{1}{l}{} & \multicolumn{7}{c}{Poisson Model with RFs}          \\ \hline
$n=400,r=n^{0.8}$  & -0.03(0.24) & 0.19(0.16) & -0.01(0.17) &  0.13(0.11) & 0.21 & 0.25 & 0.21 & 92.4\% & 0.60(0.21)           \\ 
$r=n^{0.85}$ & -0.03(0.26) & 0.20(0.17) & 0.00(0.18) & 0.14(0.11) & 0.24 & 0.29 & 0.24 & 93.7\% & 0.68(0.25)       \\ 
$r=n^{0.9}$ & -0.02(0.27) & 0.21(0.18) & 0.01(0.20) & 0.15(0.13) & 0.26 & 0.32 & 0.28 & 94.6\% & 0.79(0.30)   \\ 
$n=700,r=n^{0.8}$ & -0.03(0.21) & 0.16(0.14) & -0.01(0.15) & 0.11(0.09) & 0.18 & 0.24 & 0.19 & 92.8\% & 0.52(0.17)  \\ 
$r=n^{0.85}$ & -0.03(0.22) & 0.17(0.14) & 0.00(0.16) & 0.12(0.10) & 0.20 & 0.28 & 0.21 & 94.2\% & 0.59(0.20)   \\ 
$r=n^{0.9}$ & -0.03(0.26) & 0.20(0.17) & 0.00(0.18) & 0.14(0.12) & 0.24 & 0.32 & 0.25 & 94.3\% & 0.69(0.26)  \\ \hline
\end{tabular}
\end{table}

 We first present random forest performance under the basic settings as  in the simulation section of the main text, covering logistic and Poisson scenarios and using the same evaluation metrics as for neural networks. 
\begin{table}[t]
\setlength{\tabcolsep}{2.4pt}
\renewcommand{\arraystretch}{0.5}
\centering
\caption{Comparison under the changing setting  in logistic model.}
\label{table:simu_compare2logistic}
\begin{tabular}{rccccccccc}
\hline
\multicolumn{1}{l}{} & $\text{Bias}_f$  & $\text{MAE}_f$ & $\text{Bias}_{\psi'}$  &  $\text{MAE}_{\psi'}$   & EmpSD & SE &$\text{SE}_c$  & CP     & AIL  \\ \hline
\multicolumn{1}{l}{} & \multicolumn{7}{c}{Logistic Model with DNNs}          \\ \hline
$n=400,r=n^{0.8}$  & 0.07(0.45) & 0.36(0.28) & 0.01(0.09) & 0.07(0.06) & 0.42 & 0.50 & 0.43 & 93.8\% &0.32(0.09)     \\ 
$r=n^{0.85}$ & 0.04(0.40) & 0.32(0.25) & 0.01(0.08) & 0.07(0.05) & 0.38 & 0.45 & 0.38 & 93.9\% & 0.30(0.08)                \\ 
$r=n^{0.9}$ & 0.04(0.37) & 0.29(0.23) & 0.01(0.08) & 0.06(0.05) & 0.34 & 0.41 & 0.35 & 93.8\% & 0.28(0.07)  \\ 
$n=700,r=n^{0.8}$ & 0.05(0.29) & 0.23(0.19) & 0.01(0.06) & 0.05(0.04) & 0.26 & 0.34 & 0.27 & 93.3\% & 0.22(0.06)   \\
$r=n^{0.85}$ & 0.04(0.29) & 0.23(0.19) & 0.01(0.06) & 0.05(0.04) & 0.25 & 0.32 & 0.25 & 90.4\% & 0.21(0.05)  \\
$r=n^{0.9}$  & 0.04(0.29) & 0.22(0.19) & 0.01(0.06) & 0.05(0.04) & 0.22 & 0.31 & 0.23 & 90.0\% & 0.19(0.05)  \\  
\hline
\hline
\multicolumn{1}{l}{} & \multicolumn{7}{c}{Logistic Model with RFs}          \\ \hline
$n=400,r=n^{0.8}$  & 0.02(0.46) & 0.36(0.30) & 0.01(0.09) & 0.07(0.06) & 0.32 & 0.38 & 0.32 & 84.9\% &0.28(0.04)     \\ 
$r=n^{0.85}$ & 0.04(0.46) & 0.36(0.30) & 0.01(0.10) & 0.08(0.06) & 0.35 & 0.42 & 0.36 & 87.6\% & 0.31(0.05)                \\ 
$r=n^{0.9}$ & 0.03(0.48) & 0.37(0.30) & 0.01(0.10) & 0.08(0.06) & 0.38 & 0.48 & 0.41 & 90.8\% & 0.34(0.05)   \\ 
$n=700,r=n^{0.8}$ & 0.03(0.40) & 0.31(0.26) & 0.01(0.08) & 0.06(0.05) & 0.28 & 0.36 & 0.28 & 84.8\% & 0.24(0.03)   \\
$r=n^{0.85}$ & 0.02(0.41) & 0.32(0.26) & 0.01(0.08) & 0.07(0.05) & 0.32 & 0.41 & 0.32 & 88.5\% & 0.27(0.04)  \\
$r=n^{0.9}$  & 0.01(0.42) & 0.33(0.27) & 0.01(0.09) & 0.07(0.05) & 0.34 & 0.47 & 0.37 & 91.3\% & 0.31(0.05)  \\  \hline
\end{tabular}
\end{table}

\begin{table}[H]
\setlength{\tabcolsep}{2.4pt}
\renewcommand{\arraystretch}{0.5}
\centering
\caption{Comparison under the changing setting  in Poisson model.}
\label{table:simu_compare2poisson}
\begin{tabular}{rccccccccc}
\hline
\multicolumn{1}{l}{} & $\text{Bias}_f$  & $\text{MAE}_f$ & $\text{Bias}_{\psi'}$  &  $\text{MAE}_{\psi'}$   & EmpSD & SE &$\text{SE}_c$  & CP     & AIL  \\ \hline
\multicolumn{1}{l}{} & \multicolumn{7}{c}{Poisson Model with DNNs}          \\ \hline
$n=400,r=n^{0.8}$  & -0.20(0.31) & 0.28(0.24) & -0.12(0.24) &  0.19(0.19) & 0.29 & 0.35 & 0.30 & 90.8\% & 0.80(0.52)           \\ 
$r=n^{0.85}$ & -0.12(0.29) & 0.24(0.20) & -0.07(0.24) & 0.17(0.18) & 0.27 & 0.34 & 0.29 & 93.3\% & 0.82(0.58)       \\ 
$r=n^{0.9}$ & -0.06(0.27) & 0.21(0.18) & -0.03(0.24) & 0.16(0.18) & 0.25 & 0.32 & 0.27 & 95.0\% & 0.81(0.59)   \\ 
$n=700,r=n^{0.8}$ & -0.10(0.22) & 0.18(0.15) & -0.07(0.18) & 0.13(0.14) & 0.20 & 0.27 & 0.20 & 92.3\% & 0.59(0.37)  \\ 
$r=n^{0.85}$ & -0.04(0.20) & 0.16(0.13) & -0.03(0.17) & 0.12(0.13) & 0.19 & 0.26 & 0.19 & 94.1\% & 0.59(0.38)   \\ 
$r=n^{0.9}$ & -0.01(0.20) & 0.15(0.12) & -0.01(0.17) & 0.12(0.12) & 0.18 & 0.24 & 0.18 & 93.4\% & 0.58(0.37)  \\ 
\hline
\multicolumn{1}{l}{} & \multicolumn{7}{c}{Poisson Model with RFs}          \\ \hline
$n=400,r=n^{0.8}$  & -0.00(0.24) & 0.19(0.15) & -0.00(0.20) &  0.14(0.14) & 0.21 & 0.25 & 0.21 & 91.2\% & 0.62(0.25)           \\ 
$r=n^{0.85}$ & -0.00(0.26) & 0.20(0.16) & 0.00(0.20) & 0.15(0.14) & 0.23 & 0.28 & 0.24 & 93.3\% & 0.71(0.30)       \\ 
$r=n^{0.9}$ & -0.00(0.28) & 0.22(0.18) & 0.02(0.23) & 0.17(0.16) & 0.26 & 0.33 & 0.28 & 93.4\% & 0.84(0.40)   \\ 
$n=700,r=n^{0.8}$ & 0.00(0.21) & 0.16(0.14) & -0.00(0.17) & 0.12(0.12) & 0.18 & 0.24 & 0.18 & 91.4\% & 0.54(0.21)  \\ 
$r=n^{0.85}$ & 0.01(0.22) & 0.17(0.14) & 0.01(0.18) & 0.13(0.12) & 0.20 & 0.27 & 0.21 & 93.2\% & 0.63(0.26)   \\ 
$r=n^{0.9}$ & -0.00(0.26) & 0.20(0.16) & 0.01(0.20) & 0.15(0.13) & 0.24 & 0.32 & 0.25 & 94.1\% & 0.73(0.33)  \\\hline
\end{tabular}
\end{table}

As shown in Tables~\ref{table:simu_compare1logistic} and \ref{table:simu_compare1poisson},  when the sample size is small, neural networks tend to produce wider estimation intervals compared to random forests, reflecting greater estimation uncertainty. However, as both the sample size $n$ and subsample size $r$ increase, neural networks are able to capture the signal more effectively, leading to interval lengths and coverage that are comparable to, or even better than  those of random forests.  
The  findings are  consistent with the modern  generalization theory of neural networks \citep{cao2022benign,kou2023benign}, which suggests that neural networks can generalize well when data quality is high and sufficient samples are available. In contrast, random forests tend to offer stronger robustness in low-data or high-noise scenarios.

To further elaborate on these points, we conduct additional experiments under a modified simulation setting, where we increase the signal strength and reduce the data noise. Specifically, we generate the covariates $\xb_i$ from a multivariate normal distribution $\mathcal{N}(0, \frac{1}{4} \Ib_{10})$, and define the underlying regression function as
\(
g(\xb) = 2x_1 + x_2^2 + 0.1 \cdot \arctan(x_3 - 0.3).
\)
The other settings remain unchanged. This setup introduces weaker noise component and stronger signal component. 

The simulation results are summarized in Tables~\ref{table:simu_compare2logistic} and \ref{table:simu_compare2poisson}. In the logistic regression setting, neural networks tend to yield smaller estimation bias and narrower confidence intervals, while also achieving slightly higher coverage probabilities compared to random forests. This shows that, in scenarios with relatively strong signal structures and adequate subsample sizes, neural networks will offer improved estimation efficiency. 
For the Poisson model, both methods perform comparably in terms of coverage. However, as the subsample size $r$ increases, the interval lengths produced by neural networks generally decrease, whereas those from random forests tend to increase. 
Overall, both methods have their strengths: RFs offer stability and robustness under noise, while DNNs can capture more complex patterns when data quality allows. 

Finally, in all of these settings, we find that the corrected variance formula performs well for random forests. The corrected standard errors align more closely with the empirical standard deviations compared to the uncorrected ones. As a result, the coverage probabilities achieved by random forests are close to the nominal level.

\subsection{Additional Experiments on Kernel Regressions}
\label{sec:simu_kernel}
We consider a kernelized generalized linear model (GLM) to evaluate ESM under smooth nonlinear structures. 
Let $\gamma>0$ and $k(\xb, \xb')=\exp\{-\gamma\|\xb-\xb'\|_2^2\}$ denote the Gaussian kernel.  To make the kernel GLM computationally tractable for moderate to large sample sizes, we adopt the Nystrom approximation to obtain a finite dimensional feature representation
$\phi(\xb) \in \mathbb{R}^m$, where $k(\xb,\xb')$ can be approximated by  
$k(\xb,\xb') \approx \phi(\xb)^\top \phi(\xb')$.
Under this approximation, the kernel function becomes
$f(\xb) \approx \wb^\top \phi(\xb)$, 
and the regularized empirical loss reduces to 
\begin{align*} 
\hat{\wb}=\argmin_{\wb}\frac{1}{n}\sum_{i=1}^n [\psi(\wb^\top\phi(\xb_i)) - y_i \wb^\top\phi(\xb_i)] + \lambda \|\wb\|_2^2. 
\end{align*} 
This linearized formulation retains the nonlinear representational capacity of the kernel model while allowing efficient estimation via standard GLM solvers. The conditional mean of the response can be written as  $\mathbb{E} (y|\xb) = \psi'(\hat{\wb}^\top \phi(\xb))$. In practice, we set $\gamma=1/p$, where $p$ is the input dimension to adapt the kernel smooth to feature scale.  For the Bernoulli case, we fit a logistic regression with an $L_2$ penalty (inverse regularization strength $C=10$) and the \texttt{lbfgs} solver with a maximum of 300 iterations. For the Poisson case, we employ the same Nystrom mapping followed by a Poisson regression with $L_2$ regularization parameter $\alpha = 10^{-3}$ and 300 iterations. 
\begin{table}[t]
\setlength{\tabcolsep}{2.4pt}
\renewcommand{\arraystretch}{0.5}
\centering
\caption{Simulation Results with Kernel Regression.}
\label{table:simu_kernel}
\begin{tabular}{rccccccccc}
\hline
\multicolumn{1}{l}{} & $\text{Bias}_f$  & $\text{MAE}_f$ & $\text{Bias}_{\psi'}$  &  $\text{MAE}_{\psi'}$   & EmpSD & SE &$\text{SE}_c$  & CP     & AIL  \\ \hline
\multicolumn{1}{l}{} & \multicolumn{7}{c}{Logistic Model}          \\ \hline
$n=400,r=n^{0.8}$ & 0.05(0.53) & 0.43(0.32) & 0.01(0.11) & 0.09(0.07) & 0.50 & 0.59 & 0.51 & 93.3\% & 0.41(0.07)    \\ 
$r=n^{0.85}$ & 0.04(0.62) & 0.50(0.37) & 0.01(0.13) & 0.10(0.08) & 0.59 & 0.68 & 0.60 & 93.5\% & 0.45(0.09)               \\ 
$r=n^{0.9}$ & 0.04(0.69) & 0.56(0.42) & 0.01(0.14) & 0.12(0.08) & 0.67 & 0.79 & 0.69 & 93.9\% & 0.50(0.11)    \\ 
$n=700,r=n^{0.8}$ & 0.04(0.50) & 0.40(0.30) & 0.01(0.11) & 0.08(0.06) & 0.47 & 0.59 & 0.47 & 93.0\% & 0.37(0.07)    \\
$r=n^{0.85}$ & 0.02(0.57) & 0.46(0.35) & 0.00(0.12) & 0.10(0.07) & 0.54 & 0.69 & 0.55 & 93.3\% & 0.42(0.09)  \\
$r=n^{0.9}$ & 0.01(0.63) & 0.51(0.38) & 0.00(0.13) & 0.10(0.08) & 0.60 & 0.78 & 0.62 & 93.8\% & 0.46(0.10)  \\  
\hline
\multicolumn{1}{l}{} & \multicolumn{7}{c}{Poisson Model}          \\ \hline
$n=400,r=n^{0.8}$ & -0.29(0.35) & 0.37(0.27) & -0.18(0.25) & 0.23(0.19) & 0.32 & 0.39 & 0.33 & 84.7\% & 0.83(0.46)       \\ 
$r=n^{0.85}$ & -0.22(0.37) & 0.34(0.26) & -0.13(0.27) &   0.23(0.20) & 0.35 & 0.41 & 0.36 & 89.5\% & 0.98(0.58)             \\ 
$r=n^{0.9}$ & -0.17(0.38) & 0.33(0.25) & -0.09(0.30) & 0.23(0.21) & 0.36 & 0.43 & 0.38 & 91.6\% & 1.12(0.72)   \\ 
$n=700,r=n^{0.8}$ & -0.20(0.28) & 0.28(0.21) & -0.13(0.21) & 0.19(0.16) & 0.26 & 0.34 & 0.27 & 86.7\% & 0.71(0.37)     \\
$r=n^{0.85}$ & -0.14(0.30) & 0.26(0.20) & -0.08(0.23) & 0.18(0.16) & 0.28 & 0.35 & 0.28 & 90.6\% & 0.81(0.45)  \\
$r=n^{0.9}$ & -0.10(0.30) & 0.26(0.19) & -0.05(0.24) & 0.18(0.16) & 0.29 & 0.36 & 0.29 & 91.9\% & 0.87(0.50)   \\  \hline
\end{tabular}
\end{table}

As shown in Table~\ref{table:simu_kernel}, the kernel based models under the ESM framework also achieve empirical coverage probabilities close to the nominal 95\% level. However, compared with the random forest and deep neural network (DNN) results, the kernel approach exhibits a higher $\mathrm{MAE}_f$ and wider average interval lengths, suggesting that the kernel regression is less effective in capturing complex nonlinear structures than the DNN.

\subsection{Additional Experiments on Dimensionality and Nonlinearity}
\label{sec:simu_additionalcheck}
  We investigate how data dimensionality and nonlinearity influence the learning behavior of neural networks. Model performance depends not only on network capacity, but also on data quality, sample size, signal-to-noise ratio, and the complexity of the underlying regression function.
We begin by assessing the effect of data dimensionality, as the inclusion of irrelevant covariates can increase noise and obscure the true signal. To quantify this effect, we compare three scenarios: low ($p=7$), moderate ($p=10$), and high ($p=20$) dimensional settings. The corresponding results are summarized in Tables~\ref{table:simu_NNcomparep1} and~\ref{table:simu_NNcomparep2}.

\begin{table}[t]
\setlength{\tabcolsep}{2.4pt}
\renewcommand{\arraystretch}{0.5}
\centering
\caption{Simulation results with neural networks under $p=20$.}
\label{table:simu_NNcomparep1}
\begin{tabular}{rccccccccc}
\hline
\multicolumn{1}{l}{} & $\text{Bias}_f$  & $\text{MAE}_f$ & $\text{Bias}_{\psi'}$  &  $\text{MAE}_{\psi'}$   & EmpSD & SE &$\text{SE}_c$  & CP     & AIL  \\ \hline
\multicolumn{1}{l}{} & \multicolumn{7}{c}{Logistic Model}          \\ \hline
$n=400,r=n^{0.8}$ &0.18(1.10) & 0.90(0.66) & 0.02(0.18) & 0.15(0.10) & 0.92 & 1.11 & 0.93 & 87.5\% & 0.54(0.20)   \\ 
$r=n^{0.85}$  & 0.16(1.07) & 0.86(0.65) & 0.02(0.18) & 0.15(0.10) & 0.92 & 1.18 & 0.98 & 90.5\% & 0.58(0.20)        \\ 
$r=n^{0.9}$  & 0.13(0.99) & 0.79(0.61) & 0.02(0.17) & 0.14(0.10) & 0.89 & 1.16 & 0.96 & 91.9\% & 0.58(0.19)   \\ 
$n=700,r=n^{0.8}$ & 0.13(0.83) & 0.67(0.50) & 0.02(0.14) & 0.12(0.08) & 0.69 & 0.95 & 0.71 & 88.4\% & 0.46(0.16)   \\
$r=n^{0.85}$ & 0.08(0.69) & 0.55(0.42) & 0.01(0.13) & 0.10(0.08) & 0.62 & 0.86 & 0.64 & 92.1\%& 0.45(0.13)  \\
$r=n^{0.9}$ & 0.05(0.57) & 0.45(0.35) & 0.01(0.11) & 0.09(0.07) & 0.52 & 0.73 & 0.54 & 93.3\% & 0.40(0.11) \\  
$n=2000,r=n^{0.8}$ & 0.08(0.54) & 0.44(0.33) & 0.01(0.10) & 0.08(0.06) & 0.47 & 0.86 & 0.48 & 90.9\% & 0.35(0.11) \\
$r=n^{0.85}$ & 0.03(0.49) & 0.39(0.30) & 0.00(0.10) & 0.08(0.06) & 0.46 & 0.85 & 0.47 & 93.0\% & 0.36(0.10) \\
$r=n^{0.9}$ & 0.02(0.40) & 0.31(0.24) & 0.00(0.08) & 0.07(0.05) & 0.38 & 0.72 & 0.40 & 95.2\% & 0.32(0.08) \\
\hline
\multicolumn{1}{l}{} & \multicolumn{7}{c}{Poisson Model}          \\ \hline
$n=400,r=n^{0.8}$ & -0.60(0.49) & 0.63(0.45) & -0.30(0.29) & 0.33(0.26) & 0.43 & 0.53 & 0.44 & 77.1\% & 0.90(0.54)           \\ 
$r=n^{0.85}$ & -0.49(0.49) & 0.55(0.42) & -0.25(0.31) & 0.30(0.26) & 0.44 & 0.58 & 0.48 & 86.6\% & 1.11(0.66)      \\ 
$r=n^{0.9}$ & -0.37(0.51) & 0.49(0.40) & -0.17(0.35) & 0.29(0.26) & 0.48 & 0.62 & 0.51 & 90.7\% & 1.37(0.85)  \\ 
$n=700,r=n^{0.8}$ & -0.46(0.39) & 0.49(0.35) & -0.25(0.26) & 0.28(0.23) & 0.35 & 0.48 & 0.35 & 78.3\% & 0.79(0.47)   \\ 
$r=n^{0.85}$ & -0.31(0.38) & 0.38(0.30) & -0.18(0.27) & 0.24(0.22) & 0.35 & 0.50 & 0.37 & 88.5\% & 0.95(0.56)   \\ 
$r=n^{0.9}$ & -0.20(0.38) & 0.33(0.27) & -0.11(0.30) & 0.23(0.22) & 0.36 & 0.53 & 0.38 & 92.8\% & 1.11(0.68)    \\
$n=2000,r=n^{0.8}$ & -0.22(0.26) & 0.27(0.20) & -0.14(0.20) & 0.19(0.17) & 0.24 & 0.46 & 0.25 & 86.6\% & 0.67(0.39) \\
$r=n^{0.85}$ & -0.15(0.27) & 0.24(0.19) & -0.10(0.22) & 0.17(0.17) & 0.26 & 0.49 & 0.26 & 91.7\% & 0.78(0.46) \\
$r=n^{0.9}$ & -0.08(0.26) & 0.22(0.17) & -0.05(0.24) & 0.17(0.17) & 0.25 & 0.51 & 0.27 & 94.3\% & 0.85(0.52) \\ \hline
\end{tabular}
\end{table}

\begin{table}[H]
\setlength{\tabcolsep}{2.4pt}
\renewcommand{\arraystretch}{0.5}
\centering
\caption{Simulation results with neural networks under $p=7$.}
\label{table:simu_NNcomparep2}
\begin{tabular}{rccccccccc}
\hline
\multicolumn{1}{l}{} & $\text{Bias}_f$  & $\text{MAE}_f$ & $\text{Bias}_{\psi'}$  &  $\text{MAE}_{\psi'}$   & EmpSD & SE &$\text{SE}_c$  & CP     & AIL  \\ \hline
\multicolumn{1}{l}{} & \multicolumn{7}{c}{Logistic Model}          \\ \hline
$n=400,r=n^{0.8}$ & 0.05(0.48) & 0.38(0.30) & 0.01(0.10) & 0.08(0.06) & 0.44 & 0.53 & 0.45 & 93.3\% & 0.34(0.09)      \\ 
$r=n^{0.85}$ & 0.03(0.43) & 0.34(0.27) & 0.01(0.09) & 0.07(0.06) & 0.40 & 0.47 & 0.40 & 93.7\% & 0.31(0.08)        \\ 
$r=n^{0.9}$ & 0.05(0.38) & 0.30(0.24) & 0.01(0.08) & 0.06(0.05) & 0.34 & 0.43 & 0.36 & 94.6\% & 0.29(0.07)    \\ 
$n=700,r=n^{0.8}$ & 0.04(0.30) & 0.23(0.19) & 0.01(0.06) & 0.05(0.04) & 0.27 & 0.36 & 0.28 & 93.9\% & 0.23(0.06)    \\
$r=n^{0.85}$ & 0.03(0.30) & 0.23(0.19) & 0.01(0.06) & 0.05(0.04) & 0.27 & 0.38 & 0.28 & 93.5\% & 0.23(0.05)     \\
$r=n^{0.9}$ & 0.04(0.31) & 0.24(0.20) & 0.01(0.06) & 0.05(0.04) & 0.26 & 0.35 & 0.27 & 92.2\% & 0.22(0.05)  \\  
\hline
\multicolumn{1}{l}{} & \multicolumn{7}{c}{Poisson Model}          \\ \hline
$n=400,r=n^{0.8}$ & -0.20(0.30) & 0.28(0.23) & -0.12(0.22) & 0.18(0.17) & 0.27 & 0.34 & 0.28 & 90.6\% & 0.74(0.42)             \\ 
$r=n^{0.85}$ & -0.12(0.28) & 0.23(0.20) & -0.07(0.22) & 0.16(0.16) & 0.26 & 0.33 & 0.27 & 93.7\% & 0.78(0.47)      \\ 
$r=n^{0.9}$ & -0.07(0.26) & 0.21(0.17) & -0.04(0.22) & 0.15(0.16) & 0.25 & 0.31 & 0.26 & 94.6\% & 0.78(0.48)    \\ 
$n=700,r=n^{0.8}$ & -0.09(0.20) & 0.17(0.15) & -0.06(0.16) & 0.12(0.12) & 0.19 & 0.26 & 0.19 & 92.8\% & 0.55(0.32)    \\ 
$r=n^{0.85}$ &-0.03(0.19) & 0.15(0.13) & -0.03(0.16) & 0.11(0.11) & 0.17 & 0.26 & 0.18 & 93.2\% & 0.52(0.26)  \\ 
$r=n^{0.9}$ & -0.02(0.20) & 0.15(0.12) & -0.02(0.17) & 0.12(0.12) & 0.18 & 0.26 & 0.19 & 94.4\% & 0.59(0.34)     \\ \hline
\end{tabular}
\end{table}

\begin{table}[H]
\setlength{\tabcolsep}{2.4pt}
\renewcommand{\arraystretch}{0.5}
\centering
\caption{Simulation results with neural networks under different true function $g(\xb)$.}
\label{table:simu_NNcompareg}
\begin{tabular}{rccccccccc}
\hline
\multicolumn{1}{l}{} & $\text{Bias}_f$  & $\text{MAE}_f$ & $\text{Bias}_{\psi'}$  &  $\text{MAE}_{\psi'}$   & EmpSD & SE &$\text{SE}_c$  & CP     & AIL  \\ \hline
\multicolumn{1}{l}{} & \multicolumn{7}{c}{Logistic Model}          \\ \hline
$n=400,r=n^{0.8}$ &0.04(0.79) & 0.61(0.50) & 0.01(0.15) & 0.12(0.09) & 0.72 & 0.90 & 0.75 & 93.5\% & 0.50(0.16)   \\ 
$r=n^{0.85}$  & 0.06(0.70) & 0.55(0.44) & 0.01(0.14) & 0.11(0.09) & 0.65 & 0.78 & 0.66 & 93.6\% & 0.46(0.14)        \\ 
$r=n^{0.9}$  &0.07(0.60) & 0.47(0.38) & 0.01(0.12) & 0.10(0.12) & 0.54 & 0.68 & 0.57 & 93.8\% & 0.42(0.12)    \\ 
$n=700,r=n^{0.8}$ & 0.07(0.51) & 0.40(0.32) & 0.01(0.11) & 0.08(0.07) & 0.45 & 0.58 & 0.45 & 91.9\% & 0.34(0.10)   \\
$r=n^{0.85}$ & 0.07(0.45) & 0.35(0.28) & 0.01(0.10) & 0.07(0.06) & 0.38 & 0.51 & 0.39 & 91.8\%& 0.31(0.09)  \\
$r=n^{0.9}$ & 0.09(0.38) & 0.31(0.24) & 0.02(0.08) & 0.07(0.05) & 0.32 & 0.45 & 0.34 & 92.0\% & 0.28(0.08) \\  
\hline
\multicolumn{1}{l}{} & \multicolumn{7}{c}{Poisson Model}          \\ \hline
$n=400,r=n^{0.8}$ & -0.42(0.50) & 0.50(0.42) & -0.23(0.34) & 0.29(0.29) & 0.42 & 0.52 & 0.44 & 85.6\% & 0.97(0.66)           \\ 
$r=n^{0.85}$ & -0.29(0.48) & 0.42(0.36) & -0.16(0.33) & 0.26(0.26) & 0.41 & 0.52 & 0.43 & 90.3\% & 1.09(0.76)      \\ 
$r=n^{0.9}$ & -0.18(0.45) & 0.37(0.32) & -0.10(0.35) & 0.25(0.27) & 0.40 & 0.51 & 0.42 & 92.7\% & 1.22(0.93)  \\ 
$n=700,r=n^{0.8}$ & -0.21(0.37) & 0.32(0.28) & -0.14(0.27) & 0.21(0.22) & 0.30 & 0.41 & 0.31 & 88.0\% & 0.78(0.52)   \\ 
$r=n^{0.85}$ & -0.12(0.35) & 0.28(0.24) & -0.08(0.28) & 0.20(0.21) & 0.29 & 0.40 & 0.30 & 90.5\% & 0.83(0.57)   \\ 
$r=n^{0.9}$ & -0.05(0.34) & 0.26(0.22) & -0.05(0.27) & 0.19(0.20) & 0.27 & 0.39 & 0.28 & 91.2\% & 0.84(0.58)    \\ \hline
\end{tabular}
\end{table}

 As shown in Table~\ref{table:simu_NNcomparep1}, increasing the dimension to \( p = 20 \) degrades estimation accuracy, with \(\text{MAE}_f\) rising from 0.28 to 0.55 and empirical standard deviation from 0.34 to 0.62 under the logistic regression model with \( n = 700 \). This reflects increased bias and variability, leading to a drop in coverage from 93.7\% to 92.1\%. Similar effects of high-dimensional noise on random forests were observed in \citet{wager2018estimation}. Increasing the sample size to \( n = 2000 \) under \( p = 20 \) improves estimation, restoring coverage to nearly 95\%. Conversely, reducing the dimension to \( p = 7 \) improves neural network performance (Table~\ref{table:simu_NNcomparep2}), with \(\text{MAE}_f\) decreasing to 0.23 and empirical standard deviation to 0.27. The reduced bias yields shorter intervals and coverage close to 95\%.

We also consider another nonlinear function 
\begin{align*}
    g(\xb)= 2 \tanh\bigg( \frac{1.5 x_1 + 0.6(x_2^2 - 1) + 0.4 x_3 \cdot \tanh(x_4) + 0.15 \cdot \sin(x_5)}{2.5} \bigg),
\end{align*}
where $\tanh(x)=(e^x-e^{-x})/(e^x+e^{-x})$.  This function introduces quadratic, $\tanh$, and even interaction components, representing a moderately nonlinear structure that remains learnable for neural networks. All other simulation settings are kept the same as the basic settings as those in  Section~\ref{sec:simulations}.

Table~\ref{table:simu_NNcompareg} shows the results with a different function $g(\xb)$. We find that the corrected standard error estimator remains accurate, but the coverage probability shows a mild drop compared to the basic setting. This phenomenon shows that not only random forests but also neural networks may exhibit bias in their output logits, particularly under complex nonlinear structures. Neural networks, especially in nonlinear and high-dimensional regimes, can incur non-negligible approximation bias due to optimization challenges or limitations in representing the true underlying function $g(\xb)$. Such biases are difficult to eliminate in practice and can result in reduced coverage probability, even when variance estimation is accurate. How to effectively reduce or correct for such bias remains an open problem, and addressing it is key to improving the reliability of inference.

\subsection{Additional Experiments with Optimization}
\label{subsec:optimization}

We examine how optimization hyperparameters, i.e., weight decay and training epochs, affect results and find our inference stable across reasonable settings. In particular, we vary training epochs (300, 700) and weight decay ($2\times10^{-3}$, $2\times10^{-4}$).

\begin{table}[t]
\setlength{\tabcolsep}{2.4pt}
\renewcommand{\arraystretch}{0.5}
\centering
\caption{Simulation results   with different epochs.}
\label{table:simu_optimization_epoch}
\begin{tabular}{rccccccccc}
\hline
\multicolumn{1}{l}{} & $\text{Bias}_f$  & $\text{MAE}_f$ & $\text{Bias}_{\psi'}$  &  $\text{MAE}_{\psi'}$   & EmpSD & SE &$\text{SE}_c$  & CP     & AIL  \\ \hline
\multicolumn{1}{l}{} & \multicolumn{7}{c}{Logistic with Epochs $=300$}          \\ \hline
$n=400,r=n^{0.8}$ & 0.08(0.50) & 0.40(0.31) & 0.01(0.10) &0.08(0.06) & 0.46 & 0.56 & 0.47 & 93.4\% & 0.36(0.09)   \\ 
$r=n^{0.85}$  & 0.07(0.44) & 0.35(0.28) & 0.02(0.09) & 0.07(0.06) & 0.41 & 0.50 & 0.42 & 93.9\% & 0.34(0.08)      \\ 
$r=n^{0.9}$ & 0.07(0.40) & 0.32(0.26) & 0.01(0.09) & 0.07(0.06) & 0.38 & 0.46 & 0.39 & 93.9\% & 0.32(0.07)   \\ 
$n=700,r=n^{0.8}$ & 0.07(0.33) & 0.26(0.21) & 0.01(0.07) & 0.06(0.05) & 0.29 & 0.39 & 0.30 & 92.5\% & 0.25(0.05)    \\
$r=n^{0.85}$ & 0.07(0.31) & 0.25(0.19) & 0.02(0.07) & 0.05(0.04) & 0.27 & 0.36 & 0.27& 91.6\% & 0.23(0.05)   \\
$r=n^{0.9}$ & 0.06(0.30) & 0.24(0.19) & 0.01(0.07) & 0.05(0.04) & 0.25 & 0.35 & 0.26 & 90.0\% & 0.22(0.05)   \\  
\hline
\multicolumn{1}{l}{} & \multicolumn{7}{c}{Logistic with Epochs $=700$}          \\ \hline
$n=400,r=n^{0.8}$ & 0.13(0.86) & 0.69(0.53) & 0.02(0.15) & 0.13(0.09) & 0.76 & 0.93 & 0.78 & 90.5\% & 0.51(0.15)              \\ 
$r=n^{0.85}$ & 0.08(0.72) & 0.57(0.45) & 0.01(0.14) & 0.11(0.08) & 0.67 & 0.83 & 0.70 & 92.8\% & 0.49(0.14)       \\ 
$r=n^{0.9}$ & 0.06(0.61) & 0.48(0.38) & 0.01(0.12) & 0.10(0.07) & 0.58 & 0.72 & 0.60 & 94.1\% & 0.45(0.12)     \\ 
$n=700,r=n^{0.8}$ & 0.07(0.50) & 0.40(0.31) & 0.01(0.10) & 0.08(0.06) & 0.46 & 0.62 & 0.47 & 92.9\% & 0.36(0.09)     \\ 
$r=n^{0.85}$ & 0.05(0.41) & 0.32(0.25) & 0.01(0.09) & 0.07(0.05) & 0.39 & 0.52 & 0.40 & 94.5\% & 0.32(0.08)   \\ 
$r=n^{0.9}$ & 0.05(0.36) & 0.28(0.22) & 0.01(0.08) & 0.06(0.05) & 0.34 & 0.46 & 0.35 & 94.5\% & 0.29(0.07)     \\ \hline
\multicolumn{1}{l}{} & \multicolumn{7}{c}{ Poisson with Epochs $=300$}          \\ \hline
$n=400,r=n^{0.8}$ &-0.23(0.32) & 0.30(0.25) & -0.13(0.22) & 0.19(0.16) & 0.30 & 0.37 & 0.31 & 90.2\% & 0.78(0.40)              \\ 
$r=n^{0.85}$ & -0.15(0.31) & 0.27(0.22) & -0.08(0.23) & 0.18(0.16) & 0.30 & 0.37 & 0.30 & 92.7\% & 0.84(0.45)       \\ 
$r=n^{0.9}$ & -0.08(0.30) & 0.24(0.20) & -0.04(0.23) & 0.17(0.16) & 0.29 & 0.36 & 0.30 & 94.2\% & 0.88(0.48)     \\ 
$n=700,r=n^{0.8}$ & -0.12(0.23) & 0.20(0.16) & -0.07(0.17) & 0.14(0.13) & 0.21 & 0.29 & 0.22 & 92.2\% & 0.60(0.30)    \\ 
$r=n^{0.85}$ & -0.06(0.22) & 0.18(0.14) & -0.04(0.17) & 0.13(0.12) & 0.21 & 0.29 & 0.21 & 93.8\% & 0.62(0.31)     \\ 
$r=n^{0.9}$ & -0.01(0.22) & 0.17(0.14) & -0.01(0.17) & 0.13(0.12) & 0.20 & 0.28 & 0.20 & 93.7\% & 0.62(0.31)      \\ \hline
\multicolumn{1}{l}{} & \multicolumn{7}{c}{ Poisson with Epochs $=700$}          \\ \hline
$n=400,r=n^{0.8}$  & -0.37(0.41) & 0.43(0.34) & -0.19(0.24) & 0.24(0.19) & 0.37 & 0.47 & 0.39 & 87.2\% & 0.91(0.49)            \\ 
$r=n^{0.85}$ & -0.26(0.39) & 0.36(0.30) & -0.13(0.26) & 0.22(0.19) & 0.37 & 0.47 & 0.39 & 92.0\% & 1.02(0.57)      \\ 
$r=n^{0.9}$ &-0.18(0.37) & 0.32(0.26) & -0.08(0.26) & 0.21(0.18) & 0.35 & 0.46 & 0.38 & 93.9\% & 1.10(0.65)     \\ 
$n=700,r=n^{0.8}$ &-0.20(0.29) & 0.27(0.22) & -0.11(0.20) & 0.17(0.15) & 0.27 & 0.37 & 0.28 & 90.3\% & 0.72(0.37)    \\ 
$r=n^{0.85}$  & -0.12(0.27) & 0.23(0.19) & -0.07(0.20) & 0.16(0.14) & 0.26 & 0.36 & 0.27 & 93.1\% & 0.76(0.40)    \\ 
$r=n^{0.9}$  & -0.06(0.26) & 0.21(0.17) & -0.03(0.21) & 0.15(0.14) & 0.25 & 0.35 & 0.26 & 94.3\% & 0.77(0.42)         \\ \hline
\end{tabular}
\end{table}

\begin{table}[t]
\setlength{\tabcolsep}{2.4pt}
\renewcommand{\arraystretch}{0.5}
\centering
\caption{Simulation results   with different weight decay.}
\label{table:simu_optimization_weightdecay}
\begin{tabular}{rccccccccc}
\hline
\multicolumn{1}{l}{} & $\text{Bias}_f$  & $\text{MAE}_f$ & $\text{Bias}_{\psi'}$  &  $\text{MAE}_{\psi'}$   & EmpSD & SE &$\text{SE}_c$  & CP     & AIL  \\ \hline
\multicolumn{1}{l}{} & \multicolumn{7}{c}{Logistic with WeightDecay $=2\times 10^{-3}$}          \\ \hline
$n=400,r=n^{0.8}$ & 0.18(1.61) & 1.28(0.98) & 0.02(0.22) & 0.19(0.11) & 1.25 & 1.54 & 1.27 & 87.5\% & 0.62(0.22)   \\ 
$r=n^{0.85}$ & 0.13(1.30) & 1.02(0.81) & 0.02(0.20) & 0.16(0.10) & 1.08 & 1.36 & 1.13 & 90.9\% & 0.61(0.19)      \\ 
$r=n^{0.9}$ & 0.08(1.06) & 0.83(0.66) & 0.01(0.18) & 0.15(0.10) & 0.94 & 1.19 & 0.97 & 92.9\% & 0.59(0.17)    \\ 
$n=700,r=n^{0.8}$ & 0.10(0.93) & 0.73(0.58) & 0.01(0.15) & 0.13(0.08) & 0.75 & 1.02 & 0.76 & 89.2\% & 0.48(0.15)     \\
$r=n^{0.85}$ & 0.07(0.72) & 0.57(0.45) & 0.01(0.13) & 0.11(0.08) & 0.63 & 0.87 & 0.64 & 92.0\% & 0.45(0.12)   \\
$r=n^{0.9}$ & 0.03(0.59) & 0.46(0.37) & 0.00(0.11) & 0.09(0.07) & 0.54 & 0.76 & 0.55 & 93.8\% & 0.41(0.10)  \\  
\hline
\multicolumn{1}{l}{} & \multicolumn{7}{c}{Logistic with WeightDecay $=2\times 10^{-4}$}          \\ \hline
$n=400,r=n^{0.8}$ & 0.22(1.72) & 1.36(1.07) & 0.02(0.22) & 0.19(0.11) & 1.32 & 1.66 & 1.37 & 88.1\% & 0.64(0.23)     \\ 
$r=n^{0.85}$ & 0.15(1.41) & 1.11(0.88) & 0.02(0.20) & 0.17(0.11) & 1.17 & 1.48 & 1.22 & 91.1\% & 0.64(0.20)      \\ 
$r=n^{0.9}$  & 0.09(1.13) & 0.88(0.71) & 0.01(0.18) & 0.15(0.10) & 1.00 & 1.28 & 1.04 & 93.3\% & 0.61(0.17)     \\ 
$n=700,r=n^{0.8}$ & 0.12(1.00) & 0.80(0.63) & 0.02(0.16) & 0.14(0.09) & 0.80 & 1.10 & 0.82 & 89.1\% & 0.50(0.16)      \\ 
$r=n^{0.85}$ & 0.07(0.78) & 0.61(0.49) & 0.01(0.14) & 0.11(0.08) & 0.67 & 0.93 & 0.69 & 92.0\% & 0.47(0.13)     \\ 
$r=n^{0.9}$ & 0.04(0.64) & 0.50(0.40) & 0.01(0.12) & 0.10(0.07) & 0.59 & 0.83 & 0.60 & 93.6\% & 0.43(0.11)       \\ \hline
\multicolumn{1}{l}{} & \multicolumn{7}{c}{Poisson with WeightDecay $=2\times 10^{-3}$}          \\ \hline
$n=400,r=n^{0.8}$ & -0.67(0.60) & 0.70(0.56) & -0.28(0.25) & 0.31(0.21) & 0.49 & 0.63 & 0.51 & 83.0\% & 0.95(0.51)               \\ 
$r=n^{0.85}$ & -0.49(0.56) & 0.56(0.49) & -0.21(0.27) & 0.27(0.20) & 0.48 & 0.62 & 0.50 & 88.8\% & 1.10(0.60)         \\ 
$r=n^{0.9}$ & -0.35(0.50) & 0.46(0.41) & -0.14(0.29) & 0.25(0.20) & 0.46 & 0.61 & 0.49 & 92.9\% & 1.24(0.72)     \\ 
$n=700,r=n^{0.8}$ & -0.41(0.41) & 0.44(0.37) & -0.20(0.21) & 0.23(0.17) & 0.35 & 0.49 & 0.36 & 84.3\% & 0.77(0.39)      \\ 
$r=n^{0.85}$ & -0.27(0.37) & 0.35(0.30) & -0.13(0.22) & 0.20(0.16) & 0.33 & 0.48 & 0.34 & 90.5\% & 0.84(0.43)      \\ 
$r=n^{0.9}$ & -0.18(0.35) & 0.30(0.26) & -0.08(0.24) & 0.19(0.16) & 0.33 & 0.48 & 0.33 & 92.8\% & 0.90(0.49)      \\ \hline
\multicolumn{1}{l}{} & \multicolumn{7}{c}{Poisson with WeightDecay $=2\times 10^{-4}$}          \\ \hline
$n=400,r=n^{0.8}$ & -0.72(0.64) & 0.75(0.61) & -0.29(0.25) & 0.32(0.22) & 0.52 & 0.65 & 0.53 & 81.6\% & 0.96(0.51)       \\ 
$r=n^{0.85}$ & -0.54(0.59) & 0.60(0.53) & -0.22(0.27) & 0.29(0.21) & 0.51 & 0.65 & 0.52 & 88.4\% & 1.11(0.61)      \\ 
$r=n^{0.9}$ & -0.38(0.52) & 0.48(0.42) & -0.16(0.29) & 0.26(0.21) & 0.47 & 0.64 & 0.51 & 92.7\% & 1.26(0.73)   \\ 
$n=700,r=n^{0.8}$ & -0.44(0.43) & 0.48(0.39) & -0.21(0.21) & 0.24(0.17) & 0.36 & 0.51 & 0.37 & 83.1\% & 0.77(0.39)   \\ 
$r=n^{0.85}$ & -0.30(0.39) & 0.38(0.33) & -0.14(0.23) & 0.21(0.17) & 0.35 & 0.50 & 0.36 & 89.6\% & 0.85(0.44)     \\ 
$r=n^{0.9}$ & -0.20(0.36) & 0.31(0.27) & -0.09(0.24) & 0.19(0.16) & 0.33 &  0.50 & 0.34 & 92.9\% & 0.92(0.50)          \\ \hline
\end{tabular}
\end{table}

As shown in Tables~\ref{table:simu_optimization_epoch} and \ref{table:simu_optimization_weightdecay}, our variance estimator remains accurate across different training epochs and weight decay values, closely matching the empirical standard deviation. This confirms that optimization error is negligible under reasonable hyperparameter settings and supports viewing Assumption~\ref{assump:train} as a technical device separating optimization from inference rather than a practical constraint on SGD convergence. Larger subsample sizes $r$ reduce bias in MAE, consistent with modern generalization theory. Overall, the results validate the robustness of our inference guarantees, while highlighting that extreme hyperparameters (e.g., excessive weight decay or too few epochs) can impair learning and should be avoided.


\subsection{Additional Experiments on $r$}
\label{subsec:simu_sensitive_r}

Recall that $r = n^{\gamma}$. We vary $\gamma$ from 0.65 to 0.95 (step 0.05) while keeping other settings as in Section~\ref{sec:simulations}, and plot the standard deviation, MAE of $\hat{f}^B - f_0$, and coverage probability versus $\gamma$ for $n=400$ and $n=700$ under two models.

\begin{figure}[htbp]
    \centering
   \subfloat[Standard Deviations]{\includegraphics[width=0.32\textwidth]{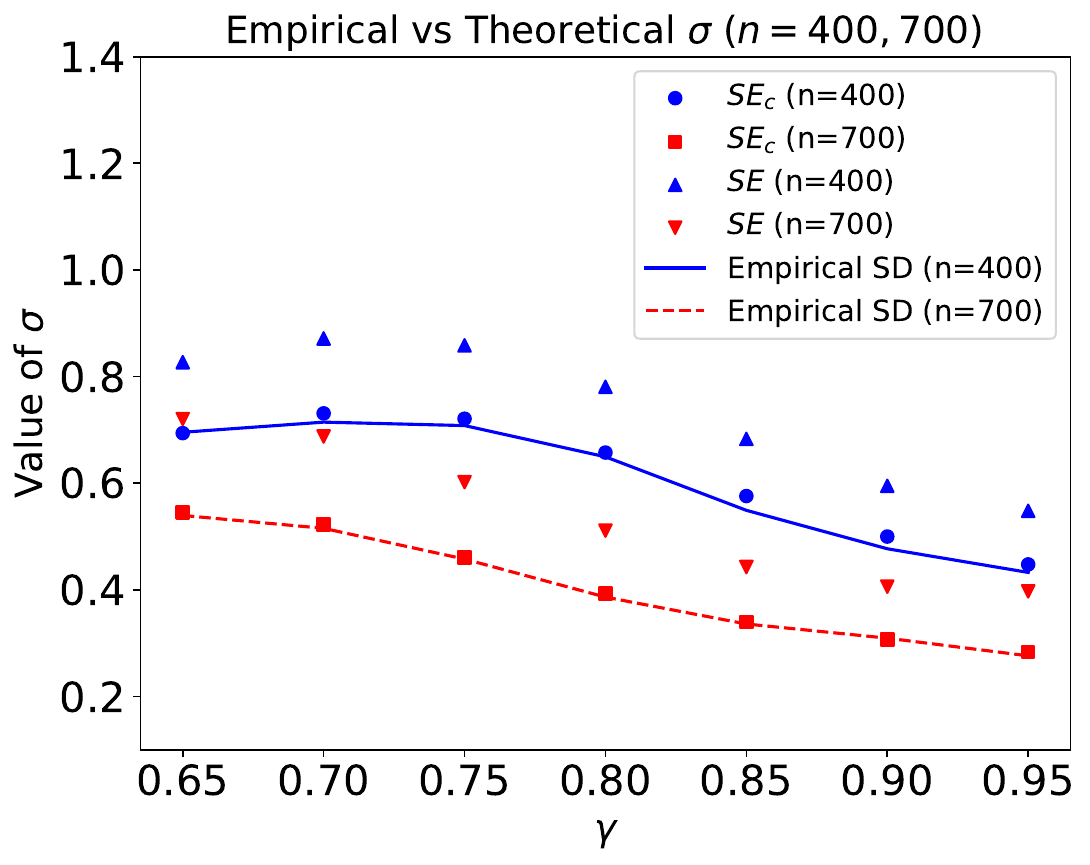}\label{fig_simu:sensitive_Bernoulli_sigma}}
    \subfloat[$\text{MAE}_f$]{\includegraphics[width=0.32\textwidth]{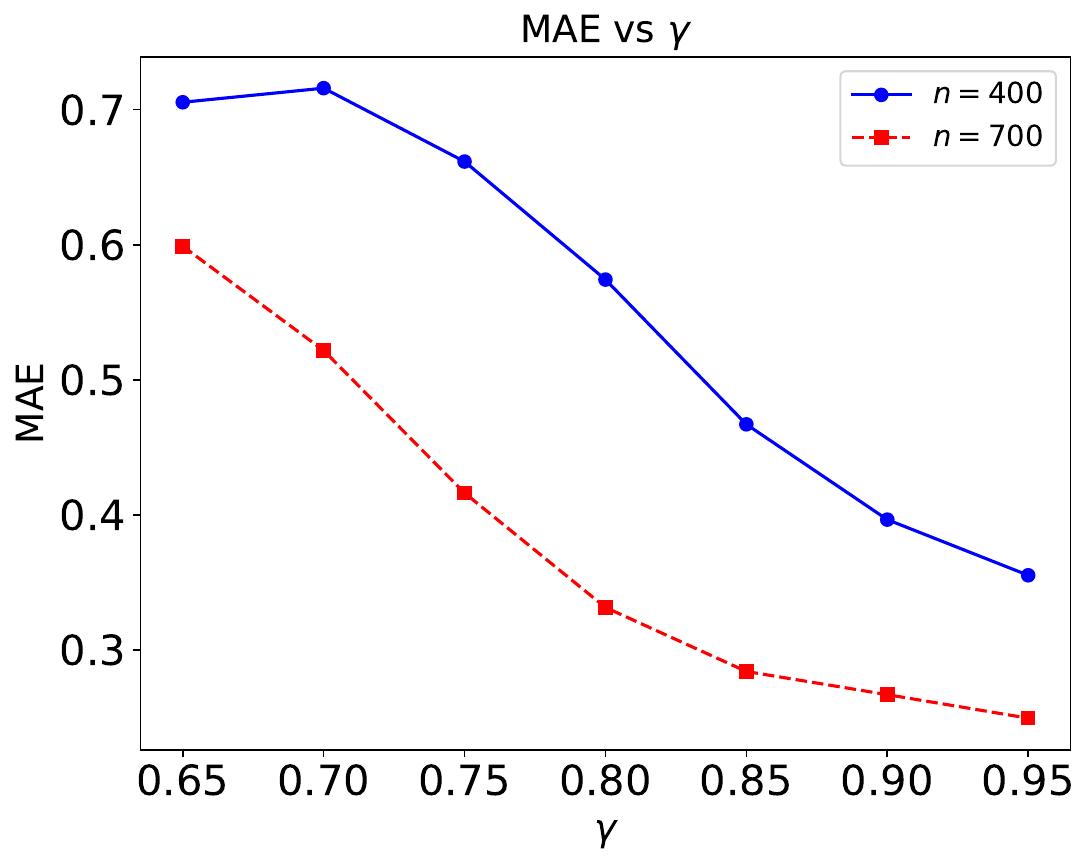}\label{fig_simu:sensitive_Bernoulli_MAE}}
    \subfloat[Cover Probability]{\includegraphics[width=0.32\textwidth]{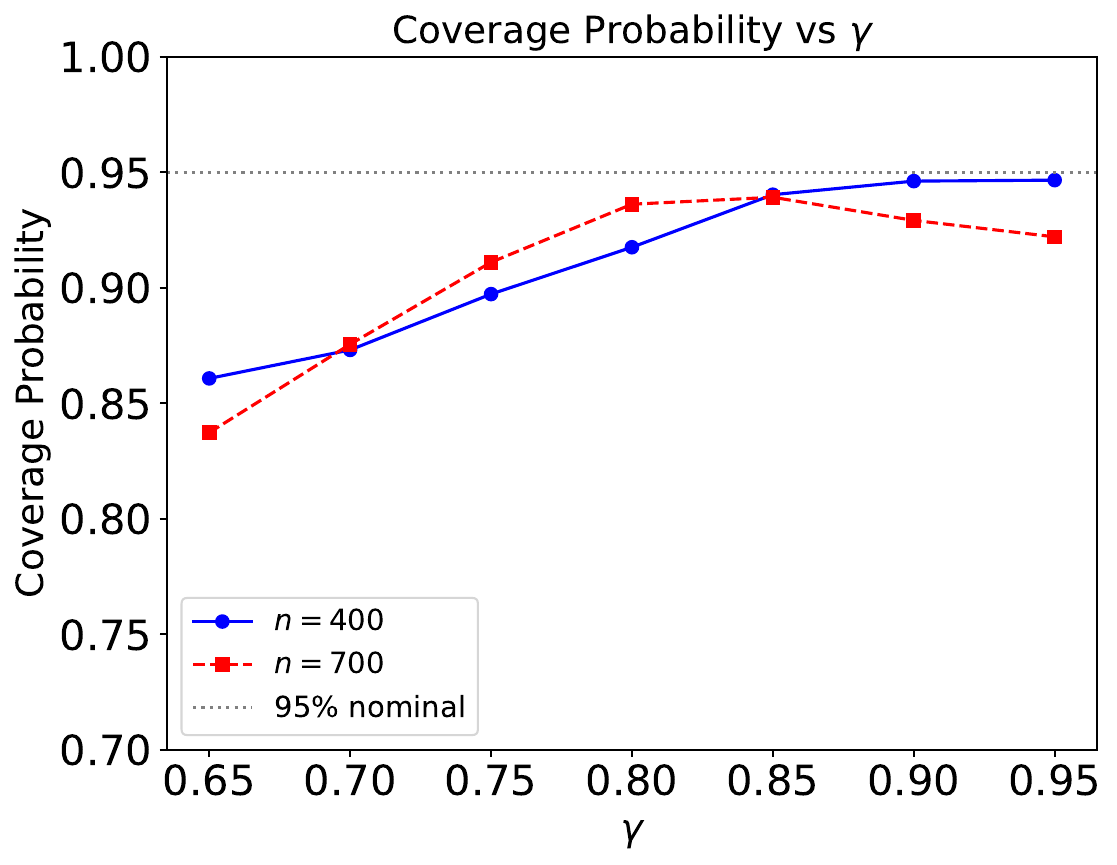}\label{fig_simu:sensitive_Bernoulli_cp}}

    \subfloat[Standard Deviations]{\includegraphics[width=0.32\textwidth]{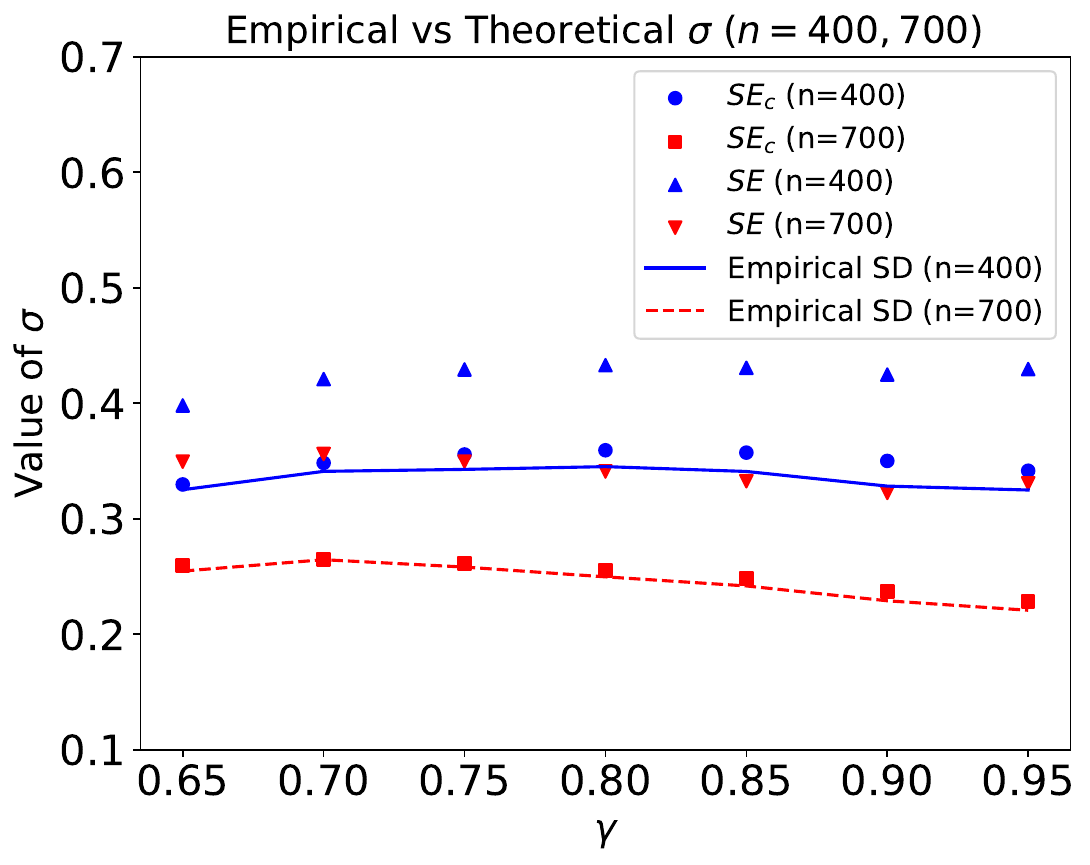}\label{fig_simu:sensitive_Poisson_sigma}}
    \subfloat[$\text{MAE}_f$]{\includegraphics[width=0.32\textwidth]{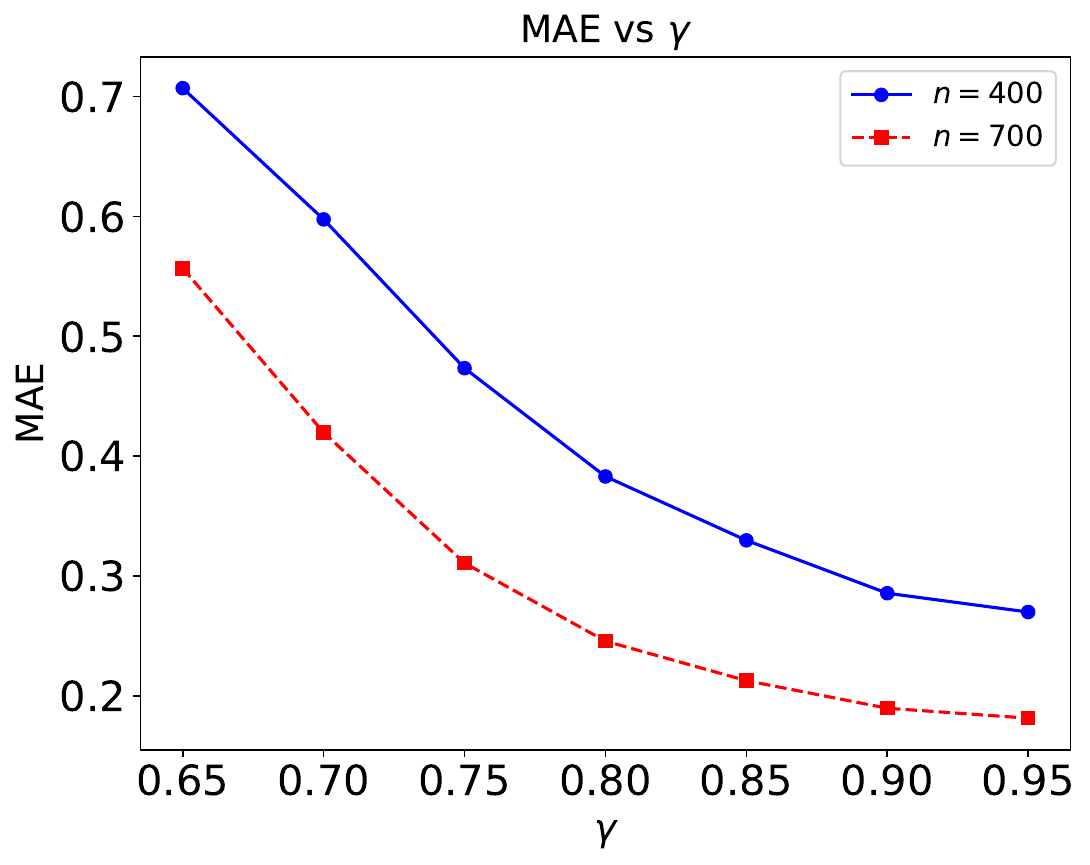}\label{fig_simu:sensitive_Poisson_MAE}}
    \subfloat[Cover Probability]{\includegraphics[width=0.32\textwidth]{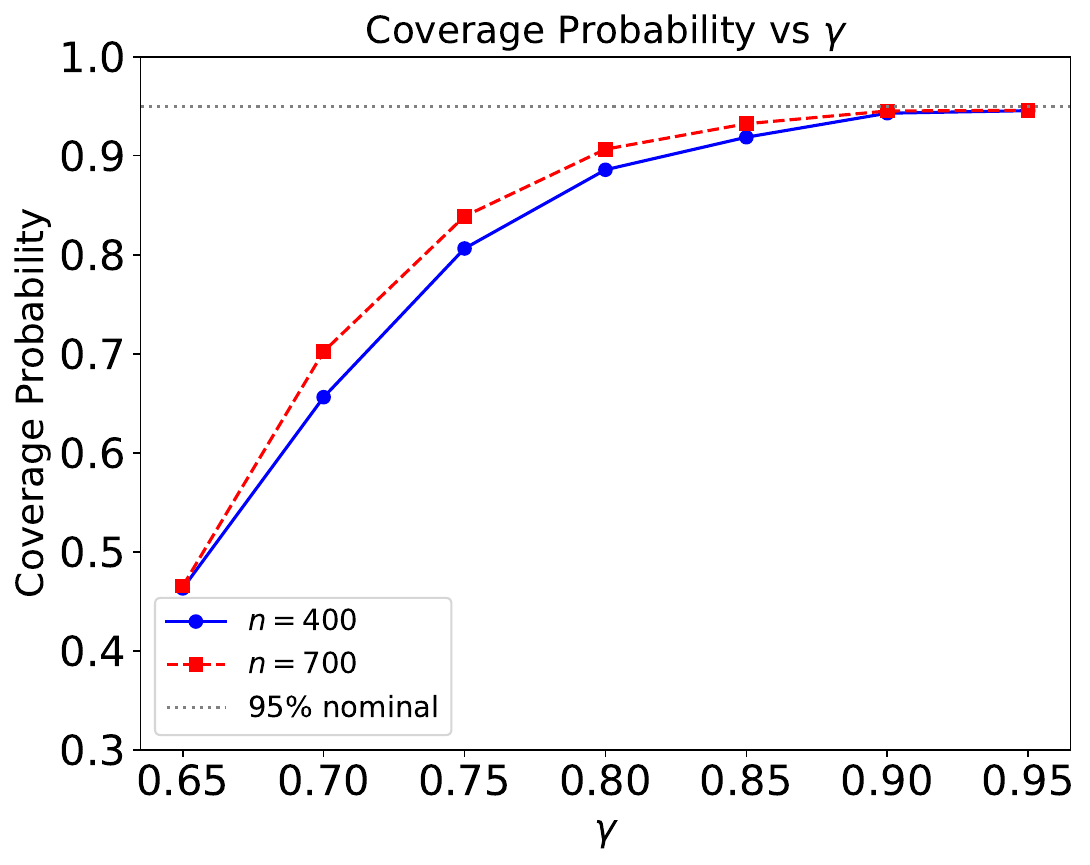}\label{fig_simu:sensitive_Poisson_cp}}
    \caption{Simulation results with different values of $\gamma$. Blue and red lines/points correspond to $n=400$ and $n=700$ respectively. Figures~\ref{fig_simu:sensitive_Bernoulli_sigma}--\ref{fig_simu:sensitive_Bernoulli_cp} present the results under the logistic model, and  Figures~\ref{fig_simu:sensitive_Poisson_sigma}--\ref{fig_simu:sensitive_Poisson_cp} show results under the Poisson model.}
    \label{figsimu:sensitive_r}
\end{figure}

Figures~\ref{fig_simu:sensitive_Bernoulli_sigma}–\ref{fig_simu:sensitive_Poisson_sigma} show that as $\gamma$ increases, our IJ variance estimator closely matches the empirical standard deviation in both logistic and Poisson models, confirming its robustness and validating the finite-sample correction factor $n(n-1)/(n-r)^2$. Figures~\ref{fig_simu:sensitive_Bernoulli_MAE}–\ref{fig_simu:sensitive_Poisson_MAE} indicate that MAE decreases with larger $r$, supporting Theorem~\ref{thm:bound_Rn} and showing bias reduction with larger subsamples. Finally, Figures~\ref{fig_simu:sensitive_Bernoulli_cp}–\ref{fig_simu:sensitive_Poisson_cp} demonstrate that as $r$ increases, coverage improves toward the nominal 95\% level, consistent with our assumption on $\gamma$ ensuring bias decay and valid inference.

We also evaluate computational cost across different $r$ values by recording the training time of a single neural network. All the simulations are conducted in the institute high performance computing platform. The simulation is repeated 300 times, with $B$ networks trained per repetition. As shown in Table~\ref{table:simu_computation_report}, training one network takes only a few seconds, while the main cost arises from repeatedly training the $B$ networks. Following \citet{wager2018estimation} and \citet{fei2024u}, we let $B$ grow proportionally with $n$ and apply bias correction to maintain variance estimation accuracy, largely reducing computation. Although $300\times B$ trainings would otherwise be intensive, high-performance computing with parallel nodes and multi-CPU systems makes the experiments feasible.


\begin{table}[t]
    \centering
\renewcommand{\arraystretch}{0.5}
\caption{Computation time  in seconds.}
\label{table:simu_computation_report}
\begin{tabular}{cccccccc}
\hline
\multicolumn{1}{c}{$\gamma=$} & $0.65$  & $0.7$ & $0.75$  &  $0.8$   & $0.85$ & $0.9$ &$0.95$    \\ \hline
\multicolumn{1}{c}{Logistic} & 0.74 & 0.82 & 1.00 & 1.38 & 1.41 & 1.90 & 2.30  \\
\multicolumn{1}{c}{Poisson} & 0.75 & 0.90 & 0.96 & 1.32 & 1.41 & 1.78 & 2.24
\\\hline
\end{tabular}
\end{table}

\subsection{Additional Experiments on $B$}
\label{subsec:simu_sensitive_B}
 We decrease $B$ from 1400 to 400 in increments of 200 to assess the consistency of variance estimation, while fixing $\gamma = 0.9$.

\begin{figure}[htbp]
    \centering
   \subfloat[Standard Deviations]{\includegraphics[width=0.32\textwidth]{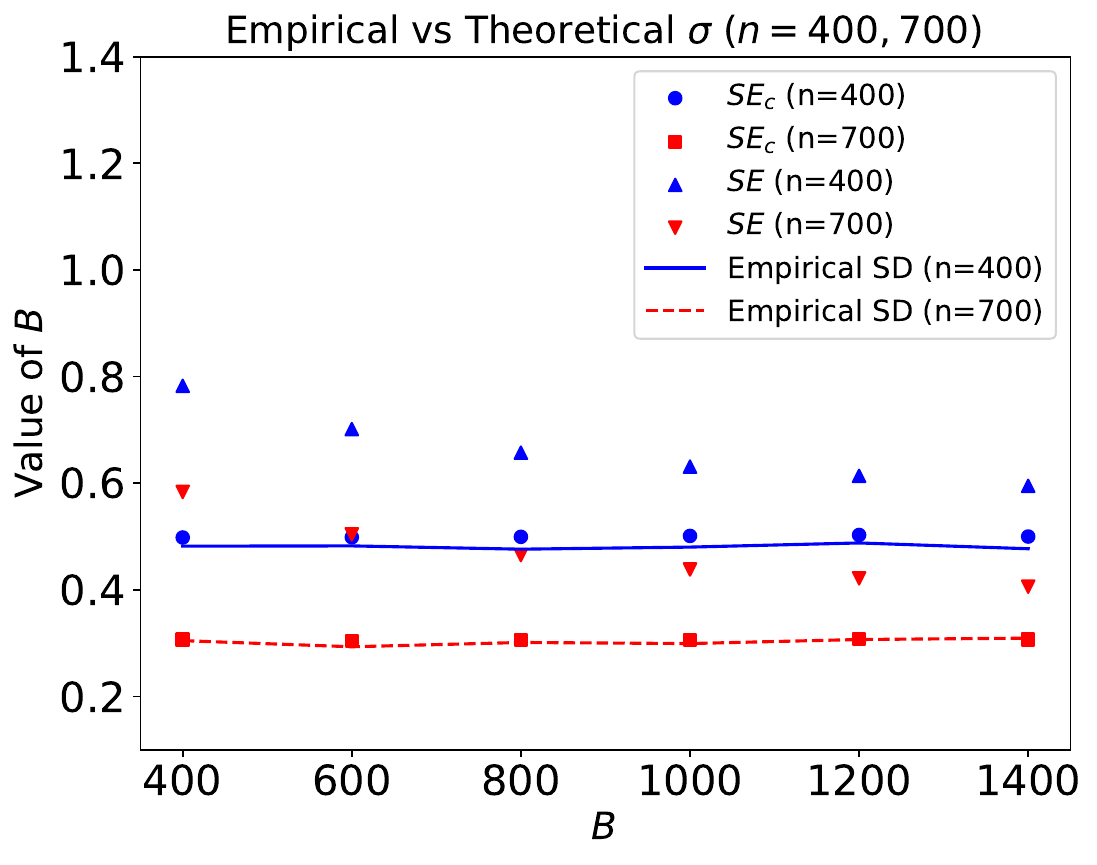}\label{fig_simu:sensitiveB_Bernoulli_sigma}}
    \subfloat[$\text{MAE}_f$]{\includegraphics[width=0.32\textwidth]{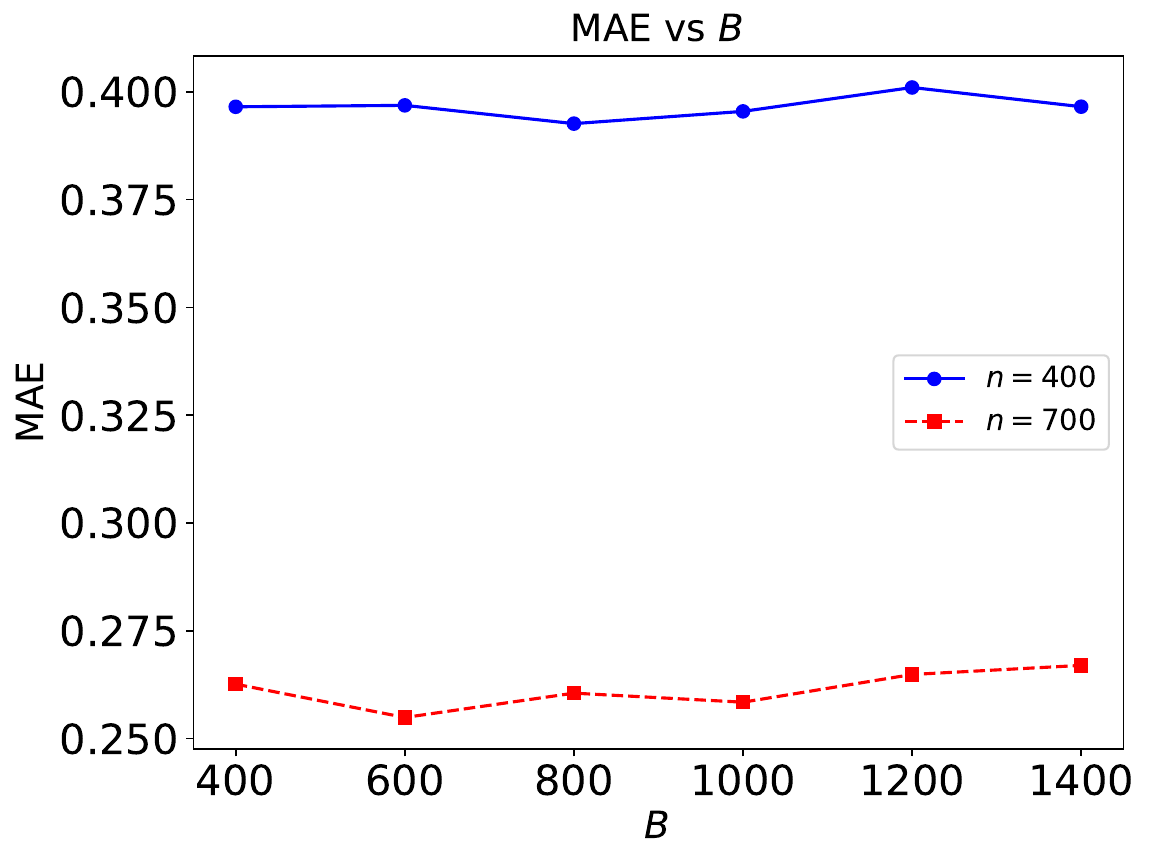}\label{fig_simu:sensitiveB_Bernoulli_MAE}}
    \subfloat[Cover Probability]{\includegraphics[width=0.32\textwidth]{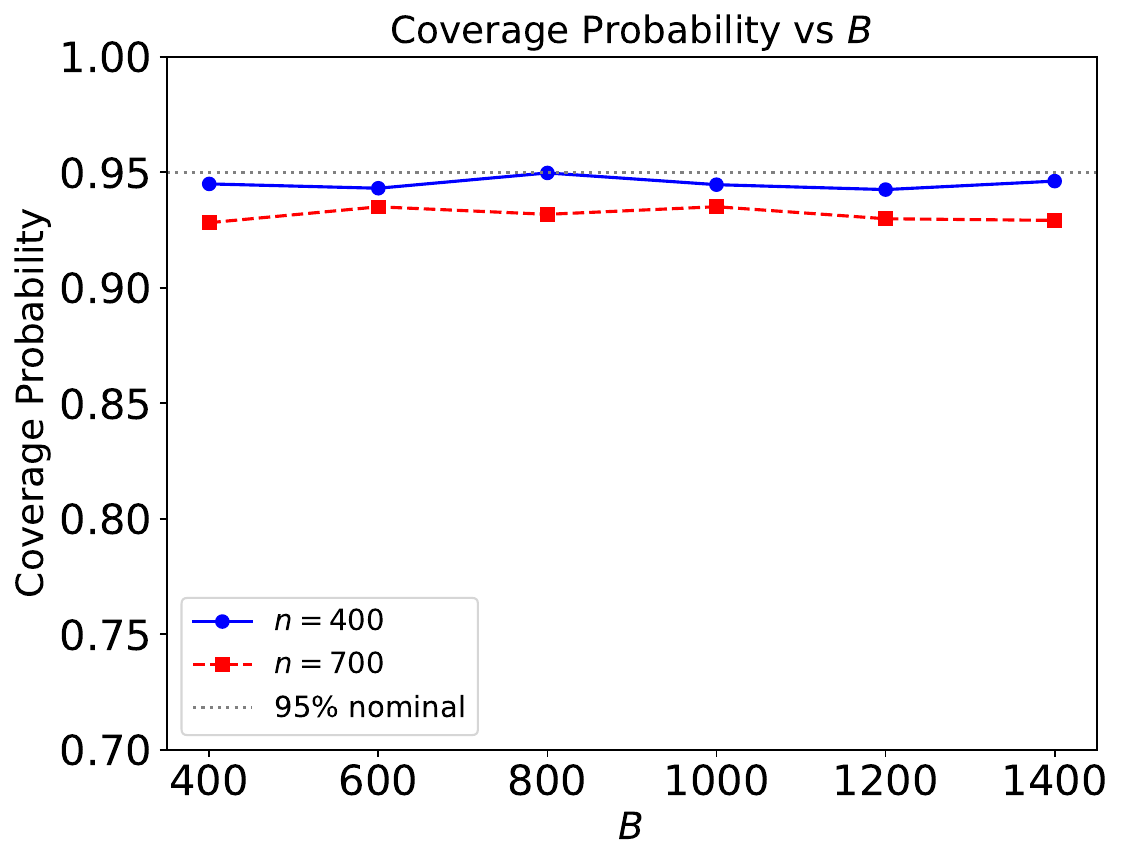}\label{fig_simu:sensitiveB_Bernoulli_cp}}

    \subfloat[Standard Deviations]{\includegraphics[width=0.32\textwidth]{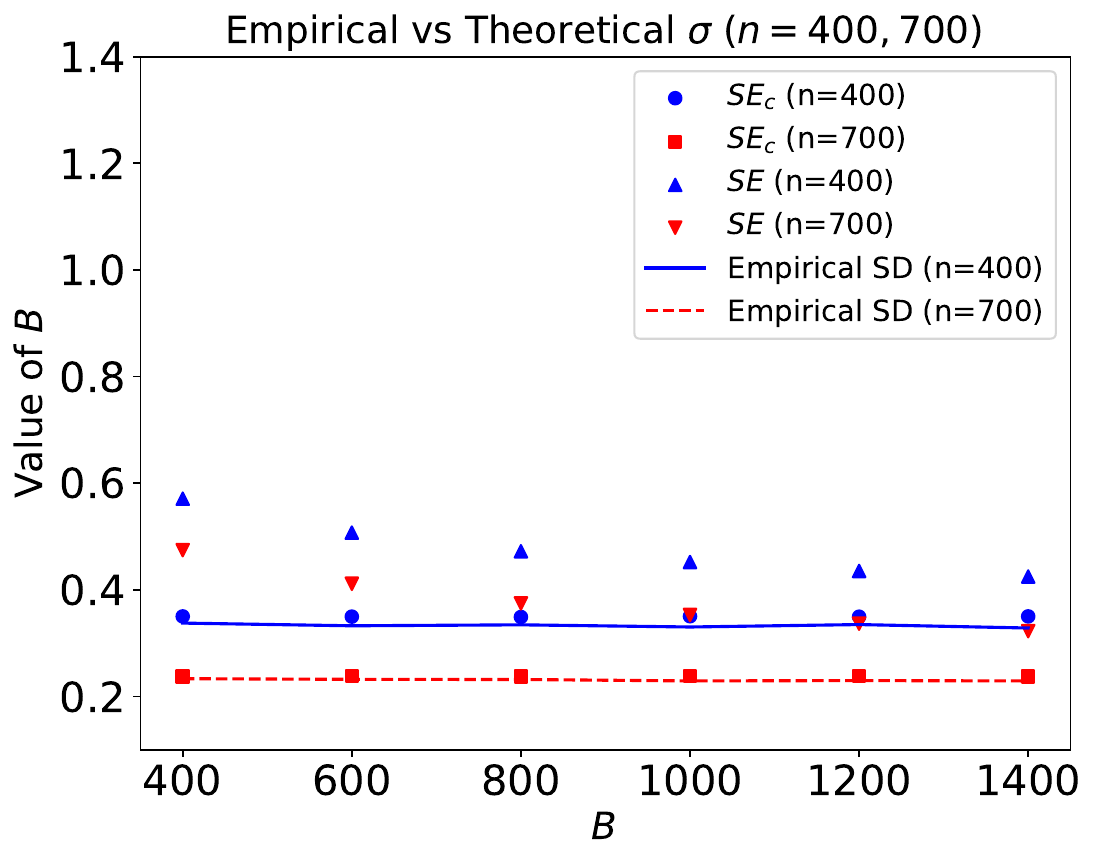}\label{fig_simu:sensitiveB_Poisson_sigma}}
    \subfloat[$\text{MAE}_f$]{\includegraphics[width=0.32\textwidth]{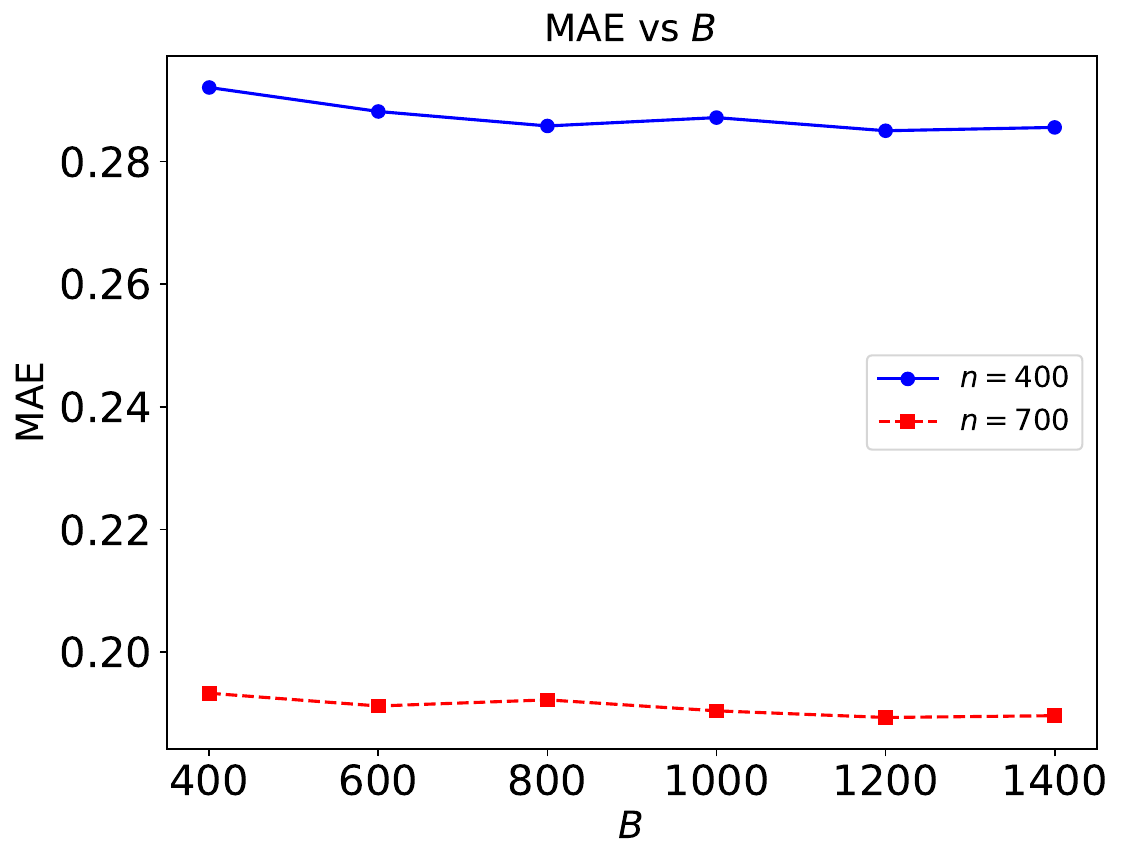}\label{fig_simu:sensitiveB_Poisson_MAE}}
    \subfloat[Cover Probability]{\includegraphics[width=0.32\textwidth]{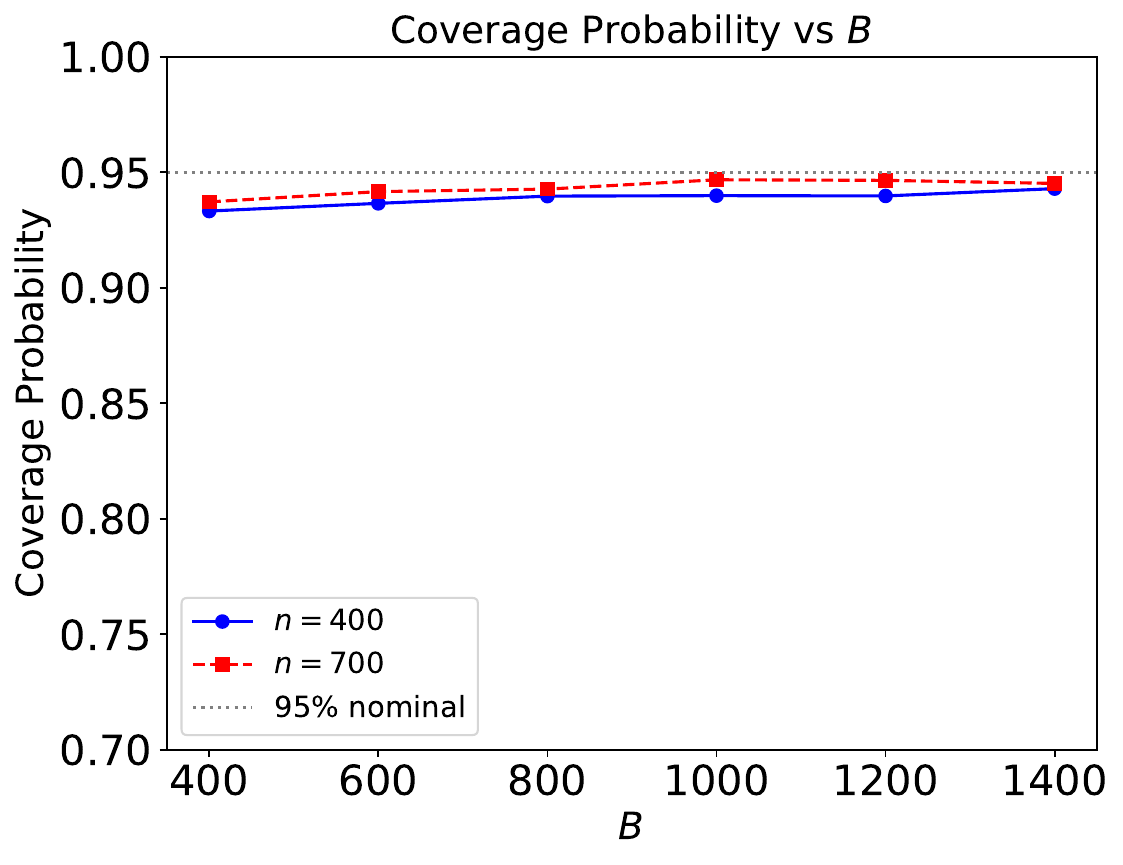}\label{fig_simu:sensitiveB_Poisson_cp}}
    \caption{Simulation results with different values of $B$. Blue and red lines/points correspond to $n=400$ and $n=700$ respectively. Figures~\ref{fig_simu:sensitiveB_Bernoulli_sigma}--\ref{fig_simu:sensitiveB_Bernoulli_cp} present the results under the logistic model, and  Figures~\ref{fig_simu:sensitiveB_Poisson_sigma}--\ref{fig_simu:sensitiveB_Poisson_cp} show results under the Poisson model.}
    \label{figsimu:sensitive_B}
\end{figure}

Figure~\ref{figsimu:sensitive_B} reports the standard deviations, $\text{MAE}_f$, and coverage probabilities as $B$ varies, analogous to Section~\ref{subsec:simu_sensitive_r}. As $B$ increases, the gap between empSD and the uncorrected SE narrows, while our bias-corrected estimator $\text{SE}_c$ closely matches empSD even when $B=400$ and $n=700$ ($B<n$). The $\text{MAE}_f$ is smaller for $n=700$ than for $n=400$, and all coverage probabilities remain near the nominal 95\% level. These results show that the bias correction is effective even with smaller $B$, allowing valid inference and stable performance while reducing computational cost.


\begin{table}[t]
\setlength{\tabcolsep}{2.7pt}
\renewcommand{\arraystretch}{0.5}
\centering
\caption{Simulation results for different models under varying sample sizes $n$ and subsampling ratios $r$ with larger $B=3000$.}
\label{table:simu_changeB}
\begin{tabular}{rccccccccc}
\hline
\multicolumn{1}{l}{} & $\text{Bias}_f$  & $\text{MAE}_f$ & $\text{Bias}_{\psi'}$  &  $\text{MAE}_{\psi'}$   & EmpSD & SE &$\text{SE}_c$  & CP     & AIL  \\ \hline
\multicolumn{1}{l}{} & \multicolumn{7}{c}{Logistic Model}          \\ \hline
$n=400,r=n^{0.8}$  & 0.10(0.71) & 0.56(0.44) & 0.02(0.13) & 0.11(0.08) & 0.64 & 0.72 & 0.66 & 92.0\% & 0.46(0.13)                   \\ 
$r=n^{0.85}$  &0.07(0.60) & 0.47(0.37) &  0.01(0.12) & 0.10(0.07) & 0.56 & 0.63& 0.57& 93.4\% & 0.43(0.11)                \\ 
$r=n^{0.9}$ & 0.05(0.49) & 0.39(0.31) & 0.01(0.10) & 0.08(0.06) & 0.47 & 0.55 & 0.50& {95.3\%} & 0.39(0.09)  \\     
$n=700,r=n^{0.8}$ & 0.06(0.41) & 0.33(0.26) & 0.01(0.09) & 0.07(0.05) & 0.38 & 0.46 & 0.40 & 93.7\% & 0.31(0.08)   \\
$r=n^{0.85}$ & 0.05(0.35) & 0.28(0.22) & 0.01(0.08) & 0.06(0.05) & 0.33 & 0.39 & 0.34 & {94.2\%} & 0.28(0.06)    \\
$r=n^{0.9}$ & 0.05(0.33)& 0.26(0.20) & 0.01(0.07) & 0.06(0.04) & 0.30 & 0.36& 0.31 & {93.8\%} & 0.26(0.06)     \\  
\hline
\multicolumn{1}{l}{} & \multicolumn{7}{c}{Poisson Model}          \\ \hline
$n=400,r=n^{0.8}$ & -0.32(0.38) &  0.39(0.32) & -0.17(0.24) & 0.22(0.18) & 0.36 & 0.39 & 0.36 & 88.1\% & 0.87(0.46)               \\ 
$r=n^{0.85}$ & -0.21(0.37) & 0.32(0.27) & -0.10(0.25) & 0.20(0.18) & 0.34 & 0.39 & 0.36 & 92.2\% & 0.96(0.53)           \\ 
$r=n^{0.9}$  & -0.13(0.35) & 0.28(0.24) & -0.06(0.26) & 0.20(0.18) & 0.33 & 0.38 & 0.35 & {93.5\%} & 1.02(0.58)       \\ 
$n=700,r=n^{0.8}$ & -0.16(0.27) & 0.24(0.20) & -0.09(0.19) & 0.16(0.14) & 0.25 & 0.30 & 0.25 & 90.4\% & 0.68(0.34)       \\ 
$r=n^{0.85}$ & -0.10(0.26) & 0.21(0.17) & -0.06(0.19) & 0.15(0.14) & 0.24 & 0.30 & 0.25 & 93.1\% & 0.71(0.37)       \\ 
$r=n^{0.9}$ & -0.04(0.24) & 0.19(0.15) & -0.02(0.19) &0.14(0.13) & 0.23 & 0.28 & 0.24 & {95.2\%} & 0.72(0.38)   \\ 
\hline
\end{tabular}
\end{table}

\begin{figure}[t]
    \centering
    \subfloat[Point Estimates and Variations]{\includegraphics[width=0.48\textwidth]{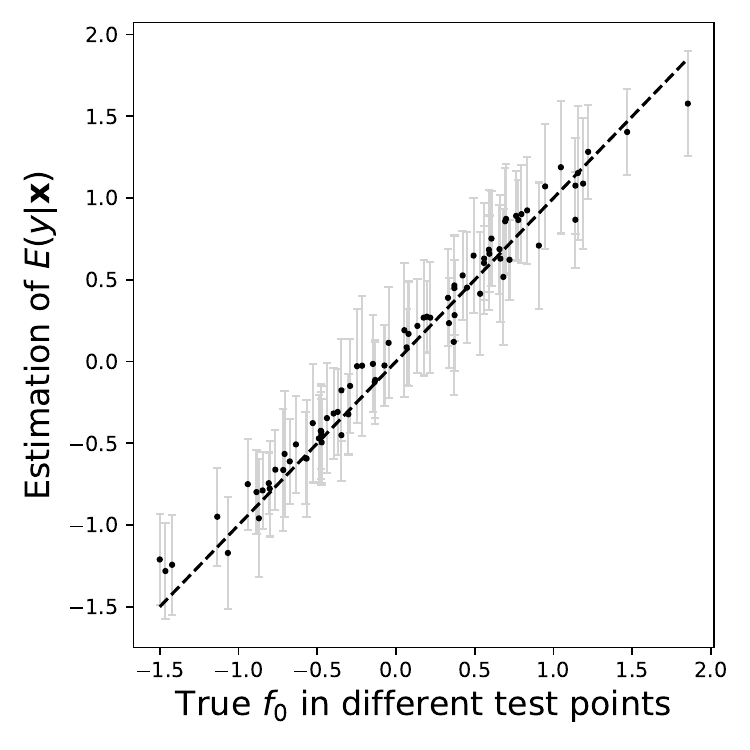}\label{fig_simu:part2CI}}
    \subfloat[Estimated Standard Errors and Variations]{\includegraphics[width=0.48\textwidth]{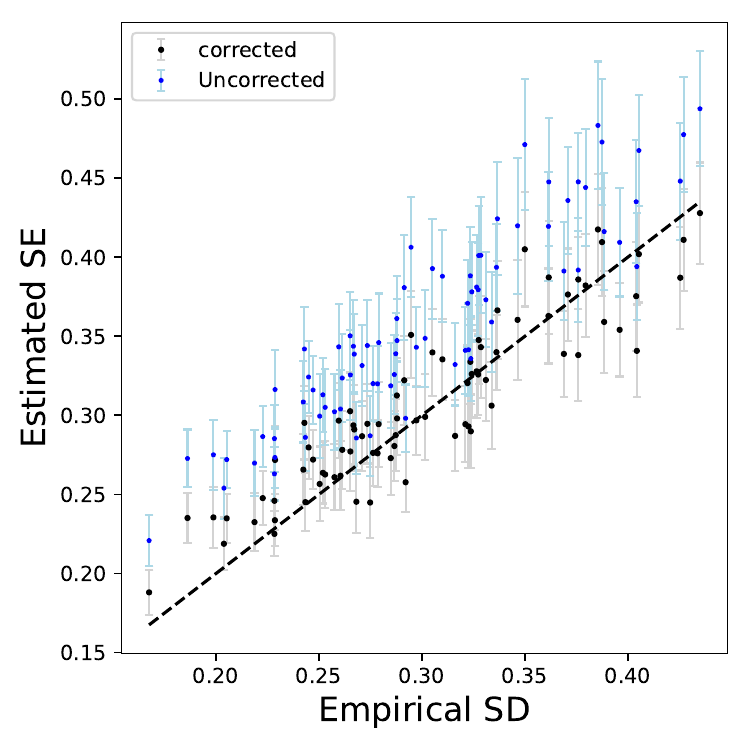}\label{fig_simu:part2SD}}
    \caption{Estimation and inference in simulation samples:  Logistic Model with $n=700$, $r=n^{0.9}$ and $B=3000$.      Figure~\ref{fig_simu:part2CI} shows the average estimated $\EE(y|\xb)$ with variability across 300 repetitions (gray band) over test points. Figure~\ref{fig_simu:part2SD} compares corrected and uncorrected standard errors and their variability (gray and blue bands, respectively) to the empirical standard deviations of estimates across all test samples.}
        \label{fig_simu:part2}
\end{figure}

We further assess the impact of a larger subsample size ($B=3000$) on variance estimation and the accuracy of the correction term. Table~\ref{table:simu_changeB} shows that the results remain consistent: increasing $B$ reduces MAE for both $\widehat{f}^B$ and $\psi'(\widehat{f}^B)$, and the empirical standard deviations (empSD) align closely with the estimated standard errors (SE), maintaining nominal 95\% coverage. Figures~\ref{fig_simu:part2CI} and \ref{fig_simu:part2SD} illustrate that our estimator accurately recovers the conditional mean and that the corrected variance closely follows empSD, whereas the uncorrected version slightly overestimates. As $B$ increases from 1400 to 3000, this bias diminishes, consistent with theory showing that the correction term vanishes asymptotically. These findings confirm that the correction improves finite-sample accuracy while preserving asymptotic validity.
   

\subsection{Additional Experiments on $L$}
\label{subsec:simu_sensitive_L}

We now examine how network depth affects estimation bias and variance using $n=700$, varying both $L$ and $r$. For each $L$, we set $\pb = (p, \underbrace{30,\ldots,30}_{L-1}, 1)$ and apply batch normalization after the 2nd, 4th, and 6th layers to ensure gradient stability. We consider $L=3,4,5,6$ and $r=n^{0.8}, n^{0.85}, n^{0.9}, n^{0.95}$.

Tables~\ref{table:simu_sensitive_Logistic_DNN_L_r}–\ref{table:simu_sensitive_Poisson_DNN_L_r} show that deeper networks tend to introduce larger biases, consistent with~\eqref{eq:bias_disappear}, which predicts larger approximation error and lower coverage for deeper architectures. Hence, in our settings, it appears that relatively shallow networks are preferable for both stability and reduced bias \citep{goodfellow2016deep}.

Consistent with the sensitivity analysis in Section~\ref{subsec:simu_sensitive_r}, increasing $r$ reduces $\text{MAE}_f$, and all bias-corrected variance estimators closely match the empirical results, confirming the robustness of our inference procedure.




\begin{table}[t]
\setlength{\tabcolsep}{2.4pt}
\renewcommand{\arraystretch}{0.5}
\centering
\caption{Simulation results with Logistic under varying $L$ and $r$.}
\label{table:simu_sensitive_Logistic_DNN_L_r}
\begin{tabular}{rccccccccc}
\hline
\multicolumn{1}{l}{} & $\text{Bias}_f$  & $\text{MAE}_f$ & $\text{Bias}_{\psi'}$  &  $\text{MAE}_{\psi'}$   & EmpSD & SE &$\text{SE}_c$  & CP     & AIL  \\ \hline
\multicolumn{1}{l}{} & \multicolumn{7}{c}{Logistic}          \\  \hline 
$r = n^{0.8}, L=3$ & 0.19(0.66) & 0.57(0.39) & 0.04(0.13) & 0.11(0.08) & 0.59 & 0.92 & 0.60 & 84.5\% & 0.43(0.15)\\
$L=4$ & 0.21(0.71) & 0.61(0.41) & 0.04(0.14) & 0.12(0.08) & 0.61 & 0.98 & 0.61 & 82.7\% & 0.44(0.16) \\
$L=5$ & 0.21(0.73) & 0.63(0.43) & 0.04(0.14) & 0.12(0.08) & 0.63 & 1.11 & 0.63 & 82.0\% & 0.44(0.17)\\
$L=6$ & 0.21(0.70) & 0.61(0.41) & 0.04(0.14) & 0.12(0.08) & 0.61 & 1.10 & 0.63 & 83.4\% & 0.44(0.16)\\  \hline
$r = n^{0.85}, L=3$ & 0.19(0.69) & 0.59(0.41) & 0.04(0.14) & 0.12(0.08) & 0.63 & 1.03 & 0.64 & 85.6\% & 0.46(0.16) \\
$L=4$ & 0.21(0.74) & 0.64(0.43) & 0.04(0.14) & 0.13(0.08) & 0.66 & 1.12 & 0.67 & 83.2\% & 0.46(0.17) \\
$L=5$ & 0.20(0.80) & 0.69(0.46) & 0.04(0.15) & 0.13(0.08) & 0.71 & 1.32 & 0.71 & 81.8\% & 0.48(0.19) \\
$L=6$ & 0.21(0.77) & 0.66(0.45) & 0.04(0.15) & 0.13(0.09) & 0.69 & 1.31 & 0.70 & 83.0\% & 0.48(0.18)\\  \hline
$r = n^{0.9}, L=3$ & 0.19(0.71) & 0.60(0.42) & 0.04(0.14) & 0.12(0.08) & 0.65 & 1.16 & 0.68 & 86.8\% & 0.48(0.16) \\
$L=4$ & 0.21(0.75) & 0.64(0.44) & 0.04(0.15) & 0.13(0.08) & 0.69 & 1.27 & 0.71 & 84.9\% & 0.49(0.18) \\
$L=5$ & 0.21(0.85) & 0.73(0.48) & 0.04(0.16) & 0.14(0.09) & 0.76 & 1.60 & 0.78 & 81.9\% & 0.51(0.21) \\
$L=6$ & 0.20(0.82) & 0.72(0.47) & 0.04(0.16) & 0.14(0.09) & 0.75 & 1.57 & 0.76 & 82.7\% & 0.51(0.20) \\  \hline
$r = n^{0.95}, L=3$ & 0.17(0.71) & 0.59(0.42) & 0.03(0.14) & 0.12(0.08) & 0.66 & 1.37 & 0.69 & 88.9\% & 0.49(0.16) \\
$L=4$ & 0.20(0.74) & 0.63(0.44) & 0.04(0.15) & 0.12(0.09) & 0.69 & 1.48 & 0.72 & 86.5\% & 0.49(0.18) \\
$L=5$ & 0.19(0.90) & 0.77(0.50) & 0.03(0.17) & 0.15(0.09) & 0.82 & 2.14 &0.84 & 81.7\% & 0.54(0.23)  \\
$L=6$ & 0.21(0.86) & 0.73(0.49) & 0.04(0.17) & 0.14(0.09) & 0.79 & 2.07 & 0.82 & 81.9\% & 0.53(0.23) \\  \hline 
\end{tabular}
\end{table}

\begin{table}[H]
\setlength{\tabcolsep}{2.4pt}
\renewcommand{\arraystretch}{0.5}
\centering
\caption{Simulation results with Poisson under varying $L$ and $r$.}
\label{table:simu_sensitive_Poisson_DNN_L_r}
\begin{tabular}{rccccccccc}
\hline
\multicolumn{1}{l}{} & $\text{Bias}_f$  & $\text{MAE}_f$ & $\text{Bias}_{\psi'}$  &  $\text{MAE}_{\psi'}$   & EmpSD & SE &$\text{SE}_c$  & CP     & AIL  \\ \hline
\multicolumn{1}{l}{} & \multicolumn{7}{c}{Poisson}          \\  \hline 
$r = n^{0.8}, L=3$ & -0.38(0.33) & 0.41(0.29) & -0.20(0.21) & 0.23(0.17) & 0.31 & 0.49 & 0.31 & 79.8\% & 0.71(0.37) \\
$L=4$ & -0.42(0.37) & 0.45(0.32) & -0.20(0.21) & 0.24(0.17) & 0.33 & 0.53 & 0.33 & 76.9\% & 0.74(0.40) \\
$L=5$ & -0.47(0.38) & 0.50(0.34) & -0.23(0.22) & 0.26(0.17) & 0.34 & 0.59 & 0.34 & 72.9\% & 0.75(0.44) \\
$L=6$ & -0.48(0.39) & 0.51(0.35) & -0.22(0.22) & 0.26(0.17) & 0.34 & 0.62 & 0.35 & 73.7\% & 0.78(0.46) \\   \hline 
$r = n^{0.85}, L=3$ & -0.30(0.35) & 0.36(0.28) & -0.15(0.22) & 0.21(0.17) & 0.33 & 0.55 & 0.33 & 87.0\% & 0.82(0.43)  \\
$L=4$ & -0.35(0.38) & 0.41(0.31) & -0.16(0.23) & 0.22(0.17) & 0.35 & 0.61 & 0.36 & 84.6\% & 0.86(0.46)\\
$L=5$ & -0.42(0.41) & 0.48(0.34) & -0.19(0.24) & 0.25(0.18) & 0.37 & 0.69 & 0.37 & 79.6\% & 0.88(0.52) \\
$L=6$ & -0.43(0.42) & 0.48(0.35) & -0.19(0.24) & 0.25(0.18) & 0.38 & 0.72 & 0.38 & 80.0\% & 0.90(0.53)\\  \hline 
$r = n^{0.9}, L=3$ & -0.22(0.36) & 0.33(0.27) & -0.10(0.24) & 0.20(0.17) & 0.34 & 0.63 & 0.35 & 91.2\% & 0.94(0.50)  \\
$L=4$ & -0.26(0.39) & 0.37(0.30) & -0.11(0.25) & 0.21(0.17) & 0.37 & 0.69 & 0.38  & 89.1\% & 0.98(0.52)\\
$L=5$ & -0.35(0.43) & 0.44(0.34) & -0.15(0.26) & 0.24(0.18) & 0.40 & 0.82 & 0.40 & 84.8\% & 1.02(0.60) \\
$L=6$ & -0.36(0.45) & 0.45(0.35) & -0.14(0.27) & 0.25(0.18) & 0.41 & 0.86  &  0.41 & 84.5\% & 1.05(0.62) \\ \hline 
$r = n^{0.95}, L=3$ & -0.14(0.37) & 0.30(0.25) & -0.04(0.27) & 0.20(0.18) & 0.35 & 0.77 & 0.37 & 93.1\% & 1.06(0.62) \\
$L=4$ & -0.17(0.40) & 0.33(0.28) & -0.05(0.27) & 0.21(0.18) & 0.37 & 0.84 & 0.39 & 91.6\% & 1.11(0.62) \\
$L=5$ & -0.26(0.45) & 0.41(0.32) & -0.09(0.29) & 0.24(0.19) & 0.42 & 1.06 & 0.43 & 87.8\% & 1.20(0.74) \\
$L=6$ & -0.27(0.46) & 0.42(0.33) & -0.09(0.29) & 0.24(0.19) & 0.42 & 1.11 & 0.44 & 87.2\% & 1.22(0.76)\\ \hline 
\end{tabular}
\end{table}

\subsection{Additional Experiments on Dropout Rate}
\label{subsec:simu_sensitive_dropout}
 We vary the dropout rates, which are set to 0.3 and 0.5 in the two models, while retaining all other settings as in Section~\ref{sec:simulations}. 
\begin{table}[t]
\setlength{\tabcolsep}{2.4pt}
\renewcommand{\arraystretch}{0.5}
\centering
\caption{Simulation results   with different dropout rate.}
\label{table:simu_dropout}
\begin{tabular}{rccccccccc}
\hline
\multicolumn{1}{l}{} & $\text{Bias}_f$  & $\text{MAE}_f$ & $\text{Bias}_{\psi'}$  &  $\text{MAE}_{\psi'}$   & EmpSD & SE &$\text{SE}_c$  & CP     & AIL  \\ \hline
\multicolumn{1}{l}{} & \multicolumn{7}{c}{Logistic with Dropout Rate $=0.3$}          \\ \hline
$n=400,r=n^{0.8}$ & 0.09(0.59) & 0.47(0.37) & 0.02(0.12) & 0.10(0.07) & 0.55 & 0.68 & 0.56 & 92.9\% & 0.41(0.11)  \\ 
$r=n^{0.85}$ & 0.08(0.51) & 0.41(0.32) & 0.02(0.11) & 0.09(0.07) & 0.48 & 0.60 & 0.49 & 93.9\% & 0.38(0.09)      \\ 
$r=n^{0.9}$ & 0.06(0.44) & 0.35(0.28) & 0.01(0.10) & 0.08(0.06) & 0.42 & 0.54 & 0.43 & 94.1\% & 0.35(0.08)      \\ 
$n=700,r=n^{0.8}$ & 0.08(0.35) & 0.28(0.23) & 0.02(0.08) & 0.06(0.05) & 0.33 & 0.46 & 0.33 & 93.2\% & 0.27(0.06)     \\
$r=n^{0.85}$ & 0.06(0.32) & 0.25(0.20) & 0.01(0.04) & 0.06(0.04) & 0.29 & 0.41 & 0.30 & 93.1\% & 0.25(0.05)    \\
$r=n^{0.9}$ & 0.07(0.31) & 0.25(0.19) & 0.02(0.07) & 0.06(0.04) & 0.27 & 0.39 & 0.27 & 90.9\% & 0.23(0.05)   \\  
\hline
\multicolumn{1}{l}{} & \multicolumn{7}{c}{Logistic with Dropout Rate $=0.5$}          \\ \hline
$n=400,r=n^{0.8}$  & 0.11(0.48) & 0.38(0.31) & 0.02(0.10) & 0.08(0.06) & 0.45 & 0.58 & 0.46 &  93.8\% & 0.36(0.09)           \\ 
$r=n^{0.85}$ & 0.08(0.42) & 0.33(0.27) & 0.02(0.09) & 0.07(0.06) & 0.40 & 0.52 & 0.40 & 93.9\% & 0.33(0.07)        \\ 
$r=n^{0.9}$ & 0.07(0.39) & 0.31(0.25) & 0.02(0.09) & 0.07(0.05) & 0.36 & 0.48 & 0.36 & 93.0\% & 0.30(0.07)     \\ 
$n=700,r=n^{0.8}$  & 0.07(0.31) & 0.25(0.20) & 0.02(0.07) & 0.06(0.04) & 0.28 & 0.40 & 0.28 & 92.1\% & 0.24(0.05)     \\ 
$r=n^{0.85}$ & 0.07(0.30) & 0.24(0.19) & 0.02(0.07) & 0.05(0.04) & 0.25  & 0.38 & 0.26 & 89.6\% & 0.22(0.05)    \\ 
$r=n^{0.9}$ & 0.07(0.31) & 0.25(0.19) & 0.02(0.07) & 0.06(0.04) & 0.24 & 0.38 & 0.24 & 86.1\% & 0.21(0.04)     \\ \hline
\multicolumn{1}{l}{} & \multicolumn{7}{c}{ Poisson with Dropout Rate $=0.3$}          \\ \hline
$n=400,r=n^{0.8}$ & -0.24(0.33) & 0.31(0.26) & -0.13(0.22) & 0.19(0.16) & 0.31  & 0.39 & 0.31 & 90.8\% & 0.77(0.36)             \\ 
$r=n^{0.85}$ & -0.13(0.31) & 0.25(0.22) & -0.07(0.22) & 0.17(0.15) & 0.28 & 0.38 & 0.30 & 94.0\% & 0.82(0.40)       \\ 
$r=n^{0.9}$  & -0.07(0.29) & 0.23(0.20) & -0.04(0.22) & 0.16(0.15) & 0.28 & 0.37 & 0.29 & 94.6\% & 0.85(0.43)     \\ 
$n=700,r=n^{0.8}$ & -0.10(0.23) & 0.19(0.16) & -0.06(0.17) & 0.13(0.12) & 0.21 & 0.30 & 0.22 & 92.8\% & 0.59(0.27)       \\ 
$r=n^{0.85}$ & -0.05(0.22) & 0.17(0.14) & -0.03(0.17) & 0.13(0.12) & 0.20 & 0.30 & 0.21 & 93.6\% & 0.60(0.28)     \\ 
$r=n^{0.9}$ & -0.01(0.21) & 0.16(0.13) & -0.01(0.17) & 0.12(0.11) & 0.19 & 0.30 & 0.20 & 93.6\% & 0.60(0.29)       \\ \hline
\multicolumn{1}{l}{} & \multicolumn{7}{c}{ Poisson with Dropout Rate $=0.5$}          \\ \hline
$n=400,r=n^{0.8}$  & -0.15(0.29) & 0.25(0.22) & -0.09(0.20) & 0.16(0.14) & 0.26 & 0.35 & 0.27 & 92.9\% & 0.70(0.29)            \\ 
$r=n^{0.85}$ & -0.07(0.27) & 0.21(0.18) & -0.04(0.20) & 0.15(0.14) & 0.25 & 0.33 & 0.25 & 93.8\% & 0.72(0.31)        \\ 
$r=n^{0.9}$ & -0.03(0.26) & 0.21(0.17) & -0.02(0.20) & 0.15(0.13) & 0.24 & 0.34 & 0.25 & 93.6\% & 0.72(0.33)       \\ 
$n=700,r=n^{0.8}$ & -0.05(0.21) & 0.17(0.14) & -0.03(0.16) & 0.12(0.11) & 0.18 & 0.27 & 0.18 & 91.8\% & 0.51(0.21)    \\ 
$r=n^{0.85}$   & -0.01(0.20) & 0.16(0.13) & -0.02(0.16) & 0.11(0.11) & 0.18 & 0.27 & 0.18 & 91.6\% & 0.51(0.22)      \\ 
$r=n^{0.9}$ & 0.02(0.20) & 0.16(0.12) & 0.01(0.15) & 0.11(0.10) & 0.17 & 0.28 &  0.17 & 90.3\% & 0.51(0.22)           \\ \hline
\end{tabular}
\end{table}
As shown in Table~\ref{table:simu_dropout}, under different dropout rates, both the bias and $\text{MAE}_f$ decrease as $r$ increases, showing a consistent trend across settings. The proposed variance estimator also remains accurate under all configurations. The coverage probability can be affected when the dropout rate is large, as excessive dropout may limit the network’s effective capacity and lead to mild underfitting.

\newpage

\section{Additional Real Data Experiments}
\label{sec:additional_realdata}
\subsection{Additional Investigation on Large Volatility in DNNs}
\label{subsec:additionalcheck}
 We further analyze patients with unusually wide prediction intervals from neural networks. As shown in Figure~\ref{fig_real:calibrationLogisticNN}, patients indexed around 800–1600 exhibit large intervals despite similar predicted risks. Among the five with the widest intervals, outlier values in glucose and WBC likely drive the elevated uncertainty. Figure~\ref{figreal:outliers} highlights these extreme values.

\begin{figure}[htbp]
    \centering
   \subfloat[Histogram for ``glucose"]{\includegraphics[width=0.48\textwidth]{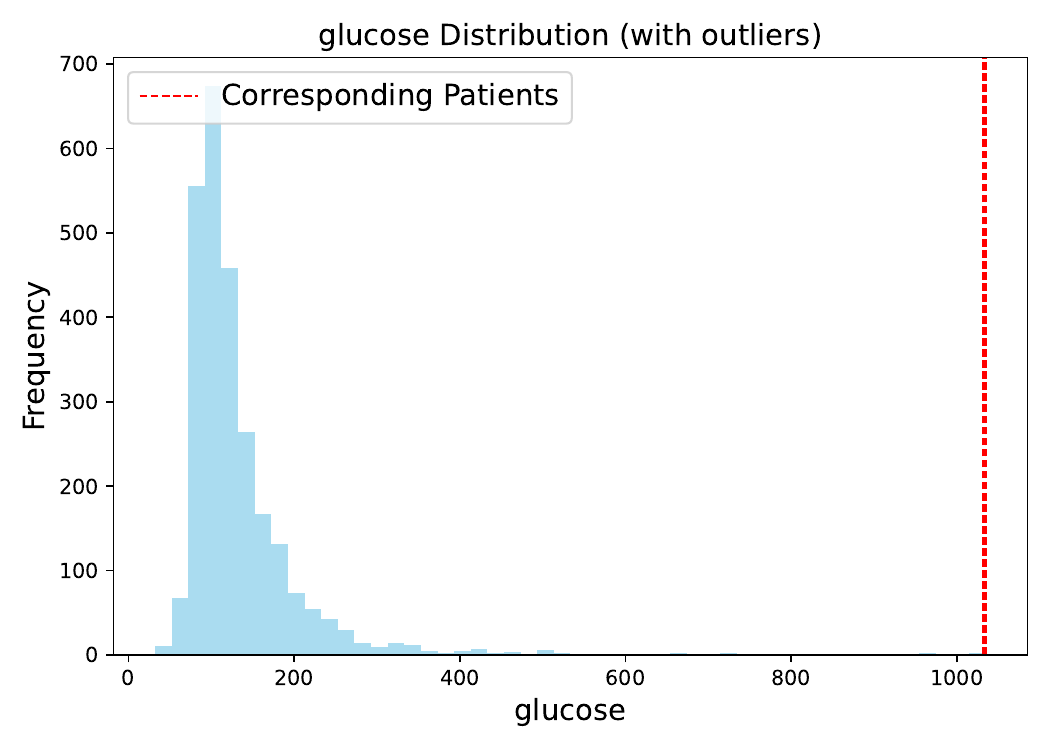}\label{fig_real:outliers1}}
    \subfloat[Histogram for ``WBC"]{\includegraphics[width=0.48\textwidth]{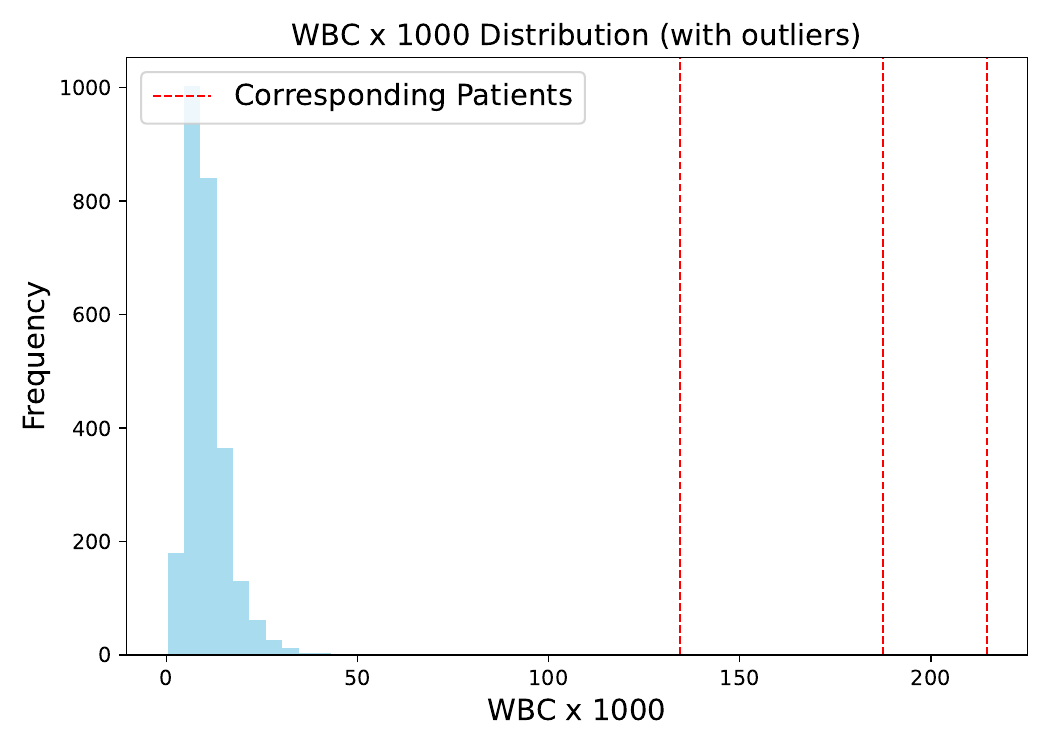}\label{fig_real:outliers2}}
    \caption{Distributions of glucose and WBC levels for the five distinct patients with the widest prediction intervals. Red dashed lines mark their observed values: two patients have extreme glucose values and three others have extreme WBC values, corresponding to those exhibiting the highest uncertainty under neural network estimations.}
    \label{figreal:outliers}
\end{figure} 

 Figure~\ref{figreal:outliers} shows that two patients have extremely high glucose levels and three have unusually high WBC, all well outside the typical range. For neural networks, such outliers reduce prediction confidence, leading to wider intervals, an adaptive and desirable feature for guarding against overconfidence in poorly supported regions. In contrast, random forests, while more robust to outliers due to localized splitting, may produce overly narrow intervals, underestimating uncertainty in critical cases.

\subsection{Additional Experiments with Additional Features}
\label{subsec:realdata_addfeature}
 
\begin{figure}[htbp]
    \centering
   \subfloat[ROC curve in DNNs]{\includegraphics[width=0.48\textwidth]{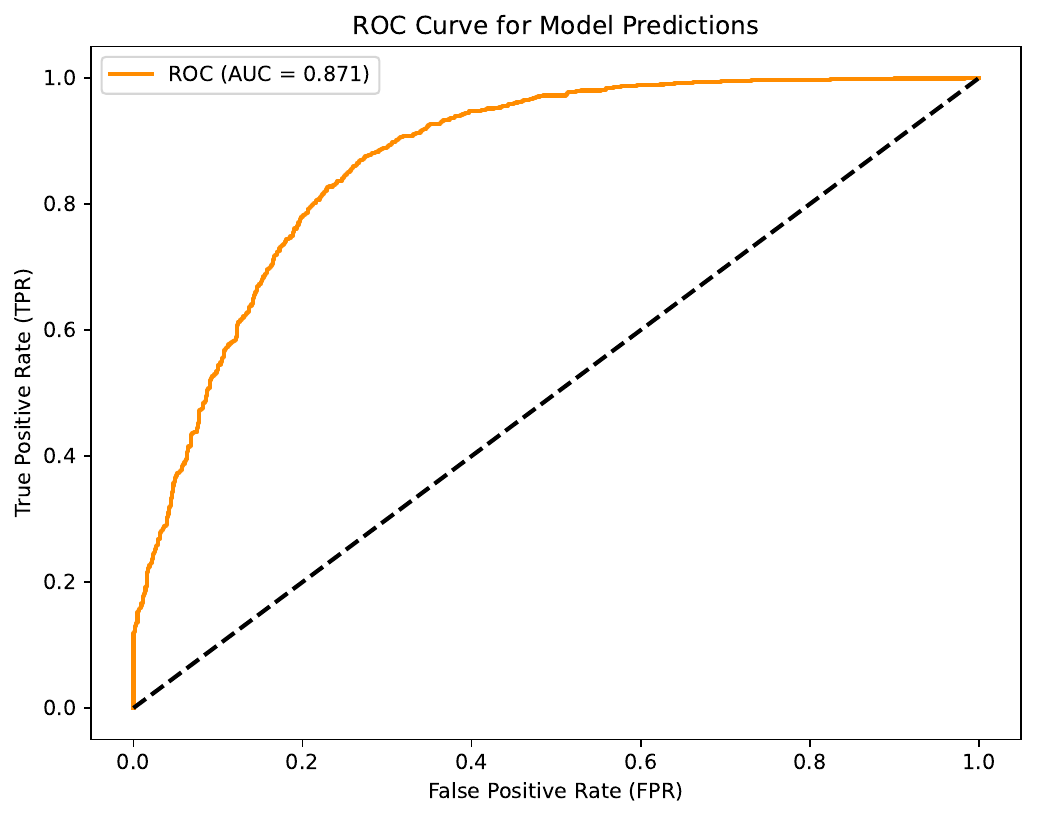}\label{fig_real:rocNNlog1}}
    \subfloat[Individual estimates and CIs in DNNs]{\includegraphics[width=0.48\textwidth]{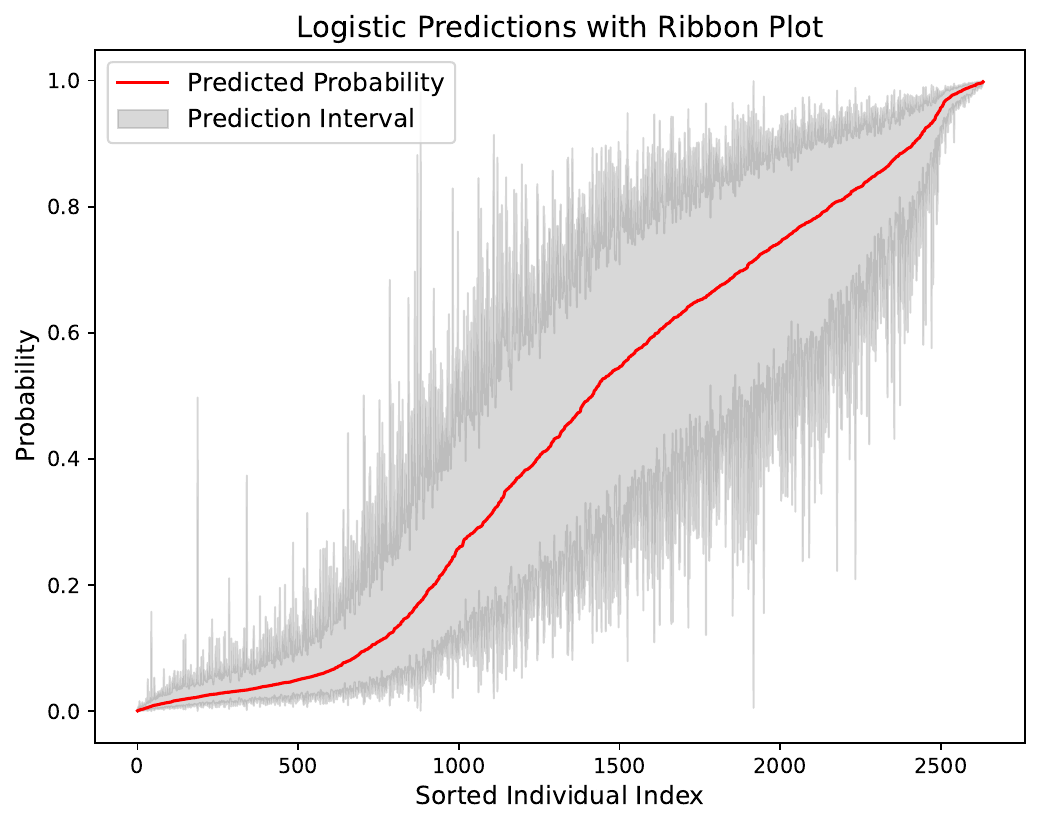}\label{fig_real:calibrationLogisticNN1}}

    \subfloat[ROC curve in RFs]{\includegraphics[width=0.48\textwidth]{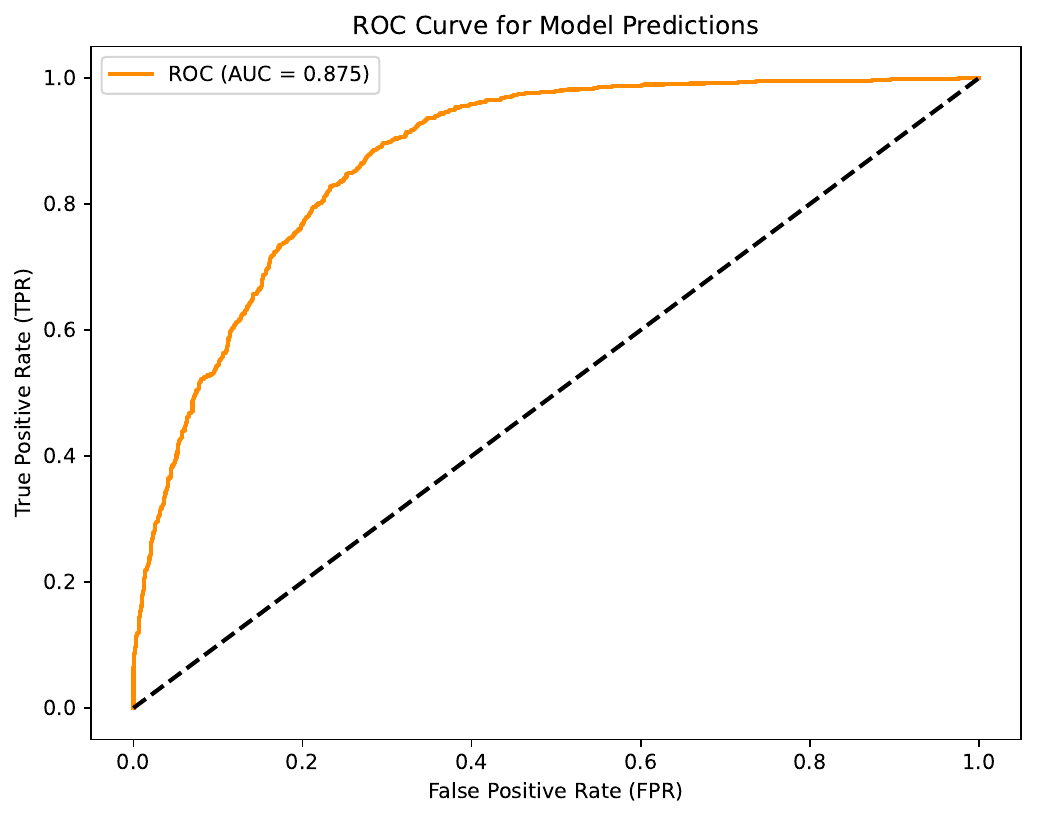}\label{fig_real:rocRFlog1}}
    \subfloat[Individual estimates and CIs in RFs]{\includegraphics[width=0.48\textwidth]{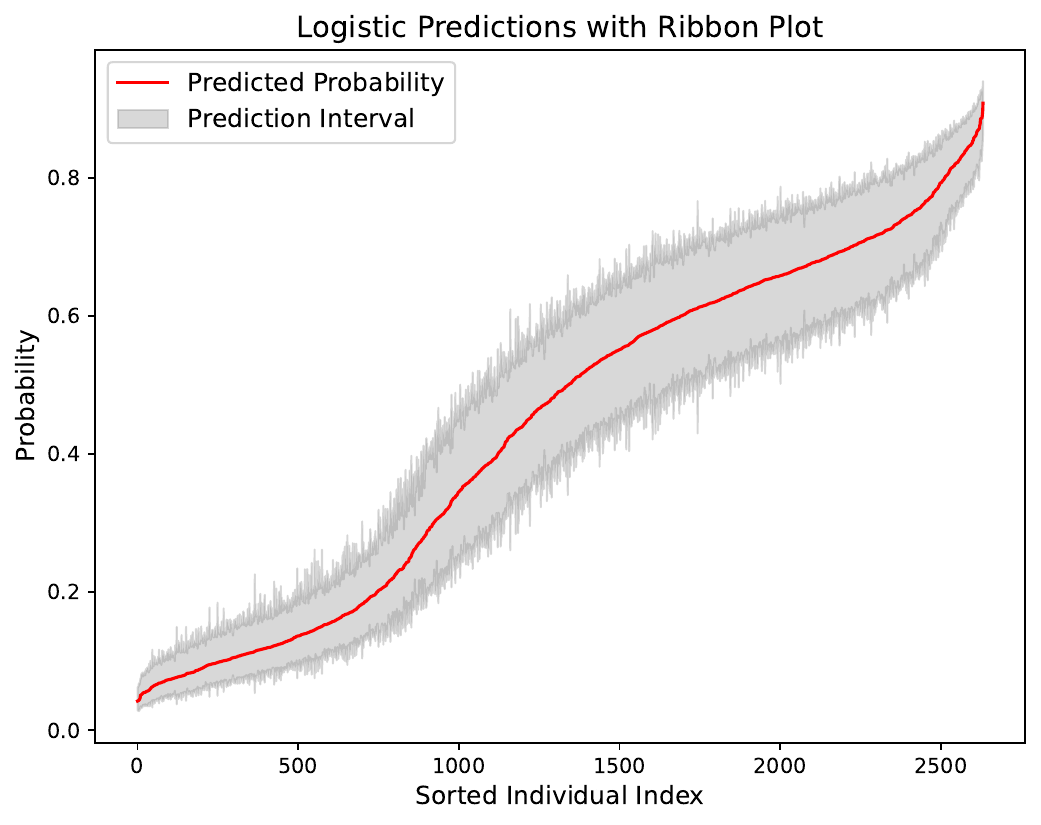}\label{fig_real:calibrationLogisticRF1}}
    \caption{ Evaluation of the nonparametric logistic model estimator. Figure~\ref{fig_real:rocNNlog1} and \ref{fig_real:rocRFlog1} show the ROC curve with both AUC of  0.84. Figure~\ref{fig_real:calibrationLogisticNN1} and \ref{fig_real:calibrationLogisticRF1} display estimated subject-level probabilities of ICU readmission and confidence intervals, illustrating  heteroskedasticity across individuals.}
    \label{figreal:logisticpart1}
\end{figure}

\begin{figure}[htbp]
    \centering
   \subfloat[Lift curve in DNNs]{\includegraphics[width=0.48\textwidth]{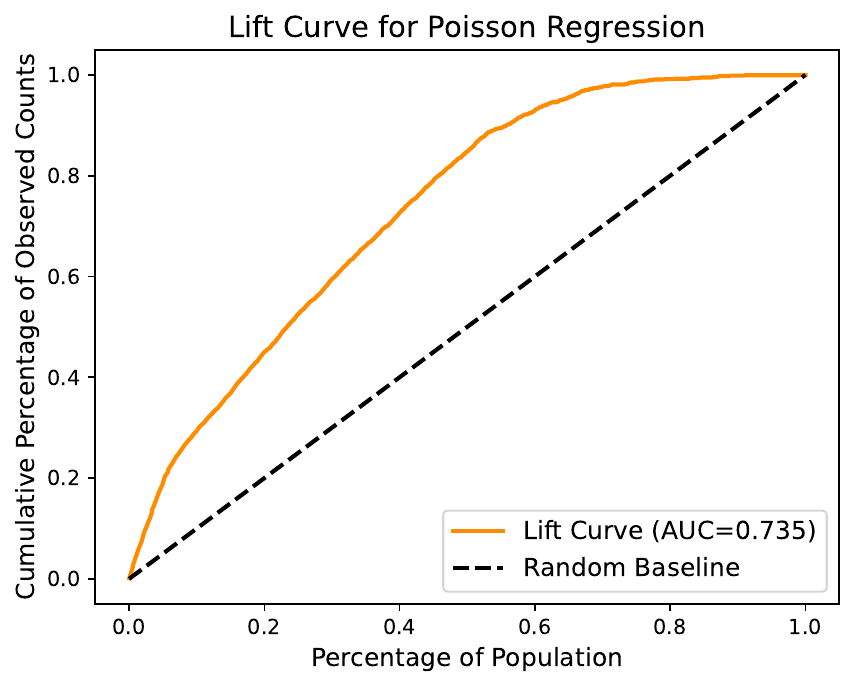}\label{fig_real:rocNNpoi1}}
    \subfloat[Individual estimates and CIs in DNNs]{\includegraphics[width=0.48\textwidth]{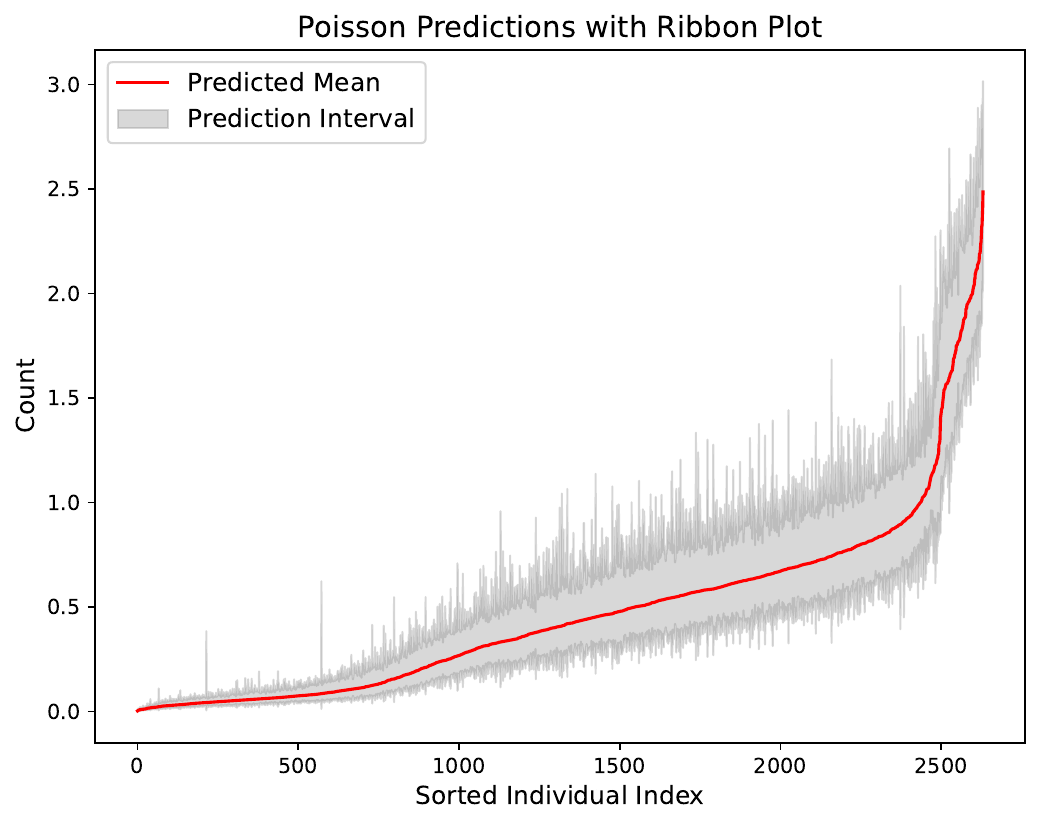}\label{fig_real:calibrationPoissonNN1}}

    \subfloat[Lift curve in RFs]{\includegraphics[width=0.48\textwidth]{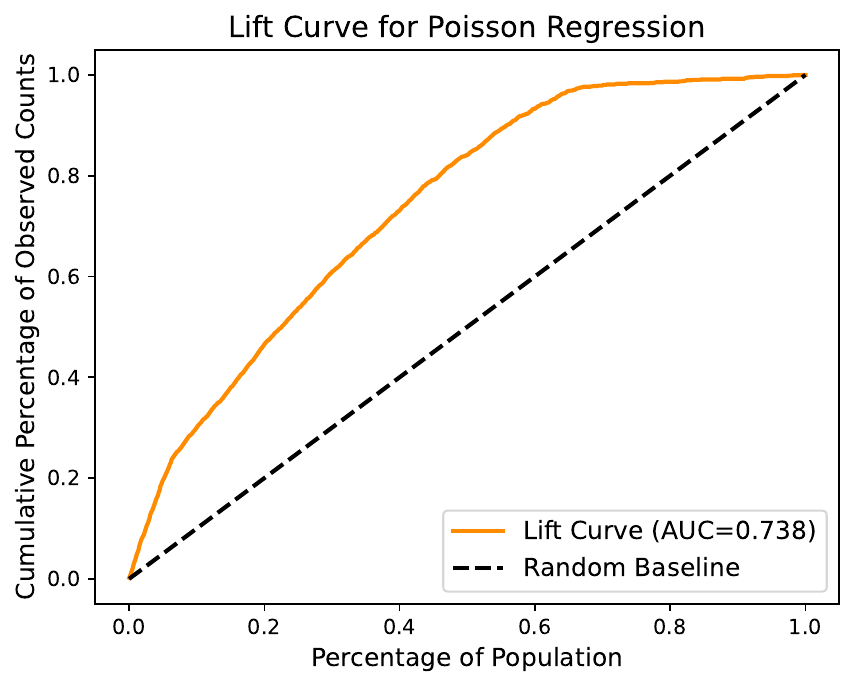}\label{fig_real:rocRFpoi1}}
    \subfloat[Individual estimates and CIs in RFs]{\includegraphics[width=0.48\textwidth]{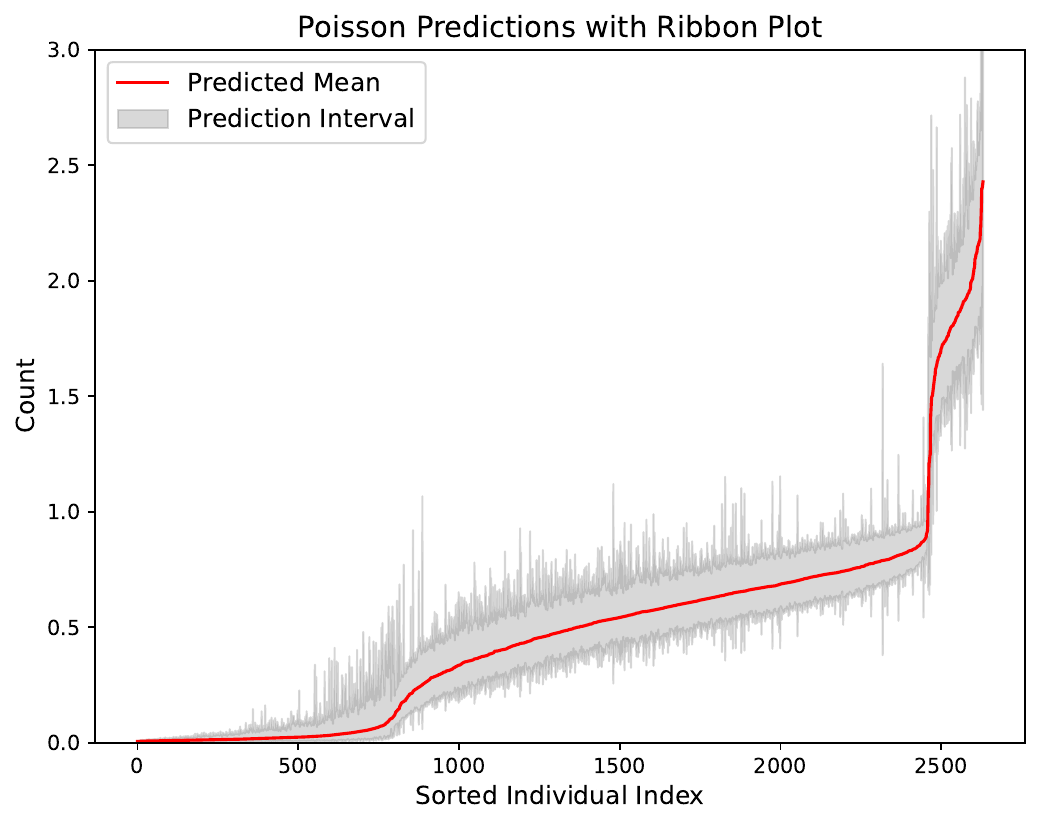}\label{fig_real:calibrationPoissonRF1}}
    \caption{ Evaluation of the nonparametric logistic model estimator. Figure~\ref{fig_real:rocNNpoi1} and \ref{fig_real:rocRFpoi1} show the ROC curve with both AUC of  0.84. Figure~\ref{fig_real:calibrationPoissonNN1} and \ref{fig_real:calibrationPoissonRF1} display estimated subject-level probabilities of ICU readmission and confidence intervals, illustrating  heteroskedasticity across individuals.}
    \label{figreal:poissonpart1}
\end{figure}
 As a sensitivity analysis, we extend the real data experiments by adding features such as RDW, MCHC, MCV, and admission indicators (e.g., Respiratory, Sepsis-related, Surgical/Trauma), expanding the covariate set to 46. This evaluates how model performance and uncertainty estimates respond to a richer feature space with added clinical signal and potential noise.
 Figures~\ref{figreal:logisticpart1} and \ref{figreal:poissonpart1} show results under the expanded feature set. For logistic regression, AUCs for DNNs and RFs remain similar, but DNNs adaptively widen their intervals (0.306 vs. 0.161), reflecting greater sensitivity to increased feature complexity. In the Poisson model, AUCs and interval lengths are comparable (0.261 for DNNs, 0.274 for RFs). These findings highlight the strength of DNNs in leveraging richer features and adapting uncertainty when signal dominates, while RFs maintain stability in noisier, high-dimensional settings, consistent with trends observed in simulation

\subsection{Additional Experiments on Cross-Site Model Transferability}
\label{sec:realdata_transfer}

We further transferred the  model  trained from hospital site 188 to hospital site 458. Site 458 was selected because it has a comparable patient volume to the original training site (188) but exhibits a markedly different outcome distribution. Although the cohort sizes are similar, the risk profiles differ. As shown in Figure~\ref{figreal:distribution_diff}, patients at site 458 generally tend to have a lower anticipated mortality likelihood (assessed by physicians at baseline), a main risk factor for ICU readmissions. Consequently, site 458 displays a  lower overall readmission rate. Specifically, it includes 2,178 patients without readmission and  111 patients with one or more readmissions (ranging from 1 to 4 events), resulting in a challenging and imbalanced scenario for model generalization.

\begin{figure}[t]
    \centering
    \includegraphics[width=0.8\linewidth]{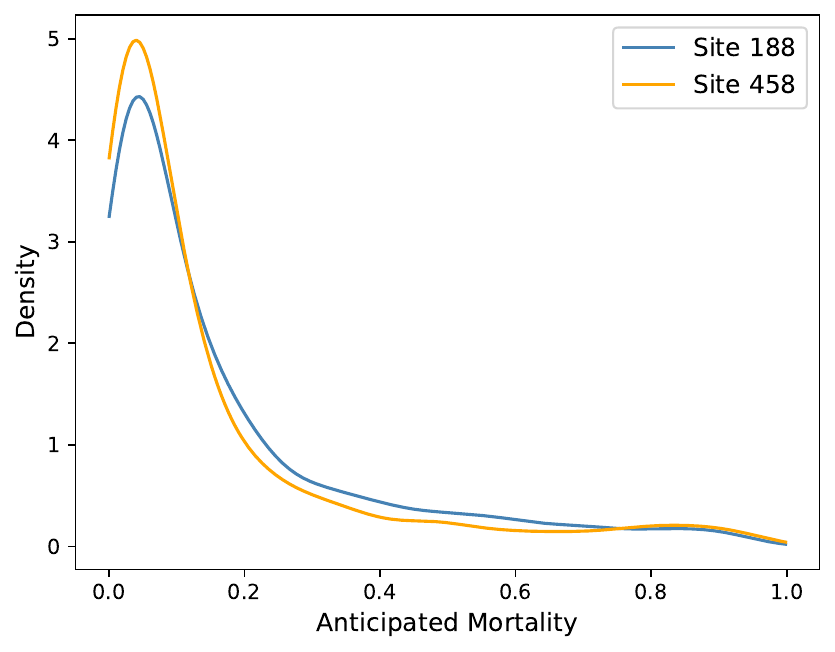}
    \caption{Density of the anticipated mortality.}
    \label{figreal:distribution_diff}
\end{figure}

As shown in Figure~\ref{figreal:transfer}, the transferred model maintains strong predictive performance, achieving a high AUC. In the logistic regression, the true positive rate is 0.92 and the false positive rate is 0.10 at a 0.5 threshold. Notably, the model continues to produce heteroskedastic prediction intervals that vary across individuals, indicating that the ESM-based uncertainty quantification remains responsive to patient-level heterogeneity even under distributional shift. Although this evaluation is limited to a single external site, the results provide encouraging evidence that the proposed method generalizes effectively and that the predictive intervals meaningfully capture site-specific differences in risk. Broader multi-site validation is warranted.

We also observed heteroskedasticity across individuals in the new site as well. Specifically, the variation in interval widths primarily stems from the input-dependent predictive variance intrinsic to the ESM framework. In ESM, each prediction $\hat{f}^B(\mathbf{x}_*)$ aggregates outputs from multiple subsampled models, leading to ensemble variance that naturally differs across regions of the feature space. Consequently, the estimated uncertainty adapts to local variability, yielding a heteroskedastic characterization of estimation uncertainty.
\begin{figure}[H]
    \centering
   \subfloat[ROC curve with Logistic]{\includegraphics[width=0.48\textwidth]{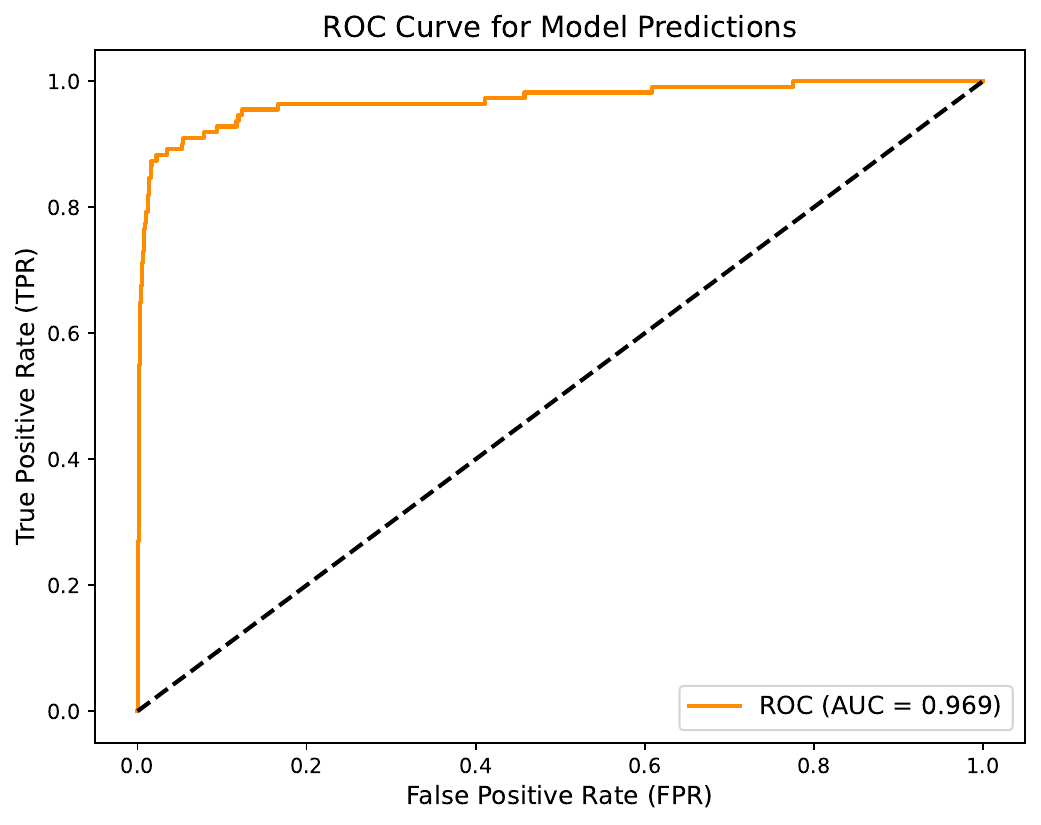}\label{fig_real:rocNNlog1transfer}}
    \subfloat[Individual estimates and CIs with Logistic]{\includegraphics[width=0.48\textwidth]{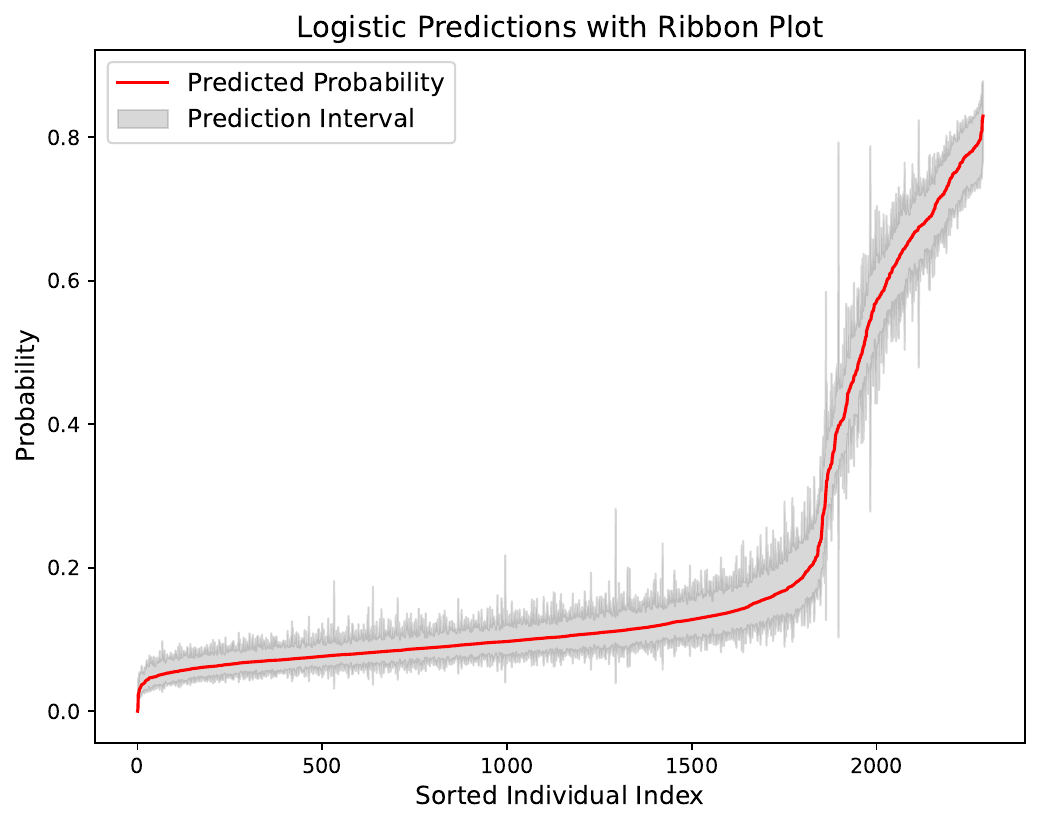}\label{fig_real:calibrationLogisticNN1transfer}}

    \subfloat[Lift curve with Poisson]{\includegraphics[width=0.48\textwidth]{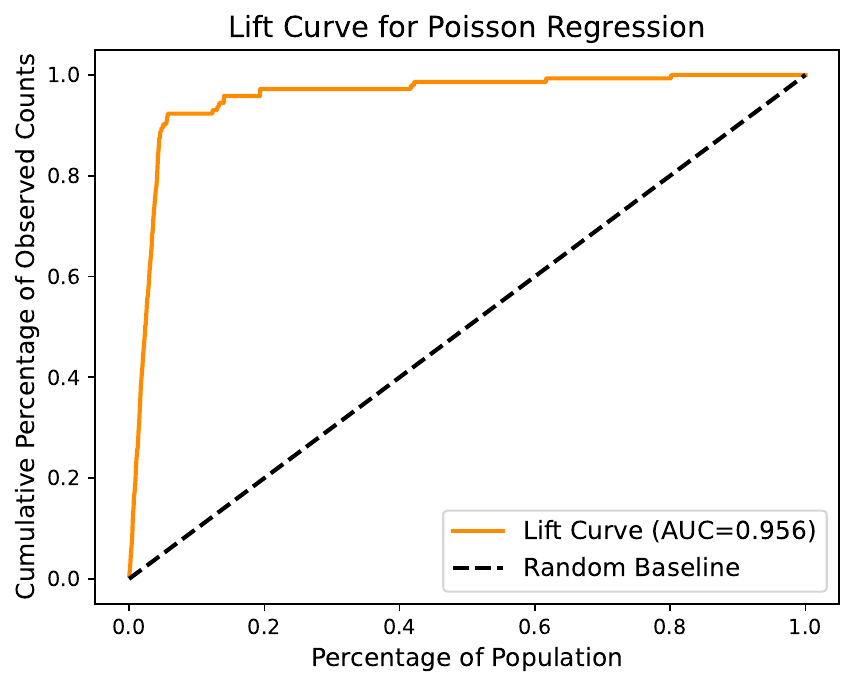}\label{fig_real:rocpoi1transfer}}
    \subfloat[Individual estimates and CIs with Poisson]{\includegraphics[width=0.48\textwidth]{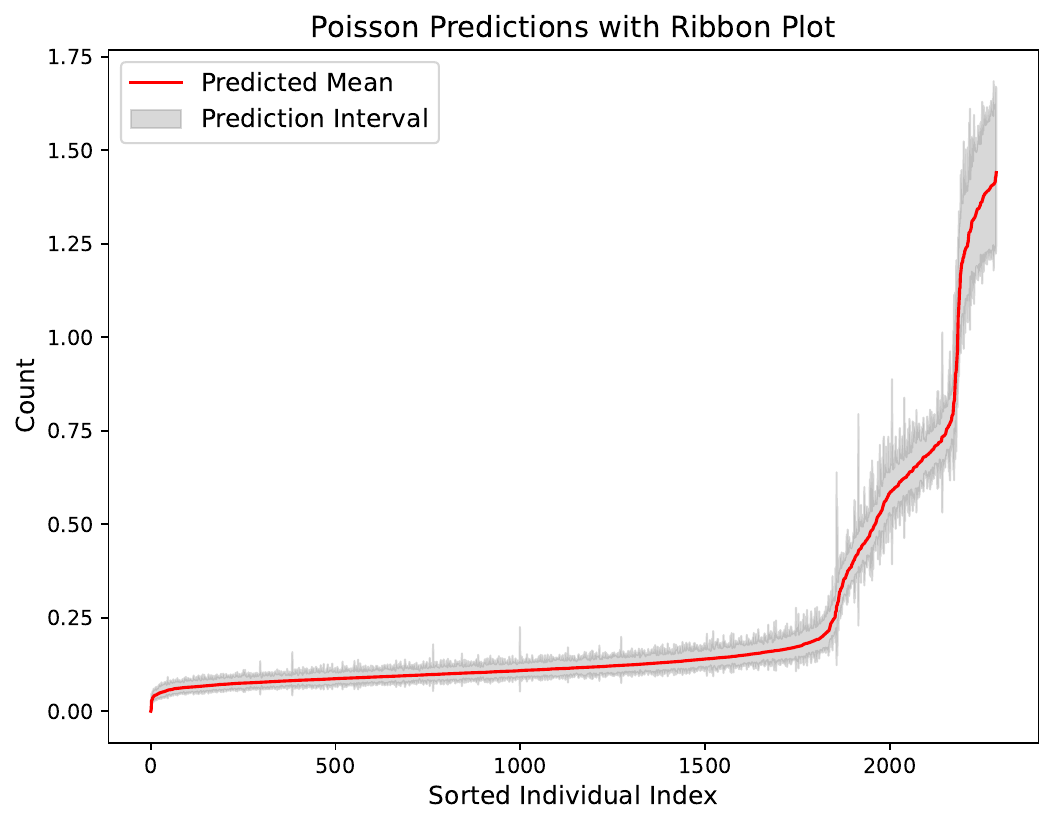}\label{fig_real:calibrationpoisson1transfer}}
    \caption{Evaluation of the transfered model from site 188 to site 458. Figure~\ref{fig_real:rocNNlog1transfer} and \ref{fig_real:calibrationLogisticNN1transfer} show the results under logistic regression. Figure~\ref{fig_real:rocpoi1transfer} and \ref{fig_real:calibrationpoisson1transfer} display results with poisson regression.}
    \label{figreal:transfer}
\end{figure}

\newpage
\begin{landscape}
\begin{table}[t] 
\centering
\renewcommand{\arraystretch}{1.15}
\setlength{\tabcolsep}{8pt}
\small
\caption{Checklist of Some Notations Used in the Paper.}
\label{table:notation}
\begin{tabular}{ll|ll}
\toprule
\textbf{Symbol} & \textbf{Description} & \textbf{Symbol} & \textbf{Description} \\ 
\midrule
$\xb_i$ & Input covariate for the $i$-th observation & 
$\EE[\cdot]$ & Expectation operator  \\
$y_i$ & Response variable &
$\Var[\cdot]$ & Variance operator \\
 $\zb_i$ & $(\xb_i,y_i)$ & $V_k$ & $\Var(T_k(\zb_{i_1},...,\zb_{i_k}))$   \\
$B$ & Number of subsampling times &
$J_{b_j i}$ & $I(i\in\cI^{b_j})$ indication function of $i$ in subset $b_j$\\
$\hf^b(\xb)$ & Estimator for each subsample &
$J_{\cdot i}$ & $\sum_{j=1}^BJ_{b_j i}/B$, average inclusion rate of sample $i$ across $B$ subsamples
 \\
$\hf^B(\xb)$ &  Ensemble estimator &
$Z_{b_j i}$ & $(J_{b_ji}-J_{\cdot i})(\hf^{b_j}(\xb_{*})-\hf^{B}(\xb_*))$ \\
$\ringf^B(\xb_{*})$ & $\sum_{i=1}^n\big( \EE \big(\hf^{B}(\xb_{*})|\zb_i\big)-\EE \hf^b(\xb_{*}) \big)$  & $M_i$ & $\EE(Z_{b_j i} |\zb_1,...,\zb_n)$, expectation takes over randomness of subsampling \\
$\xi_{1,r}(\xb_*)$& $\Cov(\hf(\xb_{*};\zb_1,\zb_2,...,\zb_{r}),\hf(\xb_{*},\zb_1,\zb'_2,...,\zb'_r))$ & $\varepsilon_{b_j i}$ & $Z_{b_j i}-M_i$, $\overline{\varepsilon}_i=\sum_{j=1}^B \varepsilon_{b_j i}/B$ \\
$g_{1,r}(\zb;\xb_{*}) $ & $\EE \hf^b(\xb_{*};\zb,\zb_2,...,\zb_r)-\EE \hf^b(\xb_{*})$  &
$\hat{V}_i$ & ${\sum_{j=1}^B Z_{b_j i}}/B$ \\
$T_k(\zb_1,...,\zb_k)$ & Hoeffding decomposition in \eqref{eq:decomp_hf_smallb} &
$\hat{\Var}_i$ & $\sum_{j=1}^{B}\bigl(Z_{b_ji}-\hat V_i\bigr)^{2}/(B-1)$ \\
$W_{i_1i_2...i_k}$ & Number of subsamples including $i_1,\ldots,i_k$ & $A_i$ & $\big(\frac{r}{n}-\frac{r^2}{n^2}\big)T_1(\zb_i)+\big(\frac{r(r-1)}{n(n-1)}-\frac{r^2}{n^2}\big)\sum_{j\neq i}^{n}T_1(\zb_j)$ \\
\bottomrule
\end{tabular}
\end{table}
\end{landscape}
\newpage
\section{ Additional Propositions and Lemmas and Proofs of Theorems}
\label{sec:proofthms}
For ease of reference, Table~\ref{table:notation} provides the key notation and glossary used throughout the paper, which will facilitate the technical developments that follow.

\subsection{Preliminaries}

We present the following proposition and lemmas, which are used to prove Theorem~\ref{thm:bound_Rn}. Let $\cN_n(\tilde{\delta}, \cF, \|\cdot\|_{\infty})$ denote the covering number, i.e., the minimal number of $\|\cdot\|_{\infty}$-balls with radius $\tilde{\delta}$ required to cover $\cF(L, \pb, s, F)$. For brevity, we write $\cN_n = \cN_n(\tilde{\delta}, \cF, \|\cdot\|_{\infty})$. Then, for any $\hf_n \in \cF(L, \pb, s, F)$, we have the following result.

\begin{proposition}
\label{prop:upper_lower_bound_R_n}
Under the conditions of Theorem~\ref{thm:bound_Rn}, there exist constants $c', C' > 0$ depending only on $\kappa$, the function $\psi$, $F$, and $K$, such that

\begin{align*}
&\frac{1}{2}\Delta_n(\hf_n)-c'\bigg(\frac{\log\cN_n}{n}+\bigg(\sqrt{\frac{\log\cN_n}{n}}+1 \bigg)\tilde{\delta}\bigg)
\quad\leq R_n(\hf_n,f_0)  \leq 2\Delta_n(\hf_n) \\& +2\inf_{f\in\cF(L,\pb,s,F)}\EE\ell(\bX;f,f_0)+C'\bigg(\frac{\log\cN_n}{n}+\bigg(\sqrt{\frac{\log\cN_n}{n}}+1 \bigg)\tilde{\delta}\bigg).
\end{align*}
\end{proposition}

The proof is provided in Section~\ref{sec:additional_proof}. To obtain the bounds in Proposition~\ref{prop:upper_lower_bound_R_n}, we apply truncation techniques that account for heterogeneity across individuals, paving the way for establishing convergence results under generalized nonparametric regression models. The following two lemmas provide properties of the covering number and characterize the minimal distance between functions $f \in \cF(L, \pb, s, F)$ and $f_0$.

\begin{lemma}
\label{lemma:bound_cN_n}
Let $\cN_n$ denote the covering number, i.e., the minimal number of $\|\cdot\|_{\infty}$-balls with radius $\tilde{\delta}$ required to cover $\cF(L, \pb, s, \infty)$. Then
\begin{align*}
\log(\cN_n) \leq (s+1) \log \left(2 \tilde{\delta}^{-1} (L+1) \left( \prod_{l=0}^{L+1} (1 + p_l) \right)^2 \right).
\end{align*}
\end{lemma}

\begin{lemma}
\label{lemma:bound_f_f0}
Under the conditions of Theorem~\ref{thm:bound_Rn}, it holds that
\begin{align*}
\inf_{f \in \cF(L, \pb, s, F)} \|f - f_0\|_{\infty}^2 \precsim \phi_n.
\end{align*}
\end{lemma}

\begin{proof}[Proofs of Lemmas~\ref{lemma:bound_cN_n} and~\ref{lemma:bound_f_f0}]
See Lemma 5 and the proof of Theorem 1 in \citet{schmidt2020nonparametric} for the proofs of Lemmas~\ref{lemma:bound_cN_n} and~\ref{lemma:bound_f_f0}.
\end{proof}

The following lemma establishes key properties of the function $\psi$ in the context of generalized nonparametric regression models.

\begin{lemma}
\label{lemma:key_property_psi}
Under Assumption~\ref{assump:1}, there exist constants $c', C', C_{\Lip} > 0$ such that for any $x, m \in [-C, C]$, where $C > 0$,
\begin{align*}
\frac{C'}{2} (x - m)^2 
\geq \psi(x) - \psi(m) - \psi'(m)(x - m) 
\geq \frac{c'}{2} (x - m)^2, \quad 
\psi(x) - \psi(m) \leq C_{\Lip} |x - m|.
\end{align*}
\end{lemma}

\subsection{Proof of Theorem~\ref{thm:bound_Rn}}

With Lemma~\ref{lemma:bound_cN_n} and if we let $\tilde{\delta}=n^{-1}$, it holds that
\begin{align*}
    \log(\cN_n)\leq (s+1)\log\Big(2n(L+1)(\max\{p_l\}+1)^{2L+4}\Big).
\end{align*}
Combining this inequality  with Assumption~\ref{assump:network}, we have $s\precsim n\phi_n\log n$, $\max\{p_l\}\precsim n $,  and the upper bound of $\log(\cN_n)$ can be written as 
\begin{align}
    \label{eq:bound_cN1}
\log(\cN_n)\precsim n\phi_n L \log^2 n.
\end{align}
By Proposition~\ref{prop:upper_lower_bound_R_n}, the lower bound for $R_n(\hf_n,f_0)$ can be written as 
\begin{align*}
    \frac{1}{2}\Delta_n(\hf_n)-R_n(\hf_n,f_0)&\leq c'\bigg(\frac{\log\cN_n}{n}+\bigg(\sqrt{\frac{\log\cN_n}{n}}+1 \bigg)n^{-1}\bigg)\\
    &\precsim \phi_nL\log^2 n. 
\end{align*}
 On the other hand, for the upper bound of $R_n(\hf_n,f_0)$,  Lemma~\ref{lemma:bound_f_f0} yields
\begin{align*}
\inf_{f\in\cF(L,\pb,s,F)}\|f-f_0\|_{\infty}^2 \precsim \phi_n.
\end{align*}
From Lemma~\ref{lemma:key_property_psi}, it follows that $\EE\ell(\bX;f,f_0)\precsim \|f-f_0\|_{\infty}^2$. Hence, combining the upper bound of $\log(\cN_n)$  in \eqref{eq:bound_cN1} and $\inf_{f\in\cF(L,\pb,s,F)}\|f-f_0\|_{\infty}^2 $ above, we have
\begin{align*}
    R_n(\hf_n,f_0)-2\Delta_n(\hf_n)\precsim \phi_nL\log^2 n.
\end{align*}

 \subsection{Proof of Proposition~\ref{prop:bias_honest}}
We adapt the proof technique in \citet{schmidt2020nonparametric} to construct the approximation error in such class of ReLU networks. Suppose that 
\begin{align*}
    \phi_r=r^{-\frac{2m}{2m+t}},
\end{align*}
and let $c=\frac{2m}{2m+t}$. Then we specifically construct the neural network $\tilde{f}$ with the approximation rate
\begin{align*}
    \sup_{f_0\in \cG}\|\tilde{f}-f_0\|_{\infty}\precsim \phi_r^{\frac{1+\delta(t+2m)/t}{2}}. 
\end{align*}
As we change the sparsity condition slightly,  this rate can be proved to be much faster than the  prediction error  with the assumption $s\asymp r^{1+\delta}\phi_r\log r$. Hence, the bias term vanishes faster than the term $\sqrt{n/r^2\inf_{\xb\in\cX}\xi_{1,r}(\xb)}$ by the lower bound of $\gamma$,  which completes the proof of Proposition~\ref{prop:bias_honest}.

\subsection{Proof of Theorem~\ref{thm:Inference}}
By Assumption~\ref{assump:4}, recall $\gamma> \frac{1}{1+\min_{j=0,...,q}(2m_j^*)/(2m_j^*+t_j)}$. For the sufficiently small $\varepsilon>0$, it is easy to see that 
\[
\lim_{n\to\infty}n\phi_rL\log^2(r)/( r^2\inf_{\xb\in\cX}\xi_{1,r}(\xb))=0, 
\]
there always exists a positive sequence $\delta_n$ with $\delta_n\to 0$ such that 
\[
\lim_{n\to\infty}n\phi_rL\log^2(r)/( \delta_n r^2\inf_{\xb\in\cX}\xi_{1,r}(\xb))=0. 
\]
For such $\delta_n$, we construct the set  $\cA_{\delta_n}$ as
\begin{align*}
\cA_{\delta_n}=\bigg\{\bX\in\cX: [\EE\hat{f}(\bX)-f_0(\bX)|\bX]^2   \leq \frac{C\phi_rL\log^2r}{\delta_n} \bigg\},
\end{align*}
By Chebyshev's inequality we have 
\begin{align*}
    \PP_{\xb}(\cA_{\delta_n})&=1-\PP_{\xb}\bigg([ \EE\hat{f}(\bX)-f_0(\bX)|\bX]^2   \geq \frac{C\phi_rL\log^2r}{\delta_n}\bigg)\\
    &\geq 1-\delta_n.
\end{align*}
Moreover, by the definition of $\cA_{\delta_n}$,  we have for every fixed \( \xb_{*} \in \cA_{\delta_n} \), 
\begin{align*}
    \sqrt{\frac{n}{r^2\xi_{1,r}(\xb_{*})}}\big|\EE (\hf^{b_j}(\xb_{*})-f_0(\xb_{*}))\big|=O\bigg(\sqrt{\frac{n\phi_r L\log^2 r} {\delta_n r^2\xi_{1,r}(\xb_{*})}}\bigg)\to0,
\end{align*}
where the last equality is by the condition $\lim_{n\to\infty}n\phi_rL\log^2(r)/(\delta_n r^2\inf_{\xb\in\cX}\xi_{1,r}(\xb))=0$.
By the basic inequality of $|\EE\hf^B(\xb_*)-f_0(\xb_*)|^2\leq \sum_{j=1}^B  (\EE \hf^{b_j}(\xb_{*})-f_0(\xb_{*}))^2/B$, we have that
\begin{align}
\sqrt{\frac{n}{r^2\xi_{1,r}(\xb_{*})}}\big|\EE (\hf^{B}(\xb_{*})-f_0(\xb_{*}))\big|=O\bigg(\sqrt{\frac{n\phi_r L\log^2 r} {\delta_n r^2\xi_{1,r}(\xb_{*})}}\bigg)\to0\label{eq:bias_disappear}
\end{align}
for any $\xb_*\in\cA_{\delta_n}$.

Assume the observed data of size $n$, $\cD_n=\{(y_i,\xb_i):i=1,...,n\}$, are independently and identically distributed copies of $(y,\bX)$, with $y\in\RR^1$ and $\bX\in\RR^{d\times 1}$. 
Let $\zb_i = (y_i, \xb_i)$ represent an independent observation of sample points, where $y_i$ is the label and $\xb_i$ is the response. By the definition of $\xi_{1,r}(\xb)$, it is clear that 
\begin{align*}
    \xi_{1,r}(\xb_{*})=\Cov(\hf(\xb_{*};\zb_1,\zb_2,...,\zb_{r}),\hf(\xb_{*},\zb_1,\zb'_2,...,\zb'_r)).
\end{align*}
Here, $\hf(\xb_{*}; \zb_1, \zb_2, \dots, \zb_r)$ means that we obtain $\hf$ from the subsample $\zb_1, \zb_2, \dots, \zb_r$, and then apply it to the point $\xb_{*}$. The terms $\zb'_i$ and $\zb_i$ are independently generated from the same data generation process. 
 In the following of the proofs, all the expectations and variance are taken over the fixed $\xb_{*}$.

Next, we consider the asymptotic normality of the first order dominants of the Hoeffding decomposition. By the definition of $\hf^B(\xb_{*})$, it is clear that $\hf^{b_j}$ and $\hf^{b_{j'}}$ are not independent given $\cI^{b_j}$ and $\cI^{b_{j'}}$ have overlap.  To handle the dependency, we consider the Hoeffding decomposition. For the $n$ sample data points $\zb_i=(y_i,\xb_i)$, define 
\begin{align*}
    \ringf^B(\xb_{*})=\sum_{i=1}^n\bigg( \EE \big(\hf^{B}(\xb_{*})|\zb_i\big)-\EE \hf^b(\xb_{*}) \bigg).
\end{align*}
Here, the expectation  $\EE$ is taken over the training sample.  By the definition of $\hf^B(\xb_{*})$, we have
\begin{align*}
    \EE \big(\hf^{B}(\xb_{*})-\big(\EE \hf^b(\xb_{*})\big)|\zb_i\big)=\EE \bigg(\frac{1}{B}\sum_{j=1}^B\hf^{b_j}(\xb_{*})|\zb_i \bigg)-\EE \hf^b(\xb_{*}). 
\end{align*}
Let $W_{b_j}$ be the number of subsamples $\cI^{b_j}$ that contain index $i$, and define $g_{1,r}(\zb;\xb_{*})=\EE \hf^b(\xb_{*};\zb,\zb_2,...,\zb_r)-\EE \hf^b(\xb_{*})$. We have $W_{b_j}$ is independent of $\zb_i$ and $\EE W_{b_j}=B{{n-1}\choose{r-1}} /{{n}\choose{r}} $. Hence, the equation above involving the term $\zb_i$ can be further expressed as 
\begin{align*}
    &\EE \big(\hf^{B}(\xb_{*})-\big(\EE \hf^b(\xb_{*})\big)|\zb_i\big)\\
    &\qquad=\frac{1}{B} \sum_{j=1}^B\EE_{W_{b_j}} \bigg(\EE (\hf^{b_j}(\xb_{*})-\big(\EE \hf^b(\xb_{*})\big)|\zb_i,W_{b_j})\bigg)\\
    &\qquad=\frac{1}{B}\EE_{W_{b_j}} W_{b_j} g_{1,r}(\zb_i;\xb_{*})+\frac{1}{B}\EE_{W_{b_j}}(B-W_{b_j})\cdot \EE (\hf^{b_j}(\xb_{*})-\EE \hf^b(\xb_{*}))\\
    &\qquad=\frac{r}{n}g_{1,r}(\zb_i;\xb_{*}).
\end{align*}
Therefore we conclude that 
\begin{align}
    \ringf^B(\xb_{*})=\frac{r}{n}\sum_{i=1}^n g_{1,r}(\zb_i;\xb_{*}). \label{eq:CLT_order1_sum}
\end{align}
Equation~\eqref{eq:CLT_order1_sum} provides a nice form, as $g_{1,r}(\zb_i)$ is independent of each other. Moreover, by the assumption that $\hf^b$ is bounded, the Lindeberg condition is satisfied, and it is clear that 
\begin{align*}
    \frac{\sqrt{n}}{r\cdot \sqrt{\Var(g_{1,r}(\zb_i;\xb_{*}))}}\cdot  \ringf^B(\xb_{*}) ~\cvd~\cN(0,1).
\end{align*}
By the definition of $\xi_{1,r}(\xb_{*})$, it is easy to see that $\Var(g_{1,r}(\zb_i;\xb_{*}))=\xi_{1,r}(\xb_{*})$, hence we have
\begin{align}
\sqrt{\frac{n}{r^2\xi_{1,r}(\xb_{*})}}\cdot  \ringf^B(\xb_{*}) ~\cvd~\cN(0,1).\label{eq:step1_conclusion}
\end{align}

In the next step, we would like to substitute $\ringf^B(\xb_{*})$ with $\hf^B(\xb_{*})-\EE\hf^b(\xb_{*})$. Applying the Hoeffding decomposition, it is clear that the simple $\hf^b(\xb_{*};\zb_1,...,\zb_r)$ can be written as 
\begin{align}
    \hf^b(\xb_{*};\zb_1,...,\zb_r)-\EE\hf^b(\xb_{*})=\sum_{i=1}^r T_1(\zb_i)+\sum_{i<j}T_2(\zb_i,\zb_j)...+T_r(\zb_1,...,\zb_r).\label{eq:decomp_hf_smallb} 
\end{align}
The $2^r-1$ random variables on the right hand side are all mean-zero and uncorrelated. Here, $T_1(\zb_i)=\EE\big(\hf^b(\xb_{*};\zb_1,...,\zb_r)|\zb_i\big)-\EE\hf^b(\xb_{*})$. Applying the decomposition to $\hf^B(\xb_{*})$ and utilizing the same trick we obtain \eqref{eq:CLT_order1_sum}, we have  
\begin{align*}
    \EE \big(\hf^B(\xb_{*})-\EE\hf^b(\xb_{*})|\zb_1,...,\zb_n\big)&=\frac{1}{{{n}\choose{r}}}\Bigg({{n-1}\choose{r-1}}\sum_{i=1}^nT_1(\zb_i)+{{n-2}\choose{r-2}}\sum_{i<j}T_2(\zb_i,\zb_j)\\
    &\quad+...+\sum_{i_1<...<i_r}T_r(\zb_{i_1},...,\zb_{i_r})\Bigg).
\end{align*}
Hence, it is clear that $\ringf^B(\xb_{*})$ is exactly the first order in the Hoeffding decomposition. We have
\begin{align*}
    &\EE\big(\hf^B(\xb_{*})-\EE\hf^b(\xb_{*})-\ringf^B(\xb_{*})|\zb_1,...,\zb_n\big)\\
    \qquad &=\frac{1}{{{n}\choose{r}}}\Bigg({{n-2}\choose{r-2}}\sum_{i<j}T_2(\zb_i,\zb_j)+...+\sum_{i_1<...<i_r}T_r(\zb_{i_1},...,\zb_{i_r})\Bigg).
\end{align*}
Denote by $V_k=\Var(T_k(\zb_{i_1},...,\zb_{i_k}))$, 
we then calculate the value of $\EE\big( \hf^B(\xb_{*})-\EE\hf^b(\xb_{*})-\ringf^B(\xb_{*})\big)^2$, which we will show later this term is negligible. To calculate this term, we apply the equality  $\EE X_n^2=\EE \Var(X_n|Y_n)+\EE (\EE X_n|Y_n)^2$. Then the term $\EE\big( \hf^B(\xb_{*})-\EE\hf^b(\xb_{*})-\ringf^B(\xb_{*})\big)^2$ can be written as 
\begin{align*}
\EE\big( \hf^B(\xb_{*})-\EE\hf^b(\xb_{*})-\ringf^B(\xb_{*})\big)^2&=\underbrace{\EE\big( \Var\big(\hf^B(\xb_{*})-\EE\hf^b(\xb_{*})-\ringf^B(\xb_{*})|\zb_1,...,\zb_n\big)\big)}_{J_1}\\
&\quad+\underbrace{\EE\big( \EE\big(\hf^B(\xb_{*})-\EE\hf^b(\xb_{*})-\ringf^B(\xb_{*})|\zb_1,...,\zb_n\big)\big)^2}_{J_2}.
\end{align*}
Here, we let $X_n=\hf^B(\xb_{*})-\EE\hf^b(\xb_{*})-\ringf^B(\xb_{*})$ and $Y_n=\zb_1,...,\zb_n$. From the expression of $\EE\big(\hf^B(\xb_{*})-\EE\hf^b(\xb_{*})-\ringf^B(\xb_{*})|\zb_1,...,\zb_n\big)$ above, we can easily see that
\begin{align*}
    J_2&=\frac{1}{{{n}\choose{r}}^2}\cdot\sum_{k=2}^{r}{{n-k}\choose{r-k}}^2\cdot {{n}\choose{k}}V_k\\
    &=\sum_{k=2}^{r}\frac{(n-k)!r!}{n!(r-k)!}{{r}\choose{k}}V_k\\
    &\leq \frac{r(r-1)}{n(n-1)}\cdot \sum_{k=2}^{r} {{r}\choose{k}}V_k\\
    &\leq\frac{r(r-1)}{n(n-1)}\Var(\hf^b(\xb_{*};\zb_1,...,\zb_r)).
        \end{align*}
Here, the first equality is by the fact that all the terms are uncorrelated, the first inequality is by the decreasing monotonicity of $\frac{(n-k)!r!}{n!(r-k)!}$ with respect to $k$, and the last inequality is by the decomposition of $\hf^b$ in \eqref{eq:decomp_hf_smallb}.
To calculate $\EE\big( \hf^B(\xb_{*})-\EE\hf^b(\xb_{*})-\ringf^B(\xb_{*})\big)^2$, it remains for us to calculate $J_1$. Let $W_{i_1i_2...i_k}$ ($k\leq r$) be the number of the $B$ subsamples that contain index $i_1,...,i_k$ simultaneously. Then we have
\begin{align*}
    \hf^B(\xb_{*})-\EE\hf^b(\xb_{*})-\ringf^B(\xb_{*})=\frac{1}{B}\bigg( \sum_{i_1<i_2}^n W_{i_1i_2}T_2(\zb_{i_1},\zb_{i_2})+...+\sum_{i_1<...<i_r}^n W_{i_1...i_r}T_r(\zb_{i_1},...,\zb_{i_r})\bigg).
\end{align*}
By the independence on the different terms $T_k$, we conclude that
\begin{align*}
    J_1&=\frac{1}{B^2}\sum_{k=2}^r \Var(W_{i_1...i_k}) {{n}\choose{k}} V_k\\
    &\leq \frac{1}{B}\sum_{k=2}^r \frac{{{n-k}\choose{r-k}}}{{{n}\choose{r}}} {{n}\choose{k}} V_k\\
    &=\frac{1}{B}\sum_{k=2}^r{{r}\choose{k}}V_k\\
    &\leq \frac{\Var(\hf^b(\xb_{*};\zb_1,...,\zb_r))}{B}.
\end{align*}
Here, the first inequality is by $\Var(W_{i_1...i_k})\leq B\cdot\frac{{{n-k}\choose{r-k}}}{{{n}\choose{r}}}$ and the uncorrelation on different terms $T_k$, and the last inequality is by $\sum_{k=2}^r{{r}\choose{k}}V_k\leq \Var(\hf^b(\xb_{*};\zb_1,...,\zb_r))$. We hence conclude that
\begin{align*}
\EE\big( \hf^B(\xb_{*})-\EE\hf^b(\xb_{*})-\ringf^B(\xb_{*})\big)^2&=J_1+J_2\\
&\leq \bigg(\frac{r(r-1)}{n(n-1)}+\frac{1}{B}\bigg)\Var(\hf^b(\xb_{*};\zb_1,...,\zb_r)),
\end{align*}
leading to
\begin{align}
\frac{n}{r^2\xi_{1,r}(\xb_{*})}\cdot \EE\big( \hf^B(\xb_{*})-\EE\hf^b(\xb_{*})-\ringf^B(\xb_{*})\big)^2 &\leq \bigg( \frac{1}{n\xi_{1,r}(\xb_{*})}+\frac{n}{Br^2\xi_{1,r}(\xb_{*})}\bigg)\Var(\hf^b(\xb_{*};\zb_1,...,\zb_r))\nonumber\\ 
&~\cvp~0.\label{eq:step2_key}
\end{align}
Here, we utilize the fact that $\hf^b(\xb_{*})$ is bounded. The bound of $\hf^b(\xb_{*})$  indicates that $\Var(\hf^b(\xb_{*};\zb_1,\ldots,\zb_r))$ is bounded. Additionally, we rely on Assumption~\ref{assump:4} that $\liminf_{n\to+\infty}n\xi_{1,r}(\xb_{*})\geq n^{1-\varepsilon-\gamma} \to +\infty $ and $\liminf_{n\to+\infty}Br^2\xi_{1,r}(\xb_{*}) /n\to +\infty $. From \eqref{eq:step2_key}, we have
\begin{align*}
    \sqrt{\frac{n}{r^2\xi_{1,r}(\xb_{*})}}\cdot  (\hf^B(\xb_{*})-\EE \hf^b(\xb_{*})-\ringf^B(\xb_{*})) ~\cvp~0.
\end{align*}
Therefore it holds that
\begin{align*}
    \sqrt{\frac{n}{r^2\xi_{1,r}(\xb_{*})}}\cdot  (\hf^B(\xb_{*})-\EE \hf^b(\xb_{*})-\ringf^B(\xb_{*})) ~\cvp~0.
\end{align*}
From \eqref{eq:step1_conclusion}, we also have that
\begin{align*}
       \sqrt{\frac{n}{r^2\xi_{1,r}(\xb_{*})}}\cdot  \ringf^B(\xb_{*}) ~\cvp~\cN(0,1),
\end{align*}
Combing the convergence above, we have
\begin{align*}
\sqrt{\frac{n}{r^2\xi_{1,r}(\xb_{*})}}\cdot  (\hf^B(\xb_{*})-\EE \hf^b(\xb_{*})) ~\cvd~\cN(0,1).
\end{align*}
The only thing that needs to prove is 
\begin{align*}
    \sqrt{\frac{n}{r^2\xi_{1,r}(\xb_{*})}}\cdot  (f_0(\xb_{*})-\EE \hf^b(\xb_{*})) ~\to~0
\end{align*}
for $\xb_{*}\in\cA_{\delta_n}$. 
This is proved by \eqref{eq:bias_disappear} and $\EE\hf^B(\xb_*)=\EE \hf^b(\xb_*)$. Wrapping all together, we complete the proof of Theorem~\ref{thm:Inference}. 

\subsection{Proof of Theorem~\ref{thm:estimate_xi1r}}
Recall the definition  
\begin{align*}
Z_{b_j i}=(J_{b_ji}-J_{\cdot i})(\hf^{b_j}(\xb_{*})-\hf^{B}(\xb_*)),\quad \hat{V}_i=\frac{\sum_{j=1}^B Z_{b_j i}}{B}.
\end{align*}
We denote by
\begin{align*}
    \tilde{\sigma}^2=\frac{n(n-1)}{(n-r)^2} \sum_{i=1}^n \hat{V}_i^2.
\end{align*}
Without loss of generality, we calculate the term $\hat{V}_1$, and the terms $\hat{V}_i$ with $i\neq 1$ are exactly the same. Consider  $\hat{V}_1^2$ and $B/n \to c$ for some constant $c>0$.   Then by the inspiration of the proof of Lemma 13 of \citet{wager2018estimation}, we can define 
\begin{align*}
M_1&=\bigg(\frac{r}{n}-\frac{r^2}{n^2}\bigg)T_1(\zb_1)+\bigg(\frac{r(r-1)}{n(n-1)}-\frac{r^2}{n^2}\bigg)\sum_{i=2}^{n}T_1(\zb_i)\\
&\quad+\frac{r}{n}\sum_{k=1}^{r-1}\bigg(\frac{{{r-1}\choose{k}}}{{{n-1}\choose{k}}}-\frac{{{r}\choose{k+1}}}{{{n}\choose{k+1}}} \bigg)\sum_{2\leq i_1<...<i_k\leq n}T_{k+1}(\zb_1,\zb_{i_1},...,\zb_{i_k})\\
&\quad +\frac{r}{n}\sum_{k=2}^{r}\bigg(\frac{{{r-1}\choose{k}}}{{{n-1}\choose{k}}}-\frac{{{r}\choose{k}}}{{{n}\choose{k}}} \bigg)\sum_{2\leq i_1<...<i_k\leq n}T_{k}(\zb_{i_1},...,\zb_{i_k}), 
\end{align*} 
and $M_i$ the same way with index $1$ of $\zb$ be replaced by $i$ of $\zb$. Let the whole data set $\cD_n=\{\zb_1,...,\zb_n\}$, and recall $Z_{b_j i}=(J_{b_ji}-J_{\cdot i})(\hf^{b_j}(\xb_{*})-\hf^{B}(\xb_*))$, the motivation of the definition of $M_i$ is given by the direct calculation that 
\begin{align*}
    \EE (Z_{b_j i}|\cD_n) =M_i.
\end{align*}
Here, the expectation takes over the randomness of subsampling. Define 
\begin{align*}
    \hsigma_1^2=\frac{n(n-1)}{(n-r)^2} \sum_{i=1}^n M_i^2, \quad R_B=\tilde{\sigma}^2-\hsigma_1^2. 
\end{align*}
We next prove two key conclusions, which will show the consistency of our results.
\begin{align}
\label{eq:target_to_prove}
    \begin{aligned}
&\frac{n(n-1)}{(n-r)^2}\frac{1}{B(B-1)}
\sum_{i=1}^{n}\sum_{j=1}^{B}\bigl(Z_{b_ji}-\hat V_i\bigr)^{2} - R_B=o_p\bigg(\frac{r^2\xi_{1,r}(\xb_{*})}{n} \bigg), \\   &\sqrt{\frac{r^2\xi_{1,r}(\xb_{*})}{n\hsigma_1^2}}~\cvp~1, 
    \end{aligned}
\end{align}
In the following of the proof, we denote by $\hat R_B=\frac{n(n-1)}{(n-r)^2}\frac{1}{B(B-1)}
\sum_{i=1}^{n}\sum_{j=1}^{B}\bigl(Z_{b_ji}-\hat V_i\bigr)^{2}$. If \eqref{eq:target_to_prove} holds, we can see that 
\begin{align*}
    \frac{n\hsigma_*^2}{r^2\xi_{1,r}(\xb_{*})}=\frac{n(\tilde{\sigma}^2-\hat R_B)}{r^2\xi_{1,r}(\xb_{*})}=\frac{n(R_B-\hat R_B)}{r^2\xi_{1,r}(\xb_{*})}+\frac{n(\tilde{\sigma}^2-R_B)}{r^2\xi_{1,r}(\xb_{*})}=o_p(1)+\frac{n \hsigma_1^2}{r^2\xi_{1,r}(\xb_{*})}~\cvp~1,
\end{align*}
which completes the proof of Theorem~\ref{thm:estimate_xi1r}. It remains for us to prove \eqref{eq:target_to_prove}.

We first investigate the term $\hat R_B-R_B=\frac{n(n-1)}{(n-r)^2}\frac{1}{B(B-1)}
\sum_{i=1}^{n}\sum_{j=1}^{B}\bigl(Z_{b_ji}-\hat V_i\bigr)^{2} - R_B$ in \eqref{eq:target_to_prove}. Recall  $R_B=\tilde{\sigma}^2-\hsigma_1^2$, it is easy to see that 
\begin{align*}
    R_B=\frac{n(n-1)}{(n-r)^2} \sum_{i=1}^n (\hat V_i^2-M_i^2).
\end{align*}
Let $\varepsilon_{b_j i}=Z_{b_j i}-M_i$, and $\overline{\varepsilon}_i=\frac{1}{B}\sum_{j=1}^B \varepsilon_{b_j i}$, recall $\hat V_i=\frac{\sum_{j=1}^B Z_{b_j i}}{B}$, we have that $\hat V_i=M_i+ \overline{\varepsilon}_i$ and $\hat V_i^2-M_i^2=2M_i \overline{\varepsilon}_i+\overline{\varepsilon}_i^2$, which means
\begin{align*}
R_B=\frac{n(n-1)}{(n-r)^2} \sum_{i=1}^n 2M_i \overline{\varepsilon}_i+\overline{\varepsilon}_i^2.
\end{align*}
Further denote by $\hat \Var_i=\frac{1}{B-1}\sum_{j=1}^{B}\bigl(Z_{b_ji}-\hat V_i\bigr)^{2}$, we have  
\begin{align*}
    \hat R_B= \frac{n(n-1)}{(n-r)^2} \frac{1}{B}\sum_{i=1}^n \hat \Var_i.
\end{align*}
We then combine the expression of  $ \hat R_B$ and $R_B$ above and get that 
\begin{align*}
    R_B-\hat R_B=\frac{n(n-1)}{(n-r)^2}\bigg( \underbrace{\sum_{i=1}^n2M_i\overline{\varepsilon}_i}_{\Delta_1}+\underbrace{\sum_{i=1}^n\overline{\varepsilon}_i^2-B^{-1}\Var_i}_{\Delta_2} -\underbrace{\frac{1}{B}\sum_{i=1}^n \hat \Var_i-\Var_i}_{\Delta_3}\bigg) .
\end{align*}
Here, $\Var_i=\Var_b(Z_{b_j i})$ with variance condition on the dataset $\cD_n$. We bound the three terms one by one. 
For the term $\Delta_1$,  by the definition of $M_i$ and $\overline{\varepsilon_i}$,  we have
\begin{align*}
    \Delta_1=\frac{2}{B}\sum_{j=1}^B \sum_{i=1}^n M_i\varepsilon_{b_j i}. 
\end{align*}
Conditional on the dataset $\cD_n$, $\sum_{i=1}^n M_i\varepsilon_{b_j i}$ is independent with different $b_j$. It is easy to see that $\EE_b \Delta_1=0$, and we give the calculation of $\Var_b \Delta_1$. From the expression of $\Delta_1$, conditional on data set $\cD_n$, $\varepsilon_{b_j i}$ are independent for different $b_j$, it is easy to see that 
\begin{align*}
    \Var_b \Delta_1  \asymp \frac{\Var_b \sum_{i=1}^n M_i\varepsilon_{b_j i}}{B}.
\end{align*}
It  remains for us to investigate the term in the numerator to bound $\Delta_1$. We have 
\begin{align*}
\Var_b \sum_{i=1}^n M_i\varepsilon_{b_j i}=\sum_{i=1}^n M_i^2 \Var_b \varepsilon_{b_j i}+\sum_{i\neq i'}M_i M_j \Cov_b(\varepsilon_{b_j i},\varepsilon_{b_j i'}). 
\end{align*}
From the definition of $M_i$, we can easily bound $M_i=O(r/n)$, for the covariance of $\varepsilon_{b_j i}$ and $\varepsilon_{b_j i'}$, simple calculation of the subsampling randomness gives us that $\Var_b \varepsilon_{b_j i}=O(\frac{r}{n}(1-\frac{r}{n})=O(r/n)$, and $\Cov_b(\varepsilon_{b_j i},\varepsilon_{b_j i'})=O(|\frac{r(r-1)}{n(n-1)}-\frac{r^2}{n^2})|=O(\frac{r}{n^2})$, we hence have 
\begin{align*}
    \Var_b \sum_{i=1}^n M_i\varepsilon_{b_j i}&=O\bigg(n\cdot \frac{r^2}{n^2}\cdot \frac{r}{n}+n^2\cdot \frac{r^2}{n^2}\cdot \frac{r}{n^2} \bigg)\\
    &=O\bigg( \frac{r^3}{n^2}\bigg).
\end{align*}
This indicates that $\Var_b \Delta_1=O(\frac{r^{3}}{Bn^2}) $. Combining the fact that $\EE_b \Delta_1=0$, we conclude that $|\Delta_1|=O_p(\frac{r^{3/2}}{n\sqrt{B}})$.

For the term $\Delta_2$, by the expression that $\overline{\varepsilon}_i^2=    \frac1{B^{2}}\sum_{j=1}^B\varepsilon_{b_ji}^{2}
+\frac1{B^{2}}\sum_{j\neq j'}\varepsilon_{b_ji}\varepsilon_{b_{j'}i}$, $\Delta_2$ can be further written as 
\begin{align*}
\Delta_2= \underbrace{\frac1{B^{2}}\sum_{i=1}^n\sum_{j=1}^B\big(\varepsilon_{b_ji}^{2}\big)-B^{-1}\Var_i
}_{\Delta_{21}}+\underbrace{\frac1{B^{2}}\sum_{i=1}^n\sum_{j\neq j'}\varepsilon_{b_ji}\varepsilon_{b_{j'}i}}_{\Delta_{22}}.
\end{align*}
We have $\EE_b \Delta_{21}=\EE_b \Delta_{22}=0$. Simple calculation based on the randomness of subsampling gives $\Var_b (\Delta_{21})=O(n^2/B^3)$, we have $|\Delta_{21}|=O_p(n/B^{3/2})$. As for $\Delta_{22}$, we consider $T_{b_j,b_{j'}}=\sum_{i=1}^n\varepsilon_{b_ji}\varepsilon_{b_{j'}i}$, and we have $\EE_bT_{b_j,b_{j'}}^2=O(r^2/n)$, and $\EE_b T_{b_j,b_{j'}}\cdot T_{b'_j,b'_{j'}}=0$ as long as $b_j,b_{j'},b'_j,b'_{j'}$ has three or more distinct values by the conditional independence. We have 
\begin{align*}
    \Var_b(\Delta_{22})\leq \frac{1}{B^4} \cdot \EE_b \sum_{b_j\neq b_{j'},b'_j\neq b'_{j'}}T_{b_j,b_{j'}}\cdot T_{b'_j,b'_{j'}}=O\bigg(\frac{r^2}{n B^2} \bigg)
\end{align*} 
We conclude that $|\Delta_2|=O_p(n/B^{3/2}+r/(B\sqrt{n}))$.

For the term $\Delta_3$, by the definition of $\hat \Var_i=\frac{1}{B-1}\sum_{j=1}^{B}\bigl(Z_{b_ji}-\hat V_i\bigr)^{2}$, it holds that 
\begin{align*}
\Delta_3&= \frac{1}{B(B-1)}\sum_{i=1}^n\sum_{j=1}^B \bigg( (Z_{b_ji}-\hat V_i)^2-\frac{B-1}{B}\Var_i\bigg)\\
&= \bigg(\frac{1}{B(B-1)}\sum_{i=1}^n\sum_{j=1}^B \bigg( (Z_{b_ji}- M_i)^2-\Var_i\bigg)\bigg)+O\bigg(\frac{n}{B^2}\bigg).
\end{align*}
Here, the first equality is by direct calculation and the second equality is by $\Var_i$ is bounded. Let    , easy to bound that $\Var_b(T_{b_j})=O(n^2)$, we have that
\begin{align*}
    &\EE_b \bigg\{ \frac{1}{B(B-1)}\sum_{i=1}^n\sum_{j=1}^B \bigg( (Z_{b_ji}-M_i)^2-\Var_i\bigg)\bigg\}=0;\\
    & \Var_b \bigg\{\frac{1}{B(B-1)}\sum_{i=1}^n\sum_{j=1}^B \bigg( (Z_{b_ji}-M_i)^2-\Var_i\bigg)\bigg\}=O(n^2/B^3),
\end{align*}
we hence conclude that $|\Delta_3|=O_p(n/B^{3/2})$.

We collect all the error terms of $\Delta_1$, $\Delta_2$ and $\Delta_3$, and have
\begin{align*}
    \Delta_1=O_p\bigg(\frac{r^{3/2}}{n\sqrt{B}}\bigg),\quad \Delta_2=O_p(n/B^{3/2}+r/(B\sqrt{n})),\quad \Delta_3=O_p(n/B^{3/2}),
\end{align*}
easy from order comparison to see that if $\Delta_1=o_p\big(\frac{r^2\xi_{1,r}(\xb_{*})}{n} \big)$, then $R_B-\hat R_{B}=o_p\big(\frac{r^2\xi_{1,r}(\xb_{*})}{n} \big)$. We have
\begin{align*}
    \frac{n\Delta_1}{r^2 \xi_{1,r}(\xb_{*})}=\frac{1}{\sqrt{rB}\xi_{1,r}(\xb_{*})}=o_p(1), 
\end{align*}
where we use the Assumption~\ref{assump:4} that $B \succsim n$ and $\varepsilon$ small enough. 
We hence conclude that $R_B-\hat R_B=o_p\big(\frac{r^2\xi_{1,r}(\xb_{*})}{n} \big)$.

It remains for us to prove $\sqrt{\frac{r^2\xi_{1,r}(\xb_{*})}{n\hsigma_1^2}}~\cvp~1$. 
All the terms $T_k$ on the right hand side in $M_i$ are uncorrelated. We define $A_1=\big(\frac{r}{n}-\frac{r^2}{n^2}\big)T_1(\zb_1)+\big(\frac{r(r-1)}{n(n-1)}-\frac{r^2}{n^2}\big)\sum_{i=2}^{n}T_1(\zb_i)$ which is the first line of the equation above.  We then have $\EE (M_1-A_1)\cdot A_1=0$. Similarly we define $A_i$ with the replace of $1$ by $i$.  
we give the calculation of three terms: $\EE A_i\cdot T_1(\zb_i)$, $\EE A_i^2$ and $\EE (M_i-A_i)^2$, and then give the proof based on these terms.

For the term $\EE A_i\cdot T_1(\zb_i)$, we have
\begin{align}
\EE A_i\cdot T_1(\zb_i)=\EE \bigg(\frac{r}{n}-\frac{r^2}{n^2}\bigg)T_1^2(\zb_1)=\bigg(\frac{r}{n}-\frac{r^2}{n^2}\bigg)\xi_{1,r}(\xb_{*}). \label{eq:A_i_dot_T}
\end{align}

For the term $\EE A_i^2$, we have 
\begin{align}
    \EE A_i^2&=\frac{r^2(n-r)^2}{n^4}\xi_{1,r}(\xb_{*})+\frac{r^2(n-1)}{n^2}\bigg(\frac{r-1}{n-1}-\frac{r}{n}\bigg)^2\xi_{1,r}(\xb_{*})\nonumber\\
    &=\frac{r^2(n-r)^2}{n^3(n-1)}\xi_{1,r}(\xb_{*}). \label{eq:E_A_i_square}
\end{align}

For the term $\EE (\hat{V}_i-A_i)^2$, we have 
\begin{align}
    \EE (M_i-A_i)^2&= \frac{r^2}{n^2}\sum_{k=1}^{r-1}{{n-1}\choose{k}}\bigg(\frac{{{r-1}\choose{k}}}{{{n-1}\choose{k}}}-\frac{{{r}\choose{k+1}}}{{{n}\choose{k+1}}} \bigg)^2V_{k+1}+\frac{r^2}{n^2}\sum_{k=2}^{r}{{n-1}\choose{k}}\bigg(\frac{{{r-1}\choose{k}}}{{{n-1}\choose{k}}}-\frac{{{r}\choose{k}}}{{{n}\choose{k}}} \bigg)^2V_k\nonumber\\
    &\precsim \frac{r^2(n-1)}{n^2} \bigg( \frac{{r-1}}{{n-1}} - \bigg( \frac{{\binom{r}{2}}}{{\binom{n}{2}}} \bigg)\bigg)^2 \binom{r}{2}^{-1} \Var[\hf^{b}(\xb_{*})]\nonumber\\
    &\precsim \frac{2r^2\Var[\hf^{b}(\xb_{*})]}{n^3}.\label{eq:cov_A_i_square}
\end{align}
Here, the first inequality is by the fact that summation in the first line is maximized in the second-order terms and $\sum_{k=1}^r {r\choose{k}}V_k=\Var(\hf^{b}(\xb_{*}))\leq C<+\infty$, and the second inequality is by simple calculation. We hence have  
\begin{align*}
    &\frac{n}{r^2\xi_{1,r}(\xb_{*})}\cdot\EE \frac{n(n-1)}{(n-r)^2}\sum_{i=1}^n\bigg(A_i-\frac{r(n-r)}{n^2}T_1(\zb_i)\bigg)^2\\
    &\qquad = \frac{n^3(n-1)}{r^2(n-r)^2\xi_{1,r}(\xb_{*})}\EE\bigg(A_i-\frac{r(n-r)}{n^2}T_1(\zb_i)\bigg)^2\\
    &\qquad = \frac{n^3(n-1)}{r^2(n-r)^2}\bigg(\frac{r^2(n-r)^2}{n^3(n-1)}-\frac{2r^2(n-r)^2}{n^4}+\frac{r^2(n-r)^2}{n^4} \bigg)=\frac{1}{n}\to 0.\end{align*}
Here, the last second equality comes from \eqref{eq:A_i_dot_T} and \eqref{eq:E_A_i_square}.  For any sequence of random variable $X_n>0$,  $\EE X_n\to 0$ implies $X_n~\cvp~0$ by the Markov inequality. Therefore, we conclude that
\begin{align}
    \frac{n}{r^2\xi_{1,r}(\xb_{*})}\cdot \frac{n(n-1)}{(n-r)^2}\sum_{i=1}^n\bigg(A_i-\frac{r(n-r)}{n^2}T_1(\zb_i)\bigg)^2~\cvp~ 0. \label{eq:covariance_Ai-T} 
\end{align}
By the weak law of large number, we also have 
\begin{align}
    \frac{n}{r^2\xi_{1,r}(\xb_{*})}\cdot \frac{n(n-1)}{(n-r)^2} \sum_{i=1}^n \frac{r^2(n-r)^2}{n^4}T_{1}^2(\zb_i) ~\cvp~1.\label{eq:convergence_Ti2}
\end{align}
Here $T_1^2(\zb_i)$ are iid and the expectation $\EE T_1^2(\zb_i)=\xi_{1,r}(\xb_{*})$. We then investigate the convergence of the term
\begin{align*}
\frac{n}{r^2\xi_{1,r}(\xb_{*})}\cdot \frac{n(n-1)}{(n-r)^2} \sum_{i=1}^n \frac{r}{n}A_i\cdot T_{1}(\zb_i).
\end{align*}
Recalling the definition of $A_i=\big(\frac{r}{n}-\frac{r^2}{n^2}\big)T_1(\zb_{i})+\big(\frac{r(r-1)}{n(n-1)}-\frac{r^2}{n^2}\big)\sum_{i'\neq i}^{n}T_1(\zb_{i'})$, the equation above can be simplified and  written as 
\begin{align*}
&\frac{n}{r^2\xi_{1,r}(\xb_{*})}\cdot \frac{n(n-1)}{(n-r)^2} \sum_{i=1}^n \frac{r(n-r)}{n^2}A_i\cdot T_{1}(\zb_i)\sim \frac{n}{r(n-r)\xi_{1,r}(\xb_{*})}\sum_{i=1}^n A_i\cdot T_{1}(\zb_i)\\
&\qquad =\frac{n}{r(n-r)\xi_{1,r}(\xb_{*})}\sum_{i=1}^n \bigg\{ \bigg(\frac{r}{n}-\frac{r^2}{n^2}\bigg)T_1^2(\zb_{i})+ \bigg(\frac{r(r-1)}{n(n-1)}-\frac{r^2}{n^2}\bigg) \sum_{i'\neq i}^{n}T_1(\zb_{i'})T_1(\zb_{i})\bigg\}.
\end{align*}
Again by utilizing the weak law of large number, we have
\begin{align*}
    \frac{n}{r(n-r)\xi_{1,r}(\xb_{*})}\sum_{i=1}^n  \bigg(\frac{r}{n}-\frac{r^2}{n^2}\bigg)T_1^2(\zb_{i})~\cvp~1.
\end{align*}
As for the  term consisting of $T_1(\zb_{i'})T_1(\zb_{i}) $ above, we have
\begin{align*}
    &\EE \bigg\{\frac{n}{r(n-r)\xi_{1,r}(\xb_{*})}\sum_{i=1}^n \bigg(\frac{r(r-1)}{n(n-1)}-\frac{r^2}{n^2}\bigg) \sum_{i'\neq i}^{n}T_1(\zb_{i'})T_1(\zb_{i}) \bigg\}^2\\
    &\qquad =\EE \bigg\{  \frac{1}{n(n-1)\xi_{1,r}(\xb_{*})} \sum_{i=1}^n\sum_{i'\neq i} T_1(\zb_{i'})T_1(\zb_{i})\bigg\}^2\\
    &\qquad =\frac{1}{n(n-1)}\to 0.
\end{align*}
Here, we use the fact $\EE (\sum_{i'\neq i}^{n}T_1(\zb_{i'})T_1(\zb_{i}))^2=n(n-1)\xi_{1,r}^2(\xb_{*})$ which follows from a straightforward calculation. This implies that the interaction term converges to $0$ in probability.  With these two terms, we have
\begin{align}
    \frac{n}{r^2\xi_{1,r}(\xb_{*})}\cdot \frac{n(n-1)}{(n-r)^2} \sum_{i=1}^n \frac{r(n-r)}{n^2}A_i\cdot T_{1}(\zb_i)~\cvp~1.\label{eq:convergence_TiAi}
\end{align}
Combing \eqref{eq:covariance_Ai-T} and  \eqref{eq:convergence_Ti2} with \eqref{eq:convergence_TiAi}, we have
\begin{align*}
    \frac{n}{r^2\xi_{1,r}(\xb_{*})}\cdot \frac{n(n-1)}{(n-r)^2}\sum_{i=1}^n A_i^2~\cvp~1.
\end{align*}
Moreover,  it follows from \eqref{eq:cov_A_i_square}  that  
\begin{align*}
    \EE \frac{n}{r^2\xi_{1,r}(\xb_{*})}\cdot \frac{n(n-1)}{(n-r)^2}\sum_{i=1}^n (M_i-A_i)^2\to 0,
\end{align*}
which means 
\begin{align*}
     \frac{n}{r^2\xi_{1,r}(\xb_{*})}\cdot \frac{n(n-1)}{(n-r)^2}\sum_{i=1}^n (M_i-A_i)^2~\cvp~ 0. 
\end{align*}
The triangle inequality implies that
\begin{align*}
    &\frac{n}{r^2\xi_{1,r}(\xb_{*})}\cdot \frac{n(n-1)}{(n-r)^2}\sum_{i=1}^n {M_i}^2\leq \frac{n}{r^2\xi_{1,r}(\xb_{*})}\cdot \frac{n(n-1)}{(n-r)^2}\bigg(\sum_{i=1}^n A_i^2+ \sum_{i=1}^n (M_i-A_i)^2\bigg)~\cvp~ 1,\\
    &\frac{n}{r^2\xi_{1,r}(\xb_{*})}\cdot \frac{n(n-1)}{(n-r)^2}\sum_{i=1}^n {M_i}^2\geq \frac{n}{r^2\xi_{1,r}(\xb_{*})}\cdot \frac{n(n-1)}{(n-r)^2}\bigg(\sum_{i=1}^n A_i^2- \sum_{i=1}^n (M_i-A_i)^2\bigg)~\cvp~ 1.
\end{align*}
Hence, the  squeeze theorem gives that
\begin{align*}
    \frac{n}{r^2\xi_{1,r}(\xb_{*})}\cdot \frac{n(n-1)}{(n-r)^2}\sum_{i=1}^n {M_i}^2~\cvp~1.
\end{align*}
This completes the proof of Theorem~\ref{thm:estimate_xi1r}.

\section{Proofs in Section~\ref{sec:proofthms}}
\label{sec:additional_proof}
We provide additional technical proofs for the conclusions in Section~\ref{sec:proofthms}. 
We first present the proof of Lemma~\ref{lemma:key_property_psi}. Then we show several additional lemmas, which will be used to prove Proposition~\ref{prop:upper_lower_bound_R_n}.

\begin{proof}[Proof of Lemma~\ref{lemma:key_property_psi}]
By the definition we have
\begin{align*}
    \psi(\eta)=\log\bigg(\int h(y)\exp(\eta y)dy\bigg),
\end{align*}
and it is clear that
\begin{align*}
    \EE(Y|\eta)=\psi'(\eta), \quad \Var (Y|\eta)=\psi''(\eta)\geq 0.
\end{align*}
We conclude that $\psi(\cdot)$ is a convex function. By Assumption~\ref{assump:1}, the non-degeneracy and the existence of distinct values in the support of $p(y|\eta) = h(y)\exp(\eta y - \psi(\eta))$ ensure that $\Var(Y|\eta) > 0$ for all $\eta \in [-C, C]$. This guarantees that the second derivative of the log-partition function, $\psi''(\eta) = \Var(Y|\eta)$, is strictly positive.
Furthermore, since $\EE(Y|\eta)$ and $\mathbb{E}(Y^2|\eta)$ are defined as smooth integrals with respect to $p(y|\eta)$, they are continuous in $\eta$. As a result, $\psi''(\eta)$ is also continuous. By the compactness of $\eta \in [-C, C]$, there exists a constant $c > 0$ such that:
\begin{align*}
   \psi''(\eta) \geq c > 0, \quad \text{for all } \eta \in [-C, C]. 
\end{align*}
Using the second-order Taylor expansion and the definition of strong convexity, we establish:
\begin{align*}
   \psi(x) - \psi(m) - \psi'(m)(x - m) = \frac{\psi''(\xi)}{2}(x - m)^2, 
\end{align*}
for some $\xi$ between $x$ and $m$. Since $C'\geq \psi''(\xi) \geq c' > 0$, we conclude:
\begin{align*}
\frac{C'}{2}(x - m)^2\geq\psi(x) - \psi(m) - \psi'(m)(x - m) \geq \frac{c'}{2}(x - m)^2.
\end{align*}

As for the Lipchitz property, by the equation $\psi'(\eta)=\EE (Y|\eta)$, the bound of $\psi'(\eta)$ directly comes from the existence of the expectation of $Y$ given $\eta$. 
\end{proof}

We propose several lemmas to address the general nonparametric regression cases, serving as the basis for the proof of Proposition~\ref{prop:upper_lower_bound_R_n}. Recall the definition of $ R_n(\hf_n, f_0) $ in \eqref{eq:def_Rn}:
\[
R_n(\hf_n, f_0) = \EE[\ell(\bX; \hf_n, f_0)],
\]
and consider the construction of
\[
\hR_n(\hf_n, f_0) = \EE\bigg[\frac{1}{n}\sum_{i=1}^n\ell(\xb_i; \hf_n, f_0)\bigg].
\]
In the following lemmas, we demonstrate that the differences between $ R_n(\hf_n, f_0) $ and $ \hR_n(\hf_n, f_0) $ are small.  Some of the proof techniques are inspired by \citet{schmidt2020nonparametric}, but our analysis extends to a more general regression framework, and our proofs  differ accordingly. In fact, the settings considered by \citet{schmidt2020nonparametric} are special cases within our broader context. Additionally, we provide upper and lower bounds for $ \hR_n(\hf_n, f_0) $. These results will be used for  proving Proposition~\ref{prop:upper_lower_bound_R_n}.

\begin{lemma}
\label{lemma:R_n-hatR_n}
Under Assumptions~\ref{assump:1}, it holds that
\begin{align*}
    \big|\hR_n(\hf_n,f_0)- R_n(\hf_n,f_0)\big|\leq \frac{2R}{n}R_n^{1/2}(\hf_n,f_0)\sqrt{9n\log\cN_n+4n\cdot\frac{1+3\log\cN_n}{\log\cN_n}} \\
    +\frac{2R(3\log\cN_n+2)}{n}+5\tilde{\delta} C_{\Lip}.
\end{align*}
Here, $\cN_n$ is the minimum number of $\tilde{\delta}$-covering of $\cF(L,\pb,s,\infty)$, and $R=3FC_{\Lip}$ is a constant assumed larger than $1$ without loss of generality.
\end{lemma}
Lemma~\ref{lemma:R_n-hatR_n} establishes the differences between $ R_n(\hf_n, f_0) $ and $ \hR_n(\hf_n, f_0) $. We now propose the following lemma, which addresses the challenges arising from the dependency of $\varepsilon_i$ on the covariates $\xb_i$. This approach is different from previous work, such as \citet{schmidt2020nonparametric} and \citet{fan2024factor}, which assumed independence between $\varepsilon_i$ and $\xb_i$.
\begin{lemma}
\label{lemma:remain_epsilon_term}
Suppose that Assumption~\ref{assump:1}  holds. Define $\varepsilon_i=y_i-\psi'(f_0(\xb_i))$. For any estimator $\tildef\in\cF$, there exists constant $C>0$ such that
\begin{align*}
\bigg|\EE\frac{1}{n}\sum_{i=1}^n \varepsilon_i\tildef(\xb_i)  \bigg|\leq \sqrt{\frac{2C\hat{R}_n(\tildef,f_0)\log\cN_n}{n}}+\sqrt{2C}\bigg(\frac{\log\cN_n}{n}+\sqrt{\frac{\log\cN_n}{n}}\tilde{\delta}\bigg)+2\tilde{\delta} C. 
\end{align*}
\end{lemma}
The main difficulty in Lemma~\ref{lemma:remain_epsilon_term} is the dependence of $\varepsilon_i$ and $\xb_i$ under the GNRM framework. Unlike traditional concentration inequalities, here we apply covering number analysis and truncation techniques to  address the heteroskedasticity in GNRMs. Details can be refered to Section~\ref{sec:proof_remain_epsilon}. With Lemma~\ref{lemma:remain_epsilon_term}, we establish the following two lemmas, which provides the upper and lower bound of $\hR_n(\hf_n,f_0)$.
\begin{lemma}
    \label{lemma:upper_bound_hatR}
    Suppose that Assumption~\ref{assump:1}  holds. For any fixed $f\in\cF$, it holds that 
\begin{align*}
    \hR_n(\hf_n,f_0)\leq \inf_{f\in\cF} \EE \ell(\bX;f,f_0)+\sqrt{\frac{2C\hat{R}_n(\tildef,f_0)\log\cN_n}{n}}+\sqrt{2C}\bigg(\frac{\log\cN_n}{n}+\sqrt{\frac{\log\cN_n}{n}}\tilde{\delta}\bigg)\\+2\tilde{\delta} C+\Delta_n(\hf_n).
\end{align*}
Here, $\bX$ is an independent copy of $\xb_i$ and $C>0$ is a  constant.
\end{lemma}

\begin{lemma}
    \label{lemma:lower_bound_hatR}
    Suppose that Assumption~\ref{assump:1}  holds. It holds that
\begin{align*}
    \hR_n(\hf_n,f_0)\geq \frac{1}{1+\rho}\cdot \bigg(\Delta_n(\hf_n)-2\sqrt{2C}\bigg(\frac{\log\cN_n}{n}+\sqrt{\frac{\log\cN_n}{n}}\tilde{\delta}\bigg)-4\tilde{\delta} C- \frac{C\log\cN_n}{2n}  \bigg)
\end{align*}
for any $\rho>0$. Here, $C>0$ is a  constant. 
\end{lemma}

\subsection{Proof of Lemma~\ref{lemma:R_n-hatR_n}}
Given a minimum $\tilde{\delta}-$covering of $\cF$, let the center of the balls by $f_j$. Recall $\cD_n=\{(\xb_i,y_i),i\in[n]\}$. By construction there exists a  $j^*$ such that $\|\hf-f_{j^*}\|_{\infty}\leq \tilde{\delta}$. Moreover, by the assumptions we have $\|f_j\|_{\infty}\leq F$. Suppose that $\bX'_i$ ($i\in[n]$) are i.i.d random variables with the same distribution of $\bX$ and independent of the sample $(\xb_i)_{i\in[n]}$ ($\cD_n$) and recall 
\begin{align*}
\ell(\xb;\hf_n,f_0)=-\psi'(f_0(\xb))\hf_n(\xb)+\psi(\hf_n(\xb))+\psi'(f_0(\xb))f_0(\xb)-\psi(f_0(\xb)).
\end{align*}
 It is clear to see that $\ell(\xb;\hf_n,f_0)\geq 0$ by Lemma~\ref{lemma:key_property_psi}, and 
\begin{align*}
&|R_n(\hf_n,f_0)-\hR_n(\hf_n,f_0)|\\
&=\bigg|  \EE\bigg\{ \frac{1}{n}\sum_{i=1}^n \Big[\ell(\bX'_i;\hf_n,f_0)-\ell(\xb_i;\hf_n,f_0)\Big]\bigg\}\bigg|.
\end{align*}
Rewriting the equation above by $\hf_n=\hf_n-f_{j^*}+f_{j^*}$, we have the inequality for any $\|\hf_n-f_{j^*}\|_{\infty}\leq \tilde{\delta}$, 
\begin{align*}
    |\ell(\xb;\hf_n,f_0)-\ell(\xb;f_{j^*},f_0)|\leq 2\tilde{\delta} C_{\Lip},
\end{align*}
and hence
\begin{align}
  |R_n(\hf_n,f_0)-\hR_n(\hf_n,f_0)|  \leq  \EE\bigg\{\bigg|\frac{1}{n}\sum_{i=1}^n g_{j^*}(\bX'_i,\xb_i) \bigg|\bigg\}+4\tilde{\delta} C_{\Lip},\label{eq:Rn-hatRn_bound}
\end{align}
where 
\begin{equation}
    \begin{split}
        g_{j^*}(\bX'_i,\xb_i)&= \ell(\bX'_i;f_{j^*},f_0)-\ell(\xb_i;f_{j^*},f_0) .
    \end{split}
    \label{eq:def_g_j*}
\end{equation}
We define $g_j$ the same way with $f_{j^*}$ replaced with $f_j$. By letting $U=\EE^{1/2} \{\ell(\bX;\hf_n,f_0)|\cD_n\}$, it is clear that 
\begin{align*}
    &\bigg|\sum_{i=1}^n g_{j^*}(\bX'_i,\xb_i)\bigg|= \frac{\big|\sum_{i=1}^n g_{j^*}(\bX'_i,\xb_i)\big|}{\max\{\sqrt{\log\cN_n/n},\EE^{1/2}[\ell(\bX;f_{j^*},f_0)|\cD_n]\}}\cdot \max\{\sqrt{\log\cN_n/n},\EE^{1/2}[\ell(\bX;f_{j^*},f_0)|\cD_n]\}\\
    &\quad \leq \max_j\frac{\big|\sum_{i=1}^n g_{j}(\bX'_i,\xb_i)\big|}{\max\{\sqrt{\log\cN_n/n},\EE^{1/2}[\ell(\bX;f_{j},f_0)]\}}\cdot \bigg(\sqrt{\log\cN_n/n}+U+\tilde{\delta} C_{\Lip}\bigg),
\end{align*}
where the second inequality comes from the fact $\|f_{j^*}-\hf_n\|_{\infty}\leq \tilde{\delta}$, and hence
$\EE^{1/2}[\ell(\bX;f_{j^*},f_0)|\cD_n]\}\leq U+2\tilde{\delta} C_{\Lip} $. Define the random variable 
\begin{align*}
    T=\max_j\frac{\big|\sum_{i=1}^n g_{j}(\bX'_i,\xb_i)\big|}{\max\{\sqrt{\log\cN_n/n},\EE^{1/2}[\ell(\bX;f_{j},f_0)]\}},
\end{align*}
and it follows  from \eqref{eq:Rn-hatRn_bound} and the inequality above that 
\begin{align}
    |R_n(\hf_n,f_0)-\hR_n(\hf_n,f_0)|\leq \EE \frac{T}{n}\cdot \Big(\sqrt{\log\cN_n/n}+U+2\tilde{\delta} C_{\Lip}\Big)+4\tilde{\delta} C_{\Lip}.\label{eq:Rn-hatRn_bound1}
\end{align}
Hence, it remains for us to consider the bound of \eqref{eq:Rn-hatRn_bound1}.  By the definition of $g_j(\bX'_i,\xb_i)$, it is clear that $\EE g_j(\bX'_i,\xb_i)=0$,  
\begin{align*}
    0\leq\ell(\xb;f_j,f_0)\leq R, 
\end{align*} 
and hence $|g_j(\bX'_i,\xb_i)|\leq 2R$. Moreover, we have that
\begin{align*}
    \Var \big(g_j(\bX'_i,\xb_i) \big)&=2\Var (\ell(\bX;f_j,f_0))\\
    &\leq 2\EE \ell^2 (\bX;f_j,f_0)\leq 2R\EE \ell (\bX;f_j,f_0).
\end{align*}
As $g_j(\bX'_i,\xb_i)$ is independent with different $i$, then for any i.i.d bounded and centered random variables $\xi_i$ with $|\xi_i|\leq M$, the Bernstein inequality shows that 
\begin{align*}
    \PP(|\sum_{i=1}^n \xi_i|\geq t_0)\leq 2\exp\bigg(-\frac{t_0^2}{2Mt_0/3+2\sum_{i=1}^n\Var(\xi_i)}\bigg).
\end{align*}
Combining  the Bernstein inequality with the expression of $T$, we can apply the union bound and get that
\begin{align*}
    &\PP(T/R\geq t)\leq 2\cN_n\cdot\exp\bigg(-\frac{t^2}{\frac{4}{3}t\max^{-1}\{\sqrt{\log\cN_n/n},\EE^{1/2}[\ell(\bX;f_{j},f_0)]\}+4n/R} \bigg).
\end{align*}
Here, we take $\xi_i=g_j(\bX'_j,\xb_j)/R$, and let $t_0=t\max\{\sqrt{\log\cN_n/n},\EE^{1/2}[\ell(\bX;f_{j},f_0)]\}$. According to \eqref{eq:T_inequality-prob}, we can bound the value of $\EE T$ and $\EE T^2$, which will further be applied in \eqref{eq:Rn-hatRn_bound1}. With the assumption above, we have $R\geq 1$, and the equation above can be simplified as  
\begin{align}
    &\PP(T/R\geq t)\leq 2\cN_n\cdot\exp\bigg(-\frac{t^2}{\frac{4}{3}t\max^{-1}\{\sqrt{\log\cN_n/n},\EE^{1/2}[\ell(\bX;f_{j},f_0)]\}+4n} \bigg).\label{eq:T_inequality-prob}
\end{align}
For the bound of $\EE T/R$, it holds that 
\begin{align}
    \EE T/R&=\EE T/R\cdot\one\{T/R\leq 6\sqrt{n\log\cN_n}\}+\EE T/R\cdot\one\{T/R\geq 6\sqrt{n\log\cN_n}\}\nonumber\\
    &= 6\sqrt{n\log\cN_n}+\int_{6\sqrt{n\log\cN_n}}^{+\infty}\PP(T\geq t)dt\nonumber\\
    &\leq 6\sqrt{n\log\cN_n}+\int_{6\sqrt{n\log\cN_n}}^{+\infty}2\cN_n \exp\bigg(-\frac{t\sqrt{\log\cN_n}}{2\sqrt{n}}\bigg) dt\nonumber\\
    &= 6\sqrt{n\log\cN_n}+4\sqrt{\frac{n}{\log\cN_n}}.\label{eq:bound_ET}
\end{align}
Here, the inequality comes from \eqref{eq:T_inequality-prob} that when $t\geq 6\sqrt{n\log\cN_n}$, $\PP(T/R>t)\leq 2\cN_n \exp\Big(-\frac{t\sqrt{\log\cN_n}}{2\sqrt{n}}\Big) $, and the last inequality is by some  algebra.

For the bound of $\EE T^2/R^2$, it holds that 
\begin{align}
    \EE T^2/R^2&=\EE T^2/R^2\cdot\one\{T^2/R^2\leq 6^2 n\log\cN_n\}+\EE T^2/R^2\cdot\one\{T^2/R^2\geq 6^2n\log\cN_n\}\nonumber\\
    &= 6^2n\log\cN_n+\int_{6^2n\log\cN_n}^{+\infty}\PP(T\geq \sqrt{t})dt\nonumber\\
    &\leq 6^2n\log\cN_n+\int_{6^2n\log\cN_n}^{+\infty}2\cN_n \exp\bigg(-\frac{\sqrt{t}\sqrt{\log\cN_n}}{2\sqrt{n}}\bigg) dt\nonumber\\
    &= 36n\log\cN_n+16\cdot\frac{1+3\log\cN_n}{\log\cN_n}.\label{eq:bound_ET2}
\end{align}
Here, the inequality comes from \eqref{eq:T_inequality-prob} that when $t\geq 6\sqrt{n\log\cN_n}$, $\PP(T/R>t)\leq 2\cN_n \exp\Big(-\frac{t\sqrt{\log\cN_n}}{(4/3+4/(2R+2))\sqrt{n}}\Big) $, and the last equality is by $\int_{b^2}^{+\infty}e^{-\sqrt{t}a}dt=2(ab+1)e^{-ab}/a^2$. 

Combining \eqref{eq:bound_ET}, \eqref{eq:bound_ET2} with \eqref{eq:Rn-hatRn_bound1}, and noting that $\EE U^2=R_n(\hf_n,f_0)$ where $U$ is defined below \eqref{eq:def_g_j*}, we have 
\begin{align*}
   |R_n(\hf_n,f_0)-\hR_n(\hf_n,f_0)|&\leq \EE T\cdot \Big(\sqrt{\log\cN_n/n}+U+\tilde{\delta} C_{\Lip}\Big)+4\tilde{\delta} C_{\Lip}\\ & \leq \frac{1}{n}\EE^{1/2} T^2\cdot \EE^{1/2}U^2+\EE \frac{T}{n} \cdot \Big(\sqrt{\log\cN_n/n}+\tilde{\delta} C_{\Lip}\Big)+4\tilde{\delta} C_{\Lip}\\
   &\leq   \frac{2R}{n}R_n^{1/2}(\hf_n,f_0)\sqrt{9n\log\cN_n+4n\cdot\frac{1+3\log\cN_n}{\log\cN_n}} \\
   &\quad +\frac{2R}{n}\bigg(3\sqrt{n\log\cN_n}+2\sqrt{\frac{n}{\log\cN_n}} \bigg)\cdot (\sqrt{\log\cN_n/n}+2\tilde{\delta} C_{\Lip})+4\tilde{\delta} C_{\Lip}\\
   &\leq \frac{2R}{n}R_n^{1/2}(\hf_n,f_0)\sqrt{9n\log\cN_n+4n\cdot\frac{1+3\log\cN_n}{\log\cN_n}}+\frac{2R(3\log\cN_n+2)}{n}+5\tilde{\delta} C_{\Lip}.
\end{align*}
Here, the first inequality is by the Cauchy-Schwarz inequality, the second inequality is by \eqref{eq:bound_ET} and \eqref{eq:bound_ET2} and the last inequality holds when $n$  is large enough. 


\subsection{Proof of Lemma~\ref{lemma:remain_epsilon_term}}
\label{sec:proof_remain_epsilon} 
For any estimator $\tildef$ taking values in $\cF$, let $j'$ be such that $\|\tildef-f_{j'}\|_{\infty}\leq \tilde{\delta}$. We have
\begin{align}
 \bigg|\EE\frac{1}{n}\sum_{i=1}^n \varepsilon_i\tildef(\xb_i)  \bigg|&= \bigg|\EE\frac{1}{n}\sum_{i=1}^n \varepsilon_i(\tildef(\xb_i)-f_{j'}(\xb_i)+f_{j'}(\xb_i))  \bigg|\nonumber\\
 &\leq \frac{\tilde{\delta}}{n}\EE\sum_{i=1}^n|\varepsilon_i|+\frac{1}{n}\bigg|\EE\sum_{i=1}^n \varepsilon_if_{j'}(\xb_i)\bigg|\nonumber\\
 &\leq 2\tilde{\delta} C+\frac{1}{n}\bigg|\EE\sum_{i=1}^n \varepsilon_if_{j'}(\xb_i)\bigg|.\label{eq:key_step2}
\end{align}
Here, $C$ is some large constant. the first inequality is by the triangle inequality, and the second inequality is by $\EE|\varepsilon_i|=\EE\{\EE \{|y_i-\psi'(f_0(\xb_i)) ||\xb_i\} \}\leq L+\EE |Y|\leq C$ for some large constant C. It is remained for us to bound the last term above. To obtain the bound, we first introduce  $r_{j}=\sqrt{\frac{\log \cN_n}{n}}\vee \sqrt{\sum_{i=1}^n(f_j(\xb_i)-f_0(\xb_i))^2/n}$. Then the term can be written as 
\begin{align}
\frac{1}{n}\bigg|\EE\sum_{i=1}^n \varepsilon_if_{j'}(\xb_i)\bigg|&=\frac{1}{n}\bigg|\EE\sum_{i=1}^n \varepsilon_i(f_{j'}(\xb_i)-f_0(\xb_i))\bigg| \nonumber\\
&\leq \frac{1}{n}\cdot \EE \bigg[\bigg(\sqrt{\frac{\log \cN_n}{n}}+ \sqrt{\sum_{i=1}^n(f_{j'}(\xb_i)-f_0(\xb_i))^2/n}~\bigg)\cdot \frac{\Big|\sum_{i=1}^n \varepsilon_i(f_{j'}(\xb_i)-f_0(\xb_i))\Big|}{r_{j'}} \bigg]\nonumber\\
&\leq \frac{1}{n}\cdot \EE \bigg[\bigg(\sqrt{\frac{\log \cN_n}{n}}+ \sqrt{\sum_{i=1}^n(\tildef(\xb_i)-f_0(\xb_i))^2/n}+\tilde{\delta}\bigg)\cdot \frac{\Big|\sum_{i=1}^n \varepsilon_i(f_{j'}(\xb_i)-f_0(\xb_i))\Big|}{r_{j'}} \bigg] \nonumber\\
&\leq \frac{1}{n}\bigg(\sqrt{\frac{\log \cN_n}{n}}+\EE^{1/2} \bigg(\frac{\sum_{i=1}^n\ell(\xb_i;\tildef,f_0)}{cn}\bigg)+\tilde{\delta}\bigg)\cdot \EE^{1/2} \frac{\Big|\sum_{i=1}^n \varepsilon_i(f_{j'}(\xb_i)-f_0(\xb_i))\Big|^2}{r_{j'}^2}\nonumber\\
&=\frac{1}{n}\bigg(\sqrt{\frac{\log \cN_n}{n}}+\hat{R}_n^{1/2}(\tildef,f_0)/\sqrt{c}+\tilde{\delta}\bigg)\cdot \EE^{1/2} \frac{\Big|\sum_{i=1}^n \varepsilon_i(f_{j'}(\xb_i)-f_0(\xb_i))\Big|^2}{r_{j'}^2}.\label{eq:key_bound_step2}
\end{align}
Here, the first equality is by $\EE \varepsilon_if_0(\xb_i)=\EE (\EE\varepsilon_if_0(\xb_i)|\xb_i)=0$, the first inequality is by $r_j\leq\frac{\log \cN_n}{n}+ \sqrt{\sum_{i=1}^n(f_j(\xb_i)-f_0(\xb_i))^2/n}$, the second inequality is by triangle inequality $\|\ab+\bbb\|\leq \|\ab\|+\|\bbb\| $ and $\|f_{j'}-\tildef\|_{\infty}\leq \tilde{\delta}$, and the last inequality is by Lemma~\ref{lemma:key_property_psi} that $\ell(\xb;\tildef,f_0)\geq c(\tildef(\xb)-f_0(\xb))^2$ for some constant $c>0$. 

It remains for us to bound $\EE^{1/2} \frac{\big|\sum_{i=1}^n \varepsilon_i(f_{j'}(\xb_i)-f_0(\xb_i))\big|^2}{r_{j'}^2}$. The heteroskedasticity in GNRMs also makes the bound of this term complicated, mainly due to the dependence among $\varepsilon_i$ and $\xb_i$. 
For simplicity in notation, define 
\begin{align}
    \xi_j=\frac{\big|\sum_{i=1}^n \varepsilon_i(f_{j}(\xb_i)-f_0(\xb_i))\big|}{r_j}.\label{eq:def_xij_step2}
\end{align}
We investigate the noise behavior $\xi_j^2$ under different regimes to bound the term with heteroskedasticity. Specifically, we consider $\max_{j\in[\cN_n]}\xi_j^2\leq Cn\log \cN_n$ and $\max_{j\in[\cN_n]}\xi_j^2\geq Cn\log \cN_n$ conditional on $\xb_i$. When $\max_{j\in[\cN_n]}\xi_j^2\leq Cn\log \cN_n$ conditional on $\xb_i$, the results will obviously hold. We mainly focus the regime when $\max_{j\in[\cN_n]}\xi_j^2\geq Cn\log \cN_n$ conditional on $\xb_i$.
Conditional on $\xb_1, \ldots, \xb_n$ [or  $(\xb_i)_i$ compactly], we have that $\varepsilon_i$'s ($i\in[n]$) are centered and independent sub-exponential random variables. Hence, we have
\begin{align*}
    \Var(\xi_j|(\xb_i)_i)&=\frac{\sum_{i=1}^n\Var \big(\varepsilon_i(f_{j}(\xb_i)-f_0(\xb_i))| (\xb_i)_i\big)}{r_j^2}\\
    &\leq  \sum_{i=1}^n \EE\big(\varepsilon_i^2|(\xb_i)_i\big)\cdot \EE \big( (f_{j}(\xb_i)-f_0(\xb_i))^2|(\xb_i)_i\big)/r_j^2\\
    &\leq 2n\kappa^2,
\end{align*}
where the first inequality is by $\Var(U|X)\leq \EE (U^2|X)$ for all random variables $X$ and $U$, and the second inequality is by the definition of $r_{j}=\sqrt{\frac{\log \cN_n}{n}}\vee \sqrt{\sum_{i=1}^n(f_j(\xb_i)-f_0(\xb_i))^2/n}$, and  $\EE\big(\varepsilon_i^2|(\xb_i)_i\big)\leq 2\kappa^2$.   Also, letting $a_{ij}=[f_{j}(\xb_i)-f_0(\xb_i)]/r_j$, we have
\begin{align*}
    |a_{ij}|\leq 2F/r_j,\quad \sum_{i=n}^n a_{ij}^2\leq n. 
\end{align*}
Therefore by the Bernstein inequality, it holds that
\begin{align*}
\PP(|\xi_j|\geq t|(\xb_i)_i)\leq 2\exp\bigg\{-c'\min\bigg\{\frac{t^2}{n\kappa^2},\frac{tr_j}{2F\kappa}\bigg\}  \bigg\}
\end{align*}
for some absolute constant $c'>0$. 

Let $\xi_{\max}=\max_{j\in[\cN_n]}\xi_j$. By applying the union bound, it holds that 
\begin{align*}
    \EE (|\xi_{\max}|^2|(\xb_i)_i)&=\int_{0}^{+\infty}\PP(|\xi_{\max}|^2\geq t|(\xb_i)_i)dt\\
    &\leq T+\int_{T}^{+\infty}\PP(|\xi_{\max}|\geq \sqrt{t}|(\xb_i)_i)dt\\
    &\leq T+\cN_n\cdot\int_{T}^{+\infty} 2\exp\bigg\{-c'\min\bigg\{\frac{t}{n\kappa^2},\frac{\sqrt{t}r_j}{2F\kappa}\bigg\}  \bigg\}dt\\
    &\leq T+\cN_n\cdot\int_{T}^{+\infty} 2\exp\bigg\{-c'\min\bigg\{\frac{t}{n\kappa^2},\frac{\sqrt{t\log\cN_n}}{2F\kappa\sqrt{n}}\bigg\}  \bigg\}dt.
\end{align*}
Here, the last inequality is by $r_j\geq \sqrt{\frac{\log\cN_n}{n}}$. By letting $T=Cn\log\cN_n$ for some  constant $C\geq 4F^2\kappa^2/c'^2$, it is clear that when $t\geq T$, $\frac{t}{n\kappa^2}\geq\frac{\sqrt{t\log\cN_n}}{2F\kappa\sqrt{n}}$, the inequality above can be written as 
\begin{align*}
\EE (|\xi_{\max}|^2|(\xb_i)_i)\leq Cn\log\cN_n+\cN_n\cdot\int_{Cn\log\cN_n}^{+\infty} 2\exp\bigg\{-c'\frac{\sqrt{t\log\cN_n}}{2F\kappa\sqrt{n}}  \bigg\}dt.
\end{align*}
Again by utilizing the equality $\int_{b^2}^{+\infty}e^{-\sqrt{t}a}dt=2(ab+1)e^{-ab}/a^2$ we conclude that as long as $C\geq 4F^2\kappa^2/c'^2$,
\begin{align*}
    \EE (|\xi_{\max}|^2|(\xb_i)_i)&\leq Cn\log\cN_n+\cN_n\cdot\int_{Cn\log\cN_n}^{+\infty} 2\exp\bigg\{-c'\frac{\sqrt{t\log\cN_n}}{2F\kappa\sqrt{n}}  \bigg\}dt\\
    &= Cn\log\cN_n+4\cN_n\bigg(\frac{c'\sqrt{C}}{F\kappa}\log\cN_n+1 \bigg)e^{-\frac{c'\sqrt{C}}{2F\kappa}\log\cN_n}\cdot \frac{F^2\kappa^2n}{c'^2\log\cN_n}\\
    &\leq 2Cn\log\cN_n.
\end{align*}
By the definition of $\xi_j$ in \eqref{eq:def_xij_step2} and with $\xi_{\max}=\max_{j\in[\cN]}\xi_j$, 
we have for any $j'$,
\begin{align*}
    \EE^{1/2} \frac{\Big|\sum_{i=1}^n \varepsilon_i(f_{j'}(\xb_i)-f_0(\xb_i))\Big|^2}{r_{j'}^2}&\leq \sqrt{\EE (\EE (|\xi_{\max}|^2|(\xb_i)_i))}\\
    &\leq \sqrt{2Cn\log\cN_n}.
\end{align*}
Combining the inequality above with \eqref{eq:key_bound_step2}, we have
\begin{align}
    \frac{1}{n}\bigg|\EE\sum_{i=1}^n \varepsilon_if_{j'}(\xb_i)\bigg|&\leq \frac{1}{n}\bigg(\sqrt{\frac{\log \cN_n}{n}}+\hat{R}_n^{1/2}(\tildef,f_0)/\sqrt{c}+\tilde{\delta}\bigg)\cdot \sqrt{2Cn\log\cN_n}.\nonumber\\
    &\leq \sqrt{\frac{2C\hat{R}_n(\tildef,f_0)\log\cN_n}{n}}+\sqrt{2C}\bigg(\frac{\log\cN_n}{n}+\sqrt{\frac{\log\cN_n}{n}}\tilde{\delta}\bigg)\label{eq:key_bound_step2_1}
\end{align}
for some constant $C>0$. Combining \eqref{eq:key_bound_step2_1} with \eqref{eq:key_step2} we conclude that for any $\tildef\in\cF$, 
\begin{align*}
     \bigg|\EE\frac{1}{n}\sum_{i=1}^n \varepsilon_i\tildef(\xb_i)  \bigg|\leq \sqrt{\frac{2C\hat{R}_n(\tildef,f_0)\log\cN_n}{n}}+\sqrt{2C}\bigg(\frac{\log\cN_n}{n}+\sqrt{\frac{\log\cN_n}{n}}\tilde{\delta}\bigg)+2\tilde{\delta} C
\end{align*}
for some large constant $C>0$, which completes the proof of Lemma~\ref{lemma:remain_epsilon_term}.

\subsection{Proof of Lemma~\ref{lemma:upper_bound_hatR}}
Recall the definition of $\Delta_n(\hf_n)$ 
\begin{align*}
   \Delta_n(\hf_n)=\EE \frac{1}{n} \sum_{i=1}^n  -  y_i \hf_n(\xb_i) +  \psi(\hf_n(\xb_i)) - \inf_{f \in \mathcal{F}(L, \mathbf{p}, s, \mathcal{F})}  \frac{1}{n} \sum_{i=1}^n  -  y_i f(\xb_i) +  \psi( f(\xb_i)) .
\end{align*}
The first inequality we can obtain is,  for any fixed $f\in\cF$, 
\begin{align}
    \EE \frac{1}{n} \sum_{i=1}^n  -  y_i \hf_n(\xb_i) +  \psi(\hf_n(\xb_i)) \leq \EE \frac{1}{n} \sum_{i=1}^n  -  y_i  f(\xb_i) +  \psi(f(\xb_i))+\Delta_n(\hf_n).\label{eq:step3_1}
\end{align}
Note that $f$ and $f_0$ are both fixed. For any $\bX\stackrel{\cD}{=}\xb_i$, the inequality above further implies that
\begin{align*}
    \hat{R}_n(\hf_n,f_0)&\leq \EE \frac{1}{n}\sum_{i=1}^n\ell(\xb_i;f,f_0)+\EE \frac{1}{n}\sum_{i=1}^n\varepsilon_i\hf_n(\xb_i)+\Delta_n(\hf_n)\\
    &=\EE \ell(\bX;f,f_0)+\EE \frac{1}{n}\sum_{i=1}^n\varepsilon_i\hf_n(\xb_i)+\Delta_n(\hf_n)\\
    &\leq \inf_{f\in\cF}\EE \ell(\bX;f,f_0)+\sqrt{\frac{2C\hat{R}_n(\hf_n,f_0)\log\cN_n}{n}} \\
    &+\sqrt{2C}\bigg(\frac{\log\cN_n}{n} +\sqrt{\frac{\log\cN_n}{n}}\tilde{\delta}\bigg)+2\tilde{\delta} C+\Delta_n(\hf_n).
\end{align*}
Here, $\varepsilon_i=y_i-\psi'(f_0(\xb_i))$. The first inequality is  from \eqref{eq:step3_1}, the first equality is by $\EE \ell(\bX;f,f_0)=\EE\ell(\xb_i;f,f_0)$ for any fixed $f\in\cF$, and the last inequality is by Lemma~\ref{lemma:remain_epsilon_term}.

\subsection{Proof of Lemma~\ref{lemma:lower_bound_hatR}}
Let $\tildef=\argmin_{f\in\cF}\sum_{i=1}^n  -  y_i f(\xb_i) +  \psi(f(\xb_i))$ be the global empirical risk estimator. We have the following decomposition, by letting $\varepsilon_i=y_i-\psi'(f_0(\xb_i))$. That is,
\begin{align}
    \hR_n(\hf_n,f_0)-\hR_n(\tildef,f_0)=\Delta_n(\hf_n)+\EE\frac{1}{n}\sum_{i=1}^n \varepsilon_i\hf_n(\xb_i)-\EE\frac{1}{n}\sum_{i=1}^n \varepsilon_i\tildef(\xb_i).\label{eq:step4_1}
\end{align}
From Lemma~\ref{lemma:remain_epsilon_term}, we have
\begin{align*}
\bigg|\EE\frac{1}{n}\sum_{i=1}^n \varepsilon_i\tildef(\xb_i)  \bigg|\leq \sqrt{\frac{2C\hat{R}_n(\tildef,f_0)\log\cN_n}{n}}+\sqrt{2C}\bigg(\frac{\log\cN_n}{n}+\sqrt{\frac{\log\cN_n}{n}}\tilde{\delta}\bigg)+2\tilde{\delta} C
\end{align*}
for any $\tildef\in\cF$ with some constant $C>0$.  Therefore  \eqref{eq:step4_1} yields
\begin{equation}
\label{eq:step4_2}
    \begin{split}
           \hR_n(\hf_n,f_0)-\hR_n(\tildef,f_0)&\geq  \Delta_n(\hf_n)-\sqrt{\frac{2C\hat{R}_n(\hf_n,f_0)\log\cN_n}{n}}-\sqrt{\frac{2C\hat{R}_n(\tildef,f_0)\log\cN_n}{n}}\\
    &\quad -2\sqrt{2C}\bigg(\frac{\log\cN_n}{n}+\sqrt{\frac{\log\cN_n}{n}}\tilde{\delta}\bigg)-4\tilde{\delta} C. 
    \end{split}
\end{equation}
For any $a,b>0$, it holds that  $2ab\leq \rho a^2+b^2/\rho$ for any $\rho>0$. By letting $a=\hR_n(\hf_n,f_0)$, $b=\hR_n(\tildef,f_0)$, $c=\sqrt{\frac{C\log\cN_n}{2n}}$ and $d= \Delta_n(\hf_n)-2\sqrt{2C}\bigg(\frac{\log\cN_n}{n}+\sqrt{\frac{\log\cN_n}{n}}\tilde{\delta}\bigg)-4\tilde{\delta} C$, \eqref{eq:step4_2} can be written as 
\begin{align*}
    a-b&\geq d-2\sqrt{a}c-2\sqrt{b}c\\
    &\geq d-\rho a-\frac{c^2}{\rho}-b-c^2
\end{align*}
for any $\rho>0$. We directly conclude that 
\begin{align*}
    a\geq \frac{d}{1+\rho}-\frac{c^2}{1+\rho},
\end{align*}
which is 
\begin{align*}
    \hR_n(\hf_n,f_0)&\geq \frac{1}{1+\rho}\cdot \bigg(\Delta_n(\hf_n)-2\sqrt{2C}\bigg(\frac{\log\cN_n}{n}+\sqrt{\frac{\log\cN_n}{n}}\tilde{\delta}\bigg)-4\tilde{\delta} C- \frac{C\log\cN_n}{2n}  \bigg).
\end{align*}
This completes the proof of Lemma~\ref{lemma:lower_bound_hatR}.

\subsection{Proof of Proposition~\ref{prop:upper_lower_bound_R_n}}
We first prove the upper bound. From Lemma~\ref{lemma:R_n-hatR_n}, we can let $a_1=\hat{R}_n(\hf_n,f_0)$, $b_1=R_n(\hf_n,f_0)$, $c_1=\frac{R}{n}\sqrt{9n\log\cN_n+4n\frac{1+3\log\cN_n}{\log\cN_n}}$ and $d_1=\frac{2R(3\log\cN_n+2)}{n}+5\tilde{\delta} C_{\Lip}$. The inequality in Lemma~\ref{lemma:R_n-hatR_n} can be written as 
\begin{align}
    |a_1-b_1|\leq 2\sqrt{a_1}c_1+d_1. \label{eq:Proof_delta_Rn_1}
\end{align}
Furthermore, from Lemma~\ref{lemma:upper_bound_hatR}, we can let $a_2=a_1=\hat{R}_n(\hf_n,f_0)$, $c_2=\sqrt{\frac{C\log\cN_n}{2n}}$ and $d_2=\inf_{f\in\cF}\EE \ell(\bX;f,f_0)+\sqrt{2C}\bigg(\frac{\log\cN_n}{n}+\sqrt{\frac{\log\cN_n}{n}}\tilde{\delta}\bigg)+2\tilde{\delta} C+\Delta_n(\hf_n)$. The inequality in Lemma~\ref{lemma:upper_bound_hatR} can be written as 
\begin{align}
    a_2\leq 2\sqrt{a_2}c_2+d_2.\label{eq:Proof_delta_Rn_2}
\end{align}
From \eqref{eq:Proof_delta_Rn_1}, we have, for any $0<\rho_1<1$, 
\begin{align*}
    b_1-a_1\leq \rho_1 a_1+\frac{c_1^2}{\rho_1}+d_1, 
\end{align*}
while \eqref{eq:Proof_delta_Rn_2} implies that, for any $\rho_2>0$, 
\begin{align*}
    a_2\leq \rho_2a_2+\frac{c_2^2}{\rho_2}+d_2. 
\end{align*}
Setting $a_1=a_2$ and combining the two inequalities above,  we have 
\begin{align*}
    b_1\leq (1+\rho_1)\cdot\bigg(\frac{c_2^2}{(1-\rho_2)\rho_2}+d_2 \bigg)+\frac{c_1^2}{\rho_1}+d_1.
\end{align*}
By the definition of  $a_1$, $b_1$, $c_1$ $d_1$, $a_2$, $c_2$ and $d_2$ and by letting $\rho_1=1$ and $\rho_2=0.5$, we conclude that there exists a constant $C'$ depending only on  $F$, function $\psi$ and $\kappa$ such that 
\begin{align*}
    R_n(\hf,f_0)\leq 2\Delta_n(\hf_n)+2\inf_{f\in\cF}\EE\ell(\bX;f,f_0)+C'\bigg(\frac{\log\cN_n}{n}+\bigg(\sqrt{\frac{\log\cN_n}{n}}+1 \bigg)\tilde{\delta}\bigg).
\end{align*}
We complete the upper bound in Proposition~\ref{prop:upper_lower_bound_R_n}.

We then prove the lower bound in Proposition~\ref{prop:upper_lower_bound_R_n}. From Lemma~\ref{lemma:lower_bound_hatR}, we can let $a_3=\hat{R}_n(\hf_n,f_0)$, and have
\begin{align}
    a_1=a_3\geq \frac{1}{1+\rho}\cdot \bigg(\Delta_n(\hf_n)-2\sqrt{2C}\bigg(\frac{\log\cN_n}{n}+\sqrt{\frac{\log\cN_n}{n}}\tilde{\delta}\bigg)-4\tilde{\delta} C- \frac{C\log\cN_n}{2n}  \bigg).\label{eq:lower_boundagain}
\end{align}
Hence, from \eqref{eq:Proof_delta_Rn_1} we have for any $0<\rho_3<1$,
\begin{align*}
    a_1-b_1\leq \rho_3 a_1+\frac{c_1^2}{\rho_3}+d_1,
\end{align*}
or 
\begin{align*}
    b_1\geq (1-\rho_3)a_1-\frac{c_1^2}{\rho_3}-d_1.
\end{align*}
We can  let $\rho=0.2$ in \eqref{eq:lower_boundagain} and $\rho_3=0.4$, and it is easy to conclude that there exists a constant $c'$  depending only on  $F$, function $\psi$ and $\kappa$ such that 
\begin{align*}
    R_n(\hf,f_0)\geq \frac{1}{2}\Delta_n(\hf_n)-c'\bigg(\frac{\log\cN_n}{n}+\bigg(\sqrt{\frac{\log\cN_n}{n}}+1 \bigg)\tilde{\delta}\bigg).
\end{align*}

\end{document}